\documentclass[journal,twoside,web]{ieeecolor}
\usepackage{generic}
\usepackage{cite}
\usepackage{amsmath,amssymb,mathrsfs,amsfonts,bbm}
\usepackage{color}
\usepackage{xr}
\usepackage[ruled,vlined]{algorithm2e}
\usepackage{graphicx}
\usepackage{subcaption}
\usepackage{textcomp}
\usepackage{mathtools}
\newtheorem{problem}{Problem}

\newtheorem{assumption}{Assumption}
\newtheorem{theorem}{Theorem}

\newtheorem{remark}{Remark}

\begin{document}

\title{\LARGE \bf Multi-Robot-Guided Crowd Evacuation: Two-Scale Modeling and Control}

\author{Tongjia Zheng$^{1}$, Zhenyuan Yuan$^{2}$, Mollik Nayyar$^{3}$, Alan R. Wagner$^{3}$, Minghui Zhu$^{2}$, and Hai Lin$^{1}$
\thanks{*This work was supported by the National Science Foundation under Grant No. CNS-1830335 and CNS-1830390. Any opinions, findings, and conclusions or recommendations expressed in this material are those of the author(s) and do not necessarily reflect the views of the National Science Foundation.}
\thanks{$^{1}$Tongia Zheng and Hai Lin are with the Department of Electrical Engineering, University of Notre Dame, Notre Dame, IN 46556, USA. {\tt\small tzheng1@nd.edu, hlin1@nd.edu.}} 
\thanks{$^{2}$Zhenyuan Yuan and Minghui Zhu are with the Department of Electrical Engineering,  The Pennsylvania State University, University Park, PA 16802, USA. {\tt\small zqy5086@psu.edu, muz16@psu.edu.}} 
\thanks{$^{3}$Mollik Nayyar and Alan Richard Wagner are with the Department of Aerospace Engineering, The Pennsylvania State University, University Park, PA 16802, USA. {\tt\small mxn244@psu.edu, azw78@psu.edu.}}
}

\maketitle
\thispagestyle{empty}

\begin{abstract}
Emergency evacuation describes a complex situation involving time-critical decision-making by evacuees.
Mobile robots are being actively explored as a potential solution to provide timely guidance.
In this work, we study a robot-guided crowd evacuation problem where a small group of robots is used to guide a large human crowd to safe locations.
The challenge lies in how to use micro-level human-robot interactions to indirectly influence a population that significantly outnumbers the robots to achieve the collective evacuation objective.
To address the challenge, we follow a two-scale modeling strategy and explore hydrodynamic models, which consist of a family of microscopic social force models that describe how human movements are locally affected by other humans, the environment, and robots, and associated macroscopic equations for the temporal and spatial evolution of the crowd density and flow velocity.
We design controllers for the robots such that they not only automatically explore the environment (with unknown dynamic obstacles) to cover it as much as possible, but also dynamically adjust the directions of their local navigation force fields based on the real-time macrostates of the crowd to guide the crowd to a safe location.
We prove the stability of the proposed evacuation algorithm and conduct extensive simulations to investigate the performance of the algorithm with different combinations of human numbers, robot numbers, and obstacle settings.
\end{abstract}


\section{Introduction}
Emergency evacuation can be a chaotic situation defined by the need to relocate a possibly large number of people to safety without causing choke points that slow the evacuation process \cite{bryner2007reconstructing, robinette2012information}.
Robots may be developed to serve as instantaneous first responders to an emergency by contacting authorities and guiding people to safety \cite{wagner2021robot}.
Compared to traditional stationary signage guidance, robot-guided emergency evacuation offers many advantages \cite{tang2016human}.
First, robots can be deployed beforehand and approach evacuees quickly.
Second, robots can be provided with a map and find the optimal routes more efficiently.
Third, robots can access the dynamic environment information in real-time and offer timely decisions.

{\color{black}
To create a robot-controlled evacuation system, it is essential to comprehend the crowd's behavior and how the robots can influence it.
Researchers from various disciplines such as social science and engineering have been investigating crowd dynamics for many years, resulting in the publication of numerous survey papers and books on the subject \cite{kachroo2008pedestrian, bellomo2011modeling, cristiani2014multiscale, haghani2021knowledge}.
A comprehensive list of reviews is available in \cite{cristiani2023all}.
Modeling strategies can generally be classified into three categories: microscopic, macroscopic, and their combinations.
}

Microscopic models focus on the details of individuals.
{\color{black}
Social force models are the most widely adopted in this category, which assume that interactions among humans are implemented by the concept of social forces \cite{helbing1995social, chen2018social}.}
They can produce several natural crowd phenomena in normal and emergent situations \cite{helbing2002simulation}.
Social force models have been widely explored for robot-guided evacuation because they can easily incorporate the influence of robots as additional social forces that act on humans.
For example, Garrell et al. \cite{garrell2009discrete} use robots to lead a group of people and prevent people from leaving the group based on social forces.
Jiang et al. \cite{jiang2016robot} employ robots as dynamic obstacles that generate repulsive forces near exits to improve evacuation efficiency.
Tang et al. \cite{tang2016human} introduce a panic propagation model to complement the social force model and develop an exit selection algorithm for the robots to find the best exit.
Yuan et al. \cite{yuan2023multi} use social force models to model human-robot interaction and propose an optimization formulation for the synthesis of robot controllers.
{\color{black}
Other approaches based on microscopic crowd models include the work \cite{boukas2014robot} that uses cellular automata to model crowd dynamics and test the system using a robot to guide human subjects.}
More results on guided crowd evacuation can be found in the survey \cite{zhou2019guided} which includes a section on robot-guided evacuation.
We also refer to the works \cite{robinette2016investigating, nayyar2019effective} which primarily consider single-robot, single-human scenarios but focus more on the aspects of human-robot interactions during emergencies.
Microscopic models are widely adopted.
However, they pose a challenge when dealing with scenarios that involve numerous humans and multiple robots. 
In such cases, the system becomes very high-dimensional, making it challenging to assess the collective evacuation performance.

Macroscopic models, on the contrary, describe general crowd movement and do not provide details about individuals \cite{hughes2002continuum, bellomo2008modelling}. 
They generally consist of a system of partial differential equations that define the evolution of the density and mean flow velocity of the crowd.
Macroscopic crowd models are particularly suited for studying large numbers of humans and analyzing their group properties.
Nevertheless, they are seldom explored for robot-guided evacuation, possibly because of the difficulty of incorporating the role of human-robot interactions taking place on the microscopic level.

{\color{black}
The weakness of microscopic and macroscopic models can be overcome by exploring two-scale models.
Two-scale models can be viewed as mixed micro-macroscopic models that inherit the strengths of both sides and have been increasingly studied in recent years \cite{borsche2018numerical, cristiani2011MultiscaleModelingGranular}.
These models are composed of not only microscopic equations (like those of social force models) for individual movements and interactions, but also macroscopic equations for crowd density and/or flow speed.}
Consequently, they can investigate the non-local effect of collective dynamics on individual actions, as seen in experimental studies \cite{moussaid2010walking, von2017empirical}.
We point out that if the robots' guidance is based on real-time macroscopic information such as the crowd density and flow, their influence on individual behaviors is necessarily non-local, for which two-scale models are more suitable for this purpose.
{\color{black}
Two-scale models have been used to address the control problem of large-scale systems, such as robotic swarms \cite{zheng2021transporting} or multi-agent systems in an abstract context \cite{fornasier2014mean, borzi2015modeling}. 
In the context of human evacuation, the unique challenge is how to incorporate the robots' role into the existing crowd models and how to use these robots to indirectly control the dynamics of the crowd. 
This work aims to address this particular challenge.
}

In this paper, we explore the role of two-scale models, particularly the hydrodynamic models, as a tool for designing efficient robot-guided crowd evacuation systems.
We start with modeling the microscopic behaviors of humans using social force models and incorporating the robots' guidance as additional navigation social forces.
{\color{black} These individual-based human models include human-human and human-robot interactions and unknown social forces to account for unpredicted environmental change (e.g., dynamic obstacles).}
Then, we derive the corresponding macroscopic hydrodynamic model which represents the temporal and spatial evolution of the crowd density and flow velocity under the navigation of the robots.
The control design for the robots is divided into two parts: position control and direction control.
The position controllers are used to drive the robots to explore and cover the environment as much as possible while avoiding obstacles.
The direction controllers are used to dynamically change the directions of the robots' local navigation force fields to eventually drive the crowd to a safe location.
These direction controllers are designed to take into account the real-time macro-states of humans (in a feedback manner) to guarantee the collective evacuation performance and include an adaptive component to compensate for unknown dynamics like dynamic obstacles.
We prove the stability of the proposed evacuation algorithm and perform a series of simulations involving unknown dynamic obstacles to validate the performance of the algorithm.

{\color{black}
Compared to existing works on robot-guided human evacuation, the main contribution of this work includes: (1) an extended two-scale model of crowd dynamics that takes into account the involvement of robots; (2) a novel controller synthesized for crowd evacuation under this model; (3) theoretical analysis of the new controller.}
Some preliminary results have been reported in our recent conference paper \cite{zheng2022multi}.
This work extends the results in \cite{zheng2022multi} by modeling human motions using social force models, which are more widely accepted, considering unknown dynamics in human models, introducing position control for robots to cover the environment, and considering obstacles in the simulated experiments.

The rest of the paper is organized as follows. 
In Section \ref{section:problem formulation}, we introduce the modeling of humans and robots and the problem formulation.
In Section \ref{section:control design}, we design position and direction controllers for the robots and prove their stability.
Section \ref{section:implementation} includes implementation details.
Section \ref{section:simulation} presents simulation studies to evaluate the algorithms.
Section \ref{section:conclusion} summarizes the contribution.

\section{Two-scale modeling of humans and robots}
\label{section:problem formulation}

In this work, we study the problem of using a small number of robots to guide a large human crowd in emergencies.
This is achieved by controlling the robots to dynamically generate navigation force fields to indirectly control the evolution of the crowd density.
We adopt a bottom-up modeling strategy to derive a two-scale model.
We start with heuristic microscopic descriptions, in the form of ordinary differential equations, for human behaviors, robot behaviors, and their interactions (including human-human and human-robot interactions).
Then, we derive an equivalent macroscopic description in the form of a hydrodynamic partial differential equation.

\subsection{Microscopic modeling}

\subsubsection{Modeling of humans and human-human interactions}
Denote by $\Omega$ the environment, which is assumed to be a convex bounded domain of $\mathbb{R}^2$ with a Lipschitz boundary $\partial\Omega$.
Let $N$ be the number of humans.
A human is treated as a particle that interacts with other humans/particles, the environment (like obstacles), and the robots; see Fig. \ref{fig:human view} for illustration.
In particular, we assume the following microscopic human models.
\begin{align}\label{eq:human model}
\begin{split}
    \frac{dx_j}{dt}& = v_j, \quad j=1,\dots,N \\
    \frac{dv_j}{dt}& = -\frac{1}{N}\nabla_{x_j}\sum_{k\neq j}U(x_j-x_k) \\
    & \quad+G(x_j,t)+{F}(x_j,t),
\end{split}
\end{align}
where $x_j(t)\in\Omega$ is the position of the $j$-th human, $v_j(t)\in\mathbb{R}^2$ is the velocity, and $\nabla_{x_j}$ represents the gradient in $x_j$.

Herein, $U$ is a pairwise interaction potential function corresponding to the classic social force model \cite{helbing1995social, helbing2002simulation}.
In general, the social force is repulsive within a short distance to maintain a social distance between humans, and attractive at a long distance to simulate the self-organization phenomenon (especially in emergencies) exhibited in a crowd.
Hence, the following potential is adopted \cite{helbing2002simulation}:
\begin{align}\label{eq:human social force}
    U(\xi)=C_re^{-\|\xi\|/\sigma_r}-C_ae^{-\|\xi\|/\sigma_a},
\end{align}
where $C_r$ and $\sigma_r$ ($C_a$ and $\sigma_a$) are positive constants representing the interaction strength and range of the repulsive (attractive) forces.
Typically, $C_r/\sigma_r>C_a/\sigma_a$ and $\sigma_r<\sigma_a$.
We assume that $U$ is known since it can be estimated using experimental data.
${F}$ is the navigation social force generated by the robots which will be discussed later.
$G$ represents all other unknown social forces from the environment, especially those from dynamic obstacles.
Note that humans instinctively avoid an obstacle when they encounter it. 
Therefore, it is reasonable to assume that an additional repulsive force (included as part of $G$) is generated and applied to a human when he/she is close to an obstacle and this additional force will vanish when the avoidance behavior terminates.
Therefore, we assume $G$ is a function of both $x$ and $t$ for sufficient generality.
Nevertheless, for performance analysis, we need to impose some mild assumptions on $G$.
We will assume $G$ is smooth, bounded, and slow-varying, at least after a finite time.

\begin{figure}
    \centering
    \includegraphics[width=0.5\columnwidth]{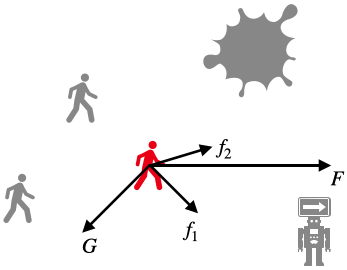}
    \caption{All the social forces applied to the red person: a navigation force $F$ from the robot, a repulsive force $G$ from the obstacle, and two repulsive forces from the other two humans.}
    \label{fig:human view}
\end{figure}

\subsubsection{Modeling of robots and human-robot interactions}
Now we discuss how to model ${F}$, the navigation social force generated by robots.
In our setup, every robot affects the human motions by providing a directed sign, such as an arrow pointing to the desired direction with instructions ``THIS WAY'' or ``EXIT''; see Fig. \ref{fig:navigation force} for illustration.
This setup is consistent with the signage systems widely applied in large buildings as safety messages to evacuees in case of emergencies.
Intuitively, humans tend to steer their movement according to the direction indicated by the robots when they are close to the robots and have less motivation to follow the sign when they are further away from the robots.
These behaviors are consistent with the impact of evacuation signage systems \cite{zhou2019guided} and have also been confirmed in experiments when signs are held by robots \cite{robinette2014assessment}.
Therefore, human-robot interactions can be modeled as follows.

\begin{figure}
    \centering
    \includegraphics[width=0.5\columnwidth]{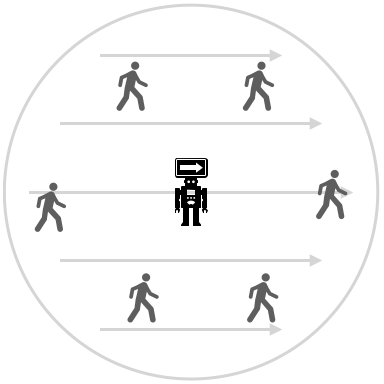}
    \caption{The navigation force field with round support generated by a robot.}
    \label{fig:navigation force}
\end{figure}

Let $n$ be the number of robots ($n\ll N$).
For the $i$-th robot, denote by $r_i(t)\in\Omega$ its position and $\theta_i(t)\in[0,2\pi)$ the direction of its arrow sign.
We emphasize that $\theta_i(t)$ is not the orientation of the robot and therefore $\theta_i(t)$ can be controlled independently of the robot's pose.
For simplicity, we also assume the robots are second-order integrators.
The robot models are given by:
\begin{align}\label{eq:robot model}
\begin{split}
    \ddot{r}_i(t) & =\tau_i(t),\quad i=1,\dots,n \\
    \dot{\theta}_i(t) & =\eta_i(t),
\end{split}
\end{align}
where $\tau_i(t)\in\mathbb{R}^2$ and $\eta_i(t)\in\mathbb{R}$ are the position and direction controllers of the robot that we need to design.

The navigation social force that is generated by the $i$-th robot and acts on the $j$-th human is modeled by:
\begin{align}\label{eq:human-robot force}
\begin{split}
    F_i(x_j,t) & =K(x_j-r_i(t),\theta_i(t)) \\
    & =\begin{bmatrix}
    \bar{K}(x_j-r_i(t))\cos\theta_i(t) \\
    \bar{K}(x_j-r_i(t))\sin\theta_i(t)
    \end{bmatrix}
\end{split}
\end{align}
where $\bar{K}(\xi)\in\mathbb{R}_{>0}$ for $\xi=[\xi_1~\xi_2]^T\in\mathbb{R}^2$ is a continuously differentiable kernel function.
It is reasonable to assume that $\bar{K}(\xi)$ has a compact support $\Omega_{\bar{K}}$, defined by
\begin{align*}
    \Omega_{\bar{K}}:=\{\xi|\bar{K}(\xi)\neq0\},
\end{align*}
which represents its range of influence, and that the magnitude of $\bar{K}(\xi)$ is strictly decreasing with respect to the distance $\|\xi\|$.
The form of $\bar{K}(\xi)$ can be designed based on experience or be estimated in advance from experiment data.
Intuitively, $F_i(x,t)$ is a locally effective force field whose direction is specified by $[\cos\theta_i(t)~\sin\theta_i(t)]^T$ and whose magnitude attains its maximum at the robot's position $r_i(t)$ and decreases with respect to the distance from $r_i(t)$.
The collective navigation force field generated by all robots is thus given by:
\begin{align}\label{eq:navigation force}
    {F}(x,t)=\sum_{i=1}^nF_i(x,t)=\sum_{i=1}^nK(x-r_i(t),\theta_i(t)).
\end{align}
The range of influence of the $i$-th robot is given by the support of $F_i$, denoted by $\Omega_i(t)$ and defined by
\begin{align*}
    \Omega_i(t):=\{x\in\Omega|F_i(x,t)\neq0\}.
\end{align*}
The area enclosed by the circle in Fig. \ref{fig:navigation force} is an illustration of $\Omega_i(t)$. 
The total range of influence is thus given by their union:
\begin{align}
    \Omega_r(t)=\bigcup_{i=1}^n\Omega_i(t).
\end{align}
The domain $\Omega_r(t)$ is critical to whether the robots are able to evacuate the complete human crowd. 

\subsection{Macroscopic modeling}

Now we derive a macroscopic model for the collective crowd behavior and human-robot interaction.

Let $f(x,v,t)$ be the distribution of $\{x_j,v_j\}_{j=1}^N$, given by 
\begin{align*}
    f(x,v,t)\approx\frac{1}{N}\sum_{j=1}^N\delta(x_j(t),v_j(t)),
\end{align*}
where $\delta(x,v)$ is the Dirac distribution.
It is well-known that as $N$ tends to infinity, $f$ satisfies the following kinetic equation \cite{carrillo2021mean}:
\begin{align}\label{eq:mean-field model}
    \partial_tf=-v\cdot\nabla_xf+\nabla_v\cdot[(\nabla U*\rho+G+{F})f]
\end{align}
where $\rho$ is the macroscopic density of the crowd
{\color{black}
given by the marginal distribution
\begin{align*}
    \rho(x,t):=\int_{\mathbb{R}^2} f(x,v,t)dv
\end{align*}
}and the $*$ notation denotes the convolution in the $x$ variable.
We point out that the macroscopic model \eqref{eq:mean-field model} and the microscopic model \eqref{eq:human model} are equivalent descriptions of human motions, as they share exactly the same set of parameters.

We define the macroscopic momentum by
\begin{align*}
    \rho u(x,t):=\int_{\mathbb{R}^2} vf(x,v,t)dv,
\end{align*}
which represents the crowd flux.
This also defines $u$ which represents the crowd velocity field.
One can use the following monokinetic closure (roughly speaking, particles at the same position have the same velocities)
\begin{align*}
    f(x,v,t)\sim\rho(x,t)\delta(v-u(x,t))
\end{align*}
to derive the following hydrodynamic limit from \eqref{eq:mean-field model} \cite{carrillo2021mean}:
\begin{align}
    \partial_t\rho= & -\nabla\cdot(\rho u), \label{eq:density}\\
    \partial_tu= & -(u\cdot\nabla)u-\nabla U*\rho+G+{F}. \label{eq:momentum}
\end{align}
The following zero-flux boundary condition is used to confine the crowd within the domain:
\begin{align}\label{eq:BC}
    \rho u\cdot\boldsymbol{n}=0 & \text{ on } \partial\Omega
\end{align}
where $\boldsymbol{n}$ is the unit inner normal to the boundary $\partial\Omega$.
The two equations \eqref{eq:density} and \eqref{eq:momentum} are referred to as the continuity and momentum equations, respectively.

\begin{remark}
This hydrodynamic limit for deterministic particles is justified by various papers; see \cite{borsche2018numerical, carrillo2021mean} and the references therein.
{\color{black}
Results regarding the well-posedness of \eqref{eq:density}-\eqref{eq:momentum} can be found in \cite{carrillo2017weak}.}
If the particles are driven by independent Wiener processes, other closures can be used to derive a different hydrodynamic limit \cite{borsche2018numerical}.
In this work, we will restrict ourselves to the deterministic case for simplicity.
Nevertheless, all the control design techniques and stability results can be generalized to the stochastic case.
\end{remark}

Finally, from the control perspective, the macroscopic model \eqref{eq:density}-\eqref{eq:momentum} shows how the crowd density $\rho$ is controlled by the navigation force field ${F}$ generated by the robots according to \eqref{eq:navigation force}.
The objective of evacuating the crowd to several safe locations can thus be formulated as a control problem of driving the crowd density toward a target density whose mass concentrates on those safe locations.

\begin{problem}[Crowd evacuation]
Given the dynamics of crowd density \eqref{eq:density}-\eqref{eq:momentum} and a target density $\rho_*(x)$, design control laws $\{\tau_i(t),\eta_i(t)\}_{i=1}^n$ for the robots \eqref{eq:robot model} to drive the crowd density $\rho(x,t)$ toward $\rho_*(x)$.
\end{problem}

\section{Control design}
\label{section:control design}
\subsection{Design philosophy}

First of all, we shall determine how the crowd density is eventually controlled by the robots' inputs $\{\tau_i(t),\eta_i(t)\}_{i=1}^n$.
Denote
\begin{align*}
    K_\xi(\xi,\theta):= & \frac{\partial}{\partial\xi}K(\xi,\theta) \\
     = & \begin{bmatrix}
    \frac{\partial}{\partial\xi_1}\bar{K}(\xi)\cos\theta & \frac{\partial}{\partial\xi_2}\bar{K}(\xi)\cos\theta \\
    \frac{\partial}{\partial\xi_1}\bar{K}(\xi)\sin\theta & \frac{\partial}{\partial\xi_2}\bar{K}(\xi)\sin\theta
    \end{bmatrix} \\
    K_\theta(\xi,\theta):= & \frac{\partial}{\partial\theta}K(\xi,\theta)=\begin{bmatrix}
    -\bar{K}(\xi)\sin\theta \\
    \bar{K}(\xi)\cos\theta
    \end{bmatrix},
\end{align*}
and for the $i$-th robot, denote
\begin{align*}
    & F_\xi^i(x,t)=K_\xi(x-r_i(t),\theta_i(t)), \\
    & F_\theta^i(x,t)=K_\theta(x-r_i(t),\theta_i(t)).
\end{align*}
Note that the support of $F_\xi^i(x,t)$ and $F_\theta^i(x,t)$ are both given by $\Omega_i(t)$.
Taking the time derivative of ${F}(x,t)$, we obtain:
\begin{align}\label{eq:navigation force dynamics}
\begin{split}
    \partial_t{F}(x,t) & =\sum_{i=1}^n\partial_tK(x-r_i(t),\theta_i(t)) \\
    & =\sum_{i=1}^n\Big(-F_\xi^i(x,t)\dot{r}_i(t)+F_\theta^i(x,t)\eta_i(t)\Big),
\end{split}
\end{align}
which shows how the navigation force field $F$ is controlled by the robot inputs $\{\tau_i(t),\eta_i(t)\}_{i=1}^n$ if we notice that $\dot{r}_i$ is uniquely determined by $\tau_i$ according to
\begin{align*}
    \dot{r}_i(t)=\dot{r}_i(0)+\int_0^t\tau_i(s)ds.
\end{align*}

Equations \eqref{eq:density}-\eqref{eq:navigation force dynamics} together constitute the complete control system of the crowd evacuation problem, which is an infinite-dimensional control problem with finite-dimensional control inputs. 
There are two sets of independent inputs: $\{\tau_i(t)\}_{i=1}^n$ and $\{\eta_i(t)\}_{i=1}^n$.
Our idea is to separate the position control and direction control problems.
The primary objective of position control is to fully explore the environment and cover as many humans as possible.
The primary objective of direction control is to dynamically adjust the arrow signs to generate a stabilizing force field $F$.

\subsection{Design of position controllers}
We design position controllers $\{\tau_i\}_{i=1}^n$ for the robots \eqref{eq:robot model} such that their total range of influence $\Omega_r(t)$ covers the environment $\Omega$ as much as possible.
For this purpose, we adopt a well-known sensor coverage algorithm based on potential fields, which is given by \cite{howard2002mobile}:
\begin{align}\label{eq:position controller}
    \tau_i=(f_i-\nu\dot{r})/m_i
\end{align}
where $\nu$ is a viscosity coefficient, $m_i$ is the robot's mass, and $f_i$ is the sum of all repulsive forces due to other robots and obstacles given by
\begin{align*}
    f_i=-\nabla_{r_i}\Big(\sum_{k\neq i}\frac{k_r}{\|r_i-r_k\|}+\sum_{l\in\mathcal{O}_i}\frac{k_o}{\|r_i-s_l\|}\Big)
\end{align*}
where $k_r$ and $k_o$ are two positive constants, $\mathcal{O}_i$ is the set of all obstacles seen by the $i$-th robot, and $s_l$ is the position of the $l$-th obstacle.
It is known that under this design, the robots will usually exhibit satisfactory coverage of the workspace \cite{howard2002mobile}.

\subsection{Design of direction controllers}
Assuming the position controllers $\{\tau_i\}_{i=1}^n$ have been determined, we further design direction controllers $\{\eta_i\}_{i=1}^n$ for the robots \eqref{eq:robot model} such that their generated navigation force field ${F}$ according to \eqref{eq:navigation force} stabilizes the crowd density $\rho$.
The main technique of our design includes: (1) using density feedback and a series of backstepping designs to determine the robot inputs $\{\eta_i\}_{i=1}^n$, and (2) combining an adaptive control law to compensate for the unknown force field $G$ online.

The backstepping design recursively constructs a sequence of stabilizing functions to stabilize systems with lower-triangular structures \cite{krstic1995nonlinear}.
In our case, the backstepping design consists of three steps.

\subsubsection{Density feedback design}
In the first step, we design a virtual stabilizing control input for \eqref{eq:density} as
\begin{align}\label{eq:ud}
\begin{split}
    & u_d=-k_\rho\nabla(\rho-\rho_*), 
\end{split}
\end{align}
where $k_\rho(x,t)>0$ is an adjustable control gain.
Intuitively, this virtual control input \eqref{eq:ud} compares the difference between $\rho$ and $\rho_*$ and uses its negative gradient to drive the crowd.

Control algorithms like \eqref{eq:ud} are called density feedback because they explicitly use the real-time density $\rho$ as feedback \cite{elamvazhuthi2019mean, zheng2021transporting}.
The stabilizing property of $u_d$ can be proved by considering a Lyapunov functional:
\begin{align}\label{eq:Lyapunov V1}
    V_1(t)=\int_\Omega\frac{1}{2}(\rho-\rho_*)^2dx.
\end{align}
By substituting \eqref{eq:ud} into \eqref{eq:density} and using the divergence theorem, the zero-flux boundary condition \eqref{eq:BC}, and the Poincar\'e inequality \cite{evans1998partial}
\begin{align*}
    \frac{dV_1}{dt}\Bigm|_{u=u_d} & =\int_\Omega(\rho-\rho_*)\partial_t\rho dx \\
    & =\int_\Omega(\rho-\rho_*)\nabla\cdot\big(k_\rho\rho\nabla(\rho-\rho_*)\big) dx \\
    & =-\int_\Omega k_\rho\rho\|\nabla(\rho-\rho_*)\|^2 dx \\
    & \leq-\int_\Omega k_\rho C_*\rho(\rho-\rho_*)^2 dx
\end{align*}
where $C_*>0$ is the Poincar\'e constant.

The fact that $k_\rho$ can be a function of both $x$ and $t$ provides significant flexibility in practice because it allows us to locally adjust the magnitude of crowd flows.
For example, we can enforce the crowd velocity field $u$ to have a constant magnitude everywhere by choosing $k_\rho$ to be:
\begin{align*}
    & k_\rho=\frac{k_\rho'}{\|\nabla(\rho-\rho_*)\|},
\end{align*}
for some positive constant $k_\rho'$.

\subsubsection{Adaptive control design}
In the second step, we need to deal with the slow-varying unknown force field $G(x,t)$.
We propose to approximate $G$ with an artificial neural network \cite{ge2013stable} according to
\begin{align*}
    G(x,t)=\sum_{k=1}^m
    \begin{bmatrix}
        \phi_{1k}(x)w_{1k}(t) \\
        \phi_{2k}(x)w_{2k}(t)
    \end{bmatrix}
    +
    \begin{bmatrix}
        \epsilon_1(x,t) \\
        \epsilon_2(x,t)
    \end{bmatrix}
\end{align*}
where $\{\phi_{1m}(x),\phi_{2m}(x)\in\mathbb{R}\}_{k=1}^m$ is a collection of smooth and bounded basis functions, $\{w_{1m}(t),w_{2m}(t)\in\mathbb{R}\}_{k=1}^m$ is the collection of the ideal but unknown weights, and $\{\epsilon_1(x,t),\epsilon_2(x,t)\}$ are the ideal approximation errors.
According to the universal approximation theorem of artificial neural networks, we can make $\{\epsilon_1,\epsilon_2\}$ arbitrarily small by increasing $m$, the number of basis functions \cite{ge2013stable}.
Denote
\begin{align*}
    & \phi=
    \begin{bmatrix}
        \phi_{11} & \dots & \phi_{1m} & 0 & \dots & 0 \\
        0 & \dots & 0 & \phi_{21} & \dots & \phi_{2m}
    \end{bmatrix}^T, \\
    & w=[w_{11} ~ \dots ~ w_{1m} ~ w_{21} ~ \dots ~ w_{2m}]^T, \\
    & \epsilon=[\epsilon_1 ~ \epsilon_2]^T.
\end{align*}
Then, we can write
\begin{align}\label{eq:ANN}
    G(x,t)=\phi(x)^Tw(t)+\epsilon(x,t).
\end{align}
Let $\hat{w}(t)$ be the estimate of $w(t)$.
Define
\begin{align}
\begin{split}
    & \tilde{\rho}(x,t)=\rho(x,t)-\rho_*(x), \\
    & \tilde{u}(x,t)=u(x,t)-u_d(x,t), \\
    & \tilde{w}(t)=\hat{w}(t)-w(t).
\end{split}
\end{align}
We want $(\tilde{\rho},\tilde{u},\tilde{w})\to0$.
For this purpose, consider an augmented control Lyapunov functional based on \eqref{eq:Lyapunov V1}:
\begin{align}\label{eq:Lyapunov V2}
    V_2(t)=\frac{1}{2}\int_\Omega(\tilde{\rho}^2+\|\tilde{u}\|^2)dx+\frac{1}{2}\tilde{w}^T\Gamma^{-1}\tilde{w},
\end{align}
where $\Gamma$ is a positive-definite constant matrix.
We have
\begin{align}\label{eq:dV2}
\begin{split}
    \frac{dV_2}{dt}= & \int_\Omega-\nabla\tilde{\rho}\cdot(k_\rho\rho\nabla\tilde{\rho}-\rho\tilde{u})+\tilde{u}^T\Big(-(u\cdot\nabla)u \\
    & -\nabla U*\rho+\phi^Tw+\epsilon+{F}-\partial_tu_d\Big)dx \\
    & +\tilde{w}^T\Gamma^{-1}(\dot{\hat{w}}-\dot{w}) \\
    = & \int_\Omega-k_\rho C_*\tilde{\rho}^2+\tilde{u}^T\Big(\rho\nabla\tilde{\rho}-(u\cdot\nabla)u \\
    & -\nabla U*\rho+\phi^T\hat{w}+\epsilon+{F}-\partial_tu_d\Big) dx \\
    & -\int_\Omega\tilde{u}^T\phi^T\tilde{w}dx+\tilde{w}^T\Gamma^{-1}(\dot{\hat{w}}-\dot{w}) \\
    = & \int_\Omega-k_\rho C_*\tilde{\rho}^2+\tilde{u}^T\Big(\rho\nabla\tilde{\rho}-(u\cdot\nabla)u \\
    & -\nabla U*\rho+\phi^T\hat{w}+\epsilon+{F}-\partial_tu_d\Big) dx \\
    & +\tilde{w}^T\Gamma^{-1}\Big(\dot{\hat{w}}-\dot{w}-\Gamma\int_\Omega\phi\tilde{u}dx\Big).
\end{split}
\end{align}

This motivates us to design a virtual stabilizing navigation force field ${F}_d(x,t)$ and the update law for $\hat{w}(t)$ as
\begin{align}
\begin{split} \label{eq:Fd}
    {F}_d & =-k_u\tilde{u}-\rho\nabla\tilde{\rho}+(u\cdot\nabla)u \\
    & \quad+\nabla U*\rho-\phi^T\hat{w}+\partial_tu_d
\end{split} \\
    \dot{\hat{w}} & =\Gamma\Big(\int_\Omega\phi\tilde{u}dx-k_w\hat{w}\Big), \label{eq:update law}
\end{align}
where $k_u,k_w>0$ are both constants.

\subsubsection{Robot control design}
In the last step, we will determine the direction controllers $\{\eta_i\}_{i=1}^n$.
For this purpose, define $\tilde{F}(x,t)={F}(x,t)-{F}_d(x,t)$.
Consider an augmented control Lyapunov functional based on \eqref{eq:Lyapunov V2}:
\begin{align}\label{eq:Lyapunov V3}
    V_3(t)=\frac{1}{2}\int_\Omega(\tilde{\rho}^2+\|\tilde{u}\|^2+\|\tilde{F}\|^2)dx+\frac{1}{2}\tilde{w}^T\Gamma^{-1}\tilde{w}.
\end{align}
Taking its derivative and substituting \eqref{eq:navigation force dynamics}, we obtain
\begin{align}\label{eq:dV3}
\begin{split}
    \frac{dV_3}{dt} & =\frac{dV_2}{dt}+\int_\Omega\tilde{F}^T(\tilde{u}+\partial_tF-\partial_t{F}_d)dx \\
    & =\frac{dV_2}{dt}+\int_\Omega\tilde{F}^T\Big(\tilde{u}-\partial_t{F}_d \\
    & \quad+\sum_{i=1}^n(-F_\xi^i\dot{r}_i+F_\theta^i\eta_i)\Big)dx \\ 
    & =\frac{dV_2}{dt}+\int_\Omega\tilde{F}^T\Big(\tilde{u}-\partial_t{F}_d-\sum_{i=1}^nF_\xi^i\dot{r}_i\Big)dx \\ 
    & \quad +\sum_{i=1}^n\eta_i\int_{\Omega_i(t)}\tilde{F}^TF_\theta^idx,
\end{split}
\end{align}
where we note that in the last term the domain of integration has been replaced by $\Omega_i(t)$ (the range of influence of the $i$-th robot) since the value of $F_\theta^i$ is zero outside $\Omega_i(t)$. 
This motivates us to design $\eta_i$ as follows.
Let $k_\eta>0$ be a constant and select scalar functions $\{\beta_i(t)\}_{i=1}^n$ such that, for all $t$,
\begin{align}\label{eq:controllability}
    \begin{cases}
        \beta_i(t)>0, \text{ if } \int_{\Omega_i(t)}\tilde{F}^TF_\theta^idx\neq0 \\
        \beta_i(t)=0, \text{ if } \int_{\Omega_i(t)}\tilde{F}^TF_\theta^idx=0 \\
        \sum_{i=1}^n\beta_i(t)=1.
    \end{cases}
\end{align}
Then, design $\eta_i$ as
\begin{align}\label{eq:direction controller}
\begin{split}
    \eta_i(t)= & -\frac{\beta_i\int_\Omega\tilde{F}^T(\tilde{u}-\partial_t{F}_d-\sum_{i=1}^nF_\xi^i\dot{r}_i+\frac{k_\eta}{\beta_i}\tilde{F})dx}{\int_{\Omega_i(t)}\tilde{F}^TF_\theta^idx}
\end{split}
\end{align}
if $\beta_i(t)>0$, and $\eta_i(t)=0$ if $\beta_i(t)=0$.

The condition \eqref{eq:controllability} always has a solution for $\{\beta_i(t)\}_{i=1}^n$ if at least one of the integrals $\{\int_{\Omega_i(t)}\tilde{F}^TF_\theta^idx\}_{i=1}^n$ is nonzero; in other words, if there is at least one robot that can affect the error term $\tilde{F}$ by changing its direction $\theta_i$, which may be interpreted as a ``controllability'' condition.
A sufficient condition to ensure a solution always exists to \eqref{eq:controllability} is that $\Omega\subset\Omega_r(t)$, i.e., the environment is completely covered by the robots.
In this way, there always exists at least one robot that can affect the error term $\tilde{F}$.

\subsection{Stability analysis}

We can prove that the crowd density will converge to a small neighborhood of the target density under the following assumptions.

\begin{assumption}\label{assump:small G}
There exist constants $C_1,C_2>0$ such that $\|G(x,t)\|_{L^\infty(\Omega)}\leq C_1$ and $\|\partial_tG(x,t)\|_{L^\infty(\Omega)}\leq C_2$ for all $t$.
\end{assumption}

\begin{assumption}\label{assump:coverage}
There exists a constant $T>0$ such that $\Omega\subset\Omega_r(t)$ for all $t>T$.    
\end{assumption}

Assumption \ref{assump:small G} requires that the unknown dynamics $G$ is bounded and slow-varying, which is a common assumption for unknown dynamics.
Assumption \ref{assump:coverage} requires that the total range of influence of the robots is able to completely cover the entire environment $\Omega$ after a finite time.
Since the coverage performance of the algorithm \eqref{eq:position controller} has been verified through various experiments in the literature \cite{howard2002mobile}, this is an expected consequence as long as there are enough robots.
While a rigorous derivation of a sufficient number of robots is complicated, an overestimate is that when all the robots are evenly spread out, we have $\Omega\subset\Omega_r$.

\begin{theorem}\label{thm:stability}
Consider the robot models \eqref{eq:robot model}, the navigation force generated by the robots \eqref{eq:navigation force}, and the crowd dynamics \eqref{eq:density}-\eqref{eq:momentum}.
Let $\{\tau_i(t)\}_{i=1}^n$ and $\{\eta_i(t)\}_{i=1}^n$ be given by \eqref{eq:position controller} and \eqref{eq:direction controller}, respectively, where $u_d$, ${F}_d$, and $\hat{w}$ are computed using \eqref{eq:ud}, \eqref{eq:Fd}, and \eqref{eq:update law}, respectively.
Then under Assumptions \ref{assump:small G} and \ref{assump:coverage}, $(\|\tilde{\rho}(x,t)\|_{L^2(\Omega)},\|\tilde{u}(x,t)\|_{L^2(\Omega)},\|\tilde{F}\|_{L^2(\Omega)},\|\tilde{w}(t)\|)$ are all uniformly ultimately bounded.
\end{theorem}

\begin{proof}
First, consider the Lyapunov functional \eqref{eq:Lyapunov V2}.
By substituting \eqref{eq:Fd} and \eqref{eq:update law} into \eqref{eq:dV2}, we obtain
\begin{align*}
    \frac{dV_2}{dt}\leq & \int_\Omega-k_\rho C_*\tilde{\rho}^2-k_u\|\tilde{u}\|^2+\tilde{u}^T\epsilon dx \\
    & -\tilde{w}^T(k_w\hat{w}+\Gamma^{-1}\dot{w}).
\end{align*}
Let $k_u=k_{u1}+k_{u2}$.
We have
\begin{align*}
    -k_{u2}\|\tilde{u}\|^2+\tilde{u}^T\epsilon=-k_{u2}\Big\|\tilde{u}-\frac{\epsilon}{2k_{u2}}\Big\|^2+\frac{\|\epsilon\|^2}{4k_{u2}}\leq\frac{\|\epsilon\|^2}{4k_{u2}}.
\end{align*}
Also,
\begin{align*}
    & -\tilde{w}^T(k_w\hat{w}+\Gamma^{-1}\dot{w}) \\
    = & -k_w\|\tilde{w}\|^2+k_w\|\tilde{w}\|\|w+(k_w\Gamma)^{-1}\dot{w}\| \\
    \leq & -\frac{k_w}{2}\|\tilde{w}\|^2+\frac{k_w}{2}\|w+(k_w\Gamma)^{-1}\dot{w}\|^2 \\
    \leq & -\frac{k_w}{2}\|\tilde{w}\|^2+\frac{k_w}{2}\|w\|^2+\frac{\|\Gamma^{-1}\|}{2}\|\dot{w}\|^2.
\end{align*}
Furthermore, under the assumptions of $G$, there exist positive constants $C_3,C_4,C_5$ (depending on $C_1,C_2$ and $\phi$) such that $\|\epsilon(x,t)\|_{L^2(\Omega)}\leq C_3$, $\|w(t)\|\leq C_4$, and $\|\dot{w}(t)\|\leq C_5$ for all $t$.
Hence, we obtain
\begin{align}\label{eq:dV2 closed-loop}
\begin{split}
    \frac{dV_2}{dt} & \leq\int_\Omega-k_\rho C_*\tilde{\rho}^2-k_{u1}\|\tilde{u}\|^2+\frac{\|\epsilon\|^2}{4k_{u2}}dx \\
    & \quad-\frac{k_w}{2}\|\tilde{w}\|^2+\frac{k_w}{2}\|\tilde{w}\|^2+\frac{\|\Gamma^{-1}\|}{2}\|\tilde{w}\|^2 \\
    & =-\bar{k}_\rho C_*\|\tilde{\rho}\|_{L^2(\Omega)}^2-k_{u1}\|\tilde{u}\|_{L^2(\Omega)}^2-\frac{k_w}{2}\|\tilde{w}\|^2 \\
    & \quad+\frac{1}{4k_{u2}}\|\epsilon\|_{L^2(\Omega)}^2+\frac{k_w}{2}\|\tilde{w}\|^2+\frac{\|\Gamma^{-1}\|}{2}\|\tilde{w}\|^2 \\
    & \leq-\bar{k}_\rho C_*\|\tilde{\rho}\|_{L^2(\Omega)}^2-k_{u1}\|\tilde{u}\|_{L^2(\Omega)}^2 \\
    & \quad-\frac{k_w}{2}\|\tilde{w}\|^2+C_G,
\end{split}
\end{align}
where $\bar{k}_\rho=\inf_{(x,t)}k_\rho(x,t)$ and
\begin{align*}
    C_G=\frac{C_3^2}{4k_{u2}}+\frac{k_wC_4^2}{2}+\frac{\|\Gamma^{-1}\|C_5^2}{2}.
\end{align*}

Next, consider the Lyapunov functional \eqref{eq:Lyapunov V3}.
Since $\Omega\subset\Omega_r(t)$ for all $t>T$, \eqref{eq:controllability} always has a solution for all $t>T$.
By substituting \eqref{eq:direction controller} and \eqref{eq:dV2 closed-loop} into \eqref{eq:dV3}, we obtain that
\begin{align}\label{eq:dV3 closed-loop}
\begin{split}
    \frac{dV_3}{dt} & =\frac{dV_2}{dt}-k_\eta\|\tilde{F}\|_{L^2(\Omega)}^2 \\
    & \leq-\bar{k}_\rho C_*\|\tilde{\rho}\|_{L^2(\Omega)}^2-k_{u1}\|\tilde{u}\|_{L^2(\Omega)}^2 \\
    & \quad-k_\eta\|\tilde{F}\|_{L^2(\Omega)}^2-\frac{k_w}{2}\|\tilde{w}\|^2+C_G.
\end{split}
\end{align}
Clearly, $\frac{dV_2}{dt}$ is negative definite whenever
\begin{align*}
    & \|\tilde{\rho}\|_{L^2(\Omega)}>\sqrt{C_G/(\bar{k}_\rho C_*)}, \quad \|\tilde{u}\|_{L^2(\Omega)}>\sqrt{C_G/k_{u1}}, \\
    & \|\tilde{F}\|_{L^2(\Omega)}>\sqrt{C_G/k_\eta}, \quad \text{or}\quad \|\tilde{w}\|>\sqrt{2C_G/k_w},
\end{align*}
which proves that $(\|\tilde{\rho}(x,t)\|_{L^2(\Omega)},\|\tilde{u}(x,t)\|_{L^2(\Omega)},\|\tilde{F}\|_{L^2(\Omega)},$ $|\tilde{w}(t)|)$ are all uniformly ultimately bounded.
\end{proof}

We can further prove that if the unknown dynamics $G$ vanishes after a finite time, then the crowd density will converge exponentially toward the target density.

\begin{assumption}\label{assump:G=0}
There exists a constant $T>0$ such that $G(x,t)\equiv0$ for all $t>T$.
\end{assumption}
  
Note that this assumption does not exclude the situation where dynamic obstacles are present.
As what we explained before, the additional repulsive forces from the obstacles will vanish once the humans are far away from them.
Therefore, as long as all humans no longer encounter obstacles after a finite time, the additional repulsive forces from the obstacles will disappear in $G$ after that.
While it is difficult to state a rigorous sufficient condition such that humans no longer encounter obstacles after a finite time, a heuristic condition is that all obstacles are convex and ``small'' (so that humans won't get trapped by the obstacles) and that no obstacles are present at the target locations of the evacuation task (so that humans can simply pass by them).
In this case, the adaptive control law \eqref{eq:update law} is no longer needed and in \eqref{eq:Fd} we set $\hat{w}\equiv0$.

\begin{theorem}\label{thm:asymptotic stability}
Consider the robot models \eqref{eq:robot model}, the navigation force generated by the robots \eqref{eq:navigation force}, and the crowd dynamics \eqref{eq:density}-\eqref{eq:momentum}.
Let $\{\tau_i(t)\}_{i=1}^n$ and $\{\eta_i(t)\}_{i=1}^n$ be given by \eqref{eq:position controller} and \eqref{eq:direction controller}, respectively, where $u_d$ and ${F}_d$ are computed using \eqref{eq:ud} and \eqref{eq:Fd} (where $\hat{w}\equiv0$), respectively.
Then under Assumptions \ref{assump:coverage} and \ref{assump:G=0}, $(\|\tilde{\rho}(x,t)\|_{L^2(\Omega)},\|\tilde{u}(x,t)\|_{L^2(\Omega)},\|\tilde{F}\|_{L^2(\Omega)})\to0$ exponentially.
\end{theorem}

\begin{proof}
Consider a Lyapunov functional
\begin{align*}
    V_4(t)=\frac{1}{2}\int_\Omega(\tilde{\rho}^2+\|\tilde{u}\|^2+\|\tilde{F}\|^2)dx.
\end{align*}
Using \eqref{eq:dV3 closed-loop}, and removing all the terms related to $w$ and $\epsilon$, we obtain
\begin{align*}
    \frac{dV_4}{dt} \leq-\bar{k}_\rho C_*\|\tilde{\rho}\|_{L^2(\Omega)}^2-k_{u1}\|\tilde{u}\|_{L^2(\Omega)}^2-k_\eta\|\tilde{F}\|_{L^2(\Omega)}^2,
\end{align*}
which proves the exponential stability.
\end{proof}

In fact, we expect that the conclusions in Theorems \ref{thm:stability} and \ref{thm:asymptotic stability} still hold if the robots are able to eventually cover the entire crowd, as opposed to the entire environment.
(This will be verified in the subsequent experiments.)
Formally, denote by $\Omega_\rho(t)\subset\Omega$ the support of the density $\rho(x,t)$, i.e.,
\begin{align}
    \Omega_\rho(t) = \{x|\rho(x,t)>0\}.
\end{align}
Roughly speaking, $\Omega_\rho(t)$ is the region where humans present.
We expect the same conclusion to hold if the condition $\Omega\subset\Omega_r(t)$ is replaced by $\Omega\subset\Omega_\rho(t)$.
This condition can significantly reduce the necessary number of robots, especially when the target density only concentrates on a small region.
However, one needs to design more complicated position controllers than \eqref{eq:position controller} since the objective, in this case, is to first sufficiently spread out to explore the environment and then shrink to cover only the human crowd.
This more complicated scenario is left for future study.

\section{Implementation}
\label{section:implementation}

In this section, we discuss how to implement the proposed algorithms in practice.
The current evacuation system is \textit{centralized}, meaning that a communication center is needed to collect information, estimate the required global states, and pass this information on to robots.

{\color{black}
We assume that experiments have been performed in advance to estimate the coefficients $C_r,\sigma_r,C_a,\sigma_a$ of the social force model \eqref{eq:human social force} of human-human interaction and the kernel function $\bar{K}$ in the social force model \eqref{eq:human-robot force} of human-robot interaction.
Relevant result on determining these models can be found in \cite{helbing2002simulation, nayyar2023learning}.
The number of robots needed can be determined by dividing the area of the environment $\Omega$ by the area of influence of a single robot $\Omega_{\bar{K}}$, both of which are known beforehand.
}

The position controllers \eqref{eq:position controller} require that the robots have appropriate sensors to detect the positions of nearby robots and obstacles.
The direction controllers \eqref{eq:direction controller} depend on the global information $\rho(x,t)$ and $u(x,t)$.
To obtain these macro-states, we need to assume that all the human positions and velocities $\{x_j(t),v_j(t)\}_{j=1}^N$ are available to the communication center.
This information can be collected either by installing a surveillance system for the environment in advance or by the robots since they will eventually exhibit satisfactory coverage of the environment under the position controllers \eqref{eq:position controller}.

\textit{Density Estimation.} The human positions $\{x_j(t)\}_{j=1}^N$ can be viewed as a collection of samples of the unknown (probability) density function $\rho(x,t)$ at every $t$.
Hence, $\rho(x,t)$ can be estimated from $\{x_j(t)\}_{j=1}^N$ using classical density estimation algorithms such as kernel density estimation which is adopted in this work.
Specifically, given the human positions $\{x_j(t)\}_{j=1}^N$, we construct a density estimate $\hat{\rho}(x,t)$ according to:
\begin{align}\label{eq:KDE}
    \hat{\rho}(x,t) = \frac{1}{N h^2} \sum_{j=1}^N H\left(\frac{1}{h}\left(x-X_{j}(t)\right)\right),
\end{align}
where $H(x)$ is a radially symmetric kernel function that satisfies
\begin{align*}
    \int H(x)=1,\quad \int xH(x)=0, \quad 0<\int x^2H(x)<+\infty,
\end{align*}
and $h$ is the bandwidth, usually chosen as a function of $N$, i.e., $h=h_N$, such that
\begin{align*}
    \lim_{N\to\infty} h_N=0, \quad \lim_{N\to\infty}Nh_N=\infty.
\end{align*}
The most common choice of $H(x)$ is the Gaussian kernel given by:
\begin{align}
    H(x)=\frac{1}{2\pi}\exp\big(-\frac{1}{2}x^Tx\big).
\end{align}

\begin{remark}
The performance of kernel density estimation largely depends on the selection of the bandwidth $h$.
A computationally more expensive but more accurate approach is to use the density filtering algorithms reported in \cite{zheng2020pde, zheng2021distributedmean}.
These filtering algorithms exploit the density dynamics \eqref{eq:density} to improve the estimates using kernel density estimation.
The estimates can be proven to be convergent and the performance is very robust to different selections of bandwidth.
\end{remark}

\textit{Velocity Field Estimation.} 
The crowd velocity field $u(x,t)$ satisfies $u(x_j(t),t)=v_j(t)$, i.e., when projecting to the position of $j$-th human, the value of $u$ is exactly the velocity of that human.
In other words, $\{x_j(t),v_j(t)\}_{j=1}^N$ can be viewed as a set of observations of the unknown function $u(x,t)$.
Hence, $u(x,t)$ can be estimated from $\{x_j(t),v_j(t)\}_{j=1}^N$ using function fitting algorithms such as linear and polynomial interpolation.

{\color{black}
The complete evacuation algorithm is summarized in Algorithm \ref{algorithm:evacuation}, where spatial integration can be approximated by summation after spatial discretization.
We shall emphasize that the macroscopic equations \eqref{eq:density}-\eqref{eq:momentum} are only used for control design and stability analysis. 
We do not need to solve any partial differential equations when implementing the control algorithms.
As shown in Algorithm \ref{algorithm:evacuation}, most of the computation can be done at the center, such as calculating the macroscopic states.
This part of the computation is largely unaffected by the number of humans and robots and depends mainly on the resolution of spatial discretization.
The only calculation that is dependent on the number of robots is the computation of the low-level robot controllers $\tau_i(t)$ using \eqref{eq:position controller} and $\eta_i(t)$ using \eqref{eq:direction controller}.
In the case of $N=250$, $n=16$, and $30\times30$ spatial discretization, the average time taken to complete one iteration of Algorithm \ref{algorithm:evacuation} was 40 ms, tested on a MacBook Pro (15-inch, 2018) with a 2.2 GHz 6-Core Intel Core i7 processor and 16 GB DDR4 RAM. 
This indicates that the algorithm has the potential to be implemented in real-time. 
The primary obstacle would be to efficiently gather real-time human positional data to enable real-time control.
To address this, we can perform the data collection at a lower rate, depending on the data collection system, while updating the robot controllers at a higher rate.}

\begin{algorithm}[h]
\SetAlgoLined
Determine the target density $\rho_*(x)$\;
Prescribe a threshold $\gamma$ for $\|\tilde{\rho}\|_{L^2(\Omega)}$\;
Initialize the adaptive weights $\hat{w}(0)$\;
\While{$\|\tilde{\rho}(\cdot,t)\|_{L^2(\Omega)}>\gamma$}{
    {\color{black}
    \For{the center}{
        Collect the crowd information $\{x_j(t),v_j(t)\}_{j=1}^N$ to estimate $\rho(x,t)$ and $u(x,t)$\;
        Collect the robots' states $\{r_i(t),\theta_i(t)\}_{i=1}^n$ and update $\hat{w}(t)$ using \eqref{eq:update law}\;
        Compute $u_d(x,t)$ using \eqref{eq:ud}\;
        Compute ${F}_d(x,t)$ using \eqref{eq:Fd}\;
        Send the above information to all robots\;
    }
    \For{each robot $i$}{
         Compute $\tau_i(t)$ using \eqref{eq:position controller} to update its position $r_i(t)$\;
         Compute $\eta_i(t)$ using \eqref{eq:direction controller} to update its direction $\theta_i(t)$\;
    }
    }
}
\caption{Robot-Guided Evacuation}
\label{algorithm:evacuation}
\end{algorithm}

\section{Simulation study}
\label{section:simulation}
We conducted simulation studies based on MATLAB to validate the evacuation algorithms presented in this work.
The simulation was agent-based, meaning that all humans and robots were simulated using their microscopic models (ordinary differential equations), and it did not involve solving any macroscopic models (partial differential equations).

\begin{figure}[t]
    \centering
    \includegraphics[width=0.5\columnwidth]{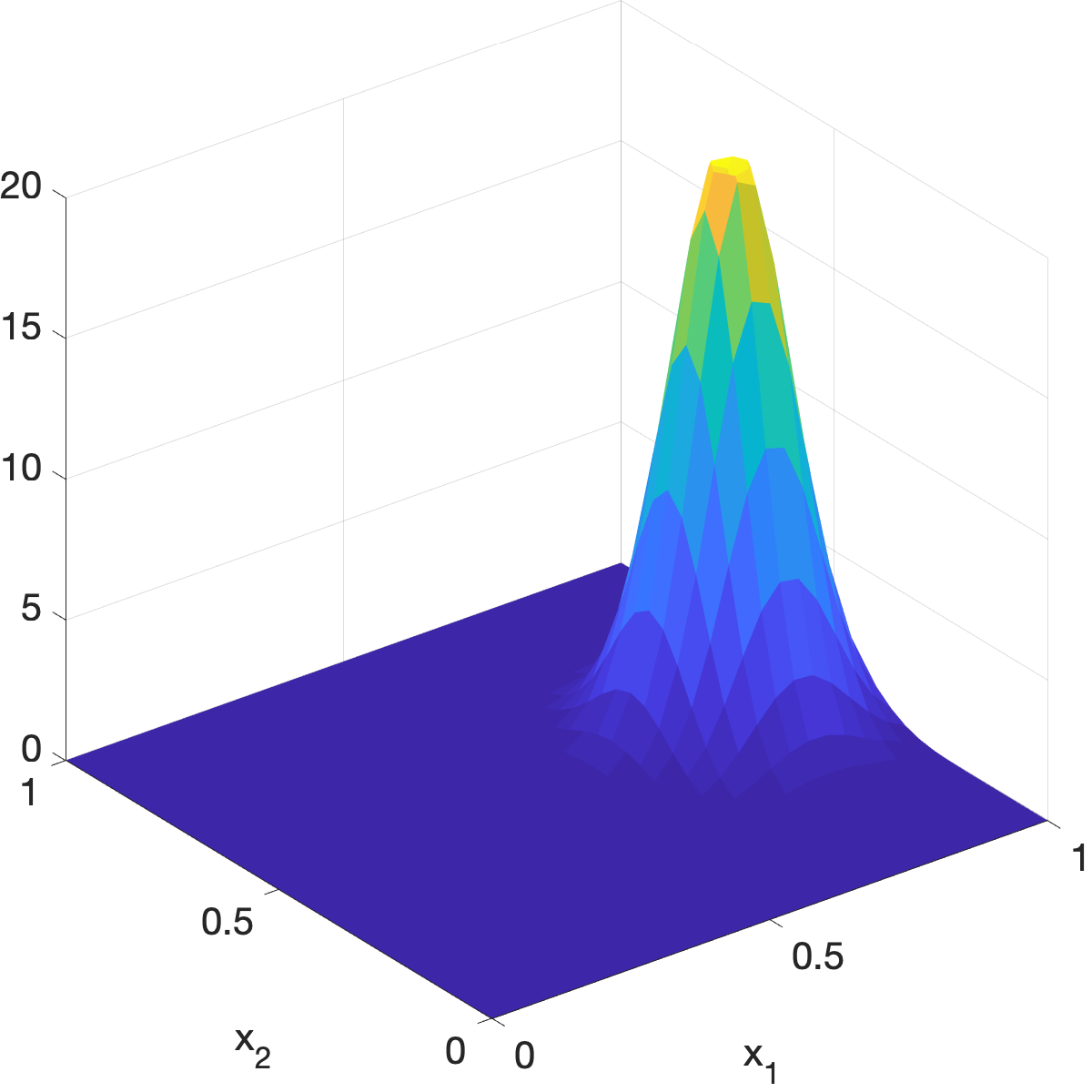}
    \caption{The target density $\rho_*(x)$.}
    \label{fig:target density}
\end{figure}

\begin{figure}[t]
    \centering
    \includegraphics[width=0.75\columnwidth]{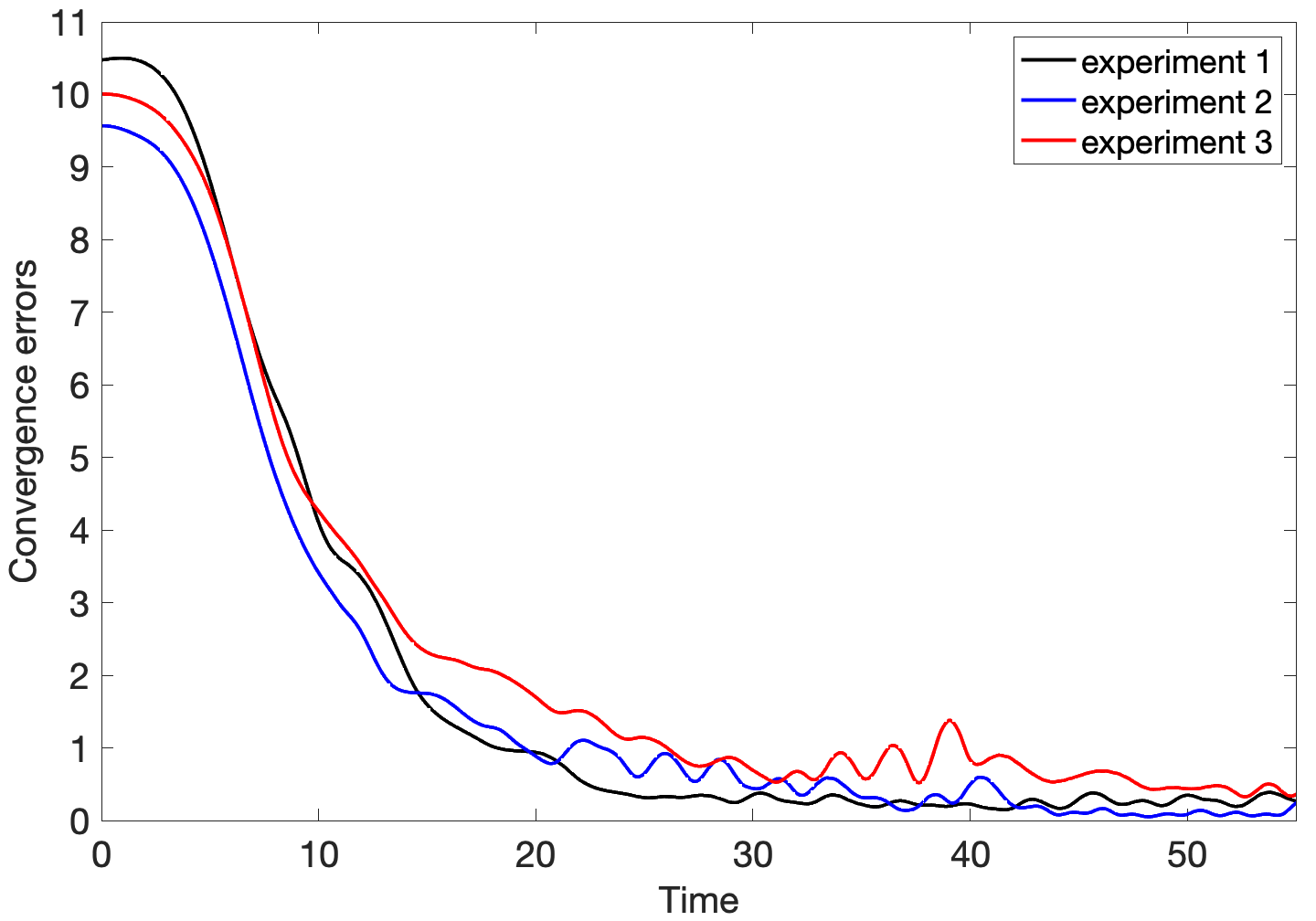}
    \caption{The density convergence errors $\|\tilde{\rho}(\cdot,t)\|_{L^2(\Omega)}^2$.}
    \label{fig:error}
\end{figure}

\begin{figure*}[t]
    \centering
    \begin{subfigure}[b]{0.19\textwidth}
        \centering
        \includegraphics[width=\textwidth]{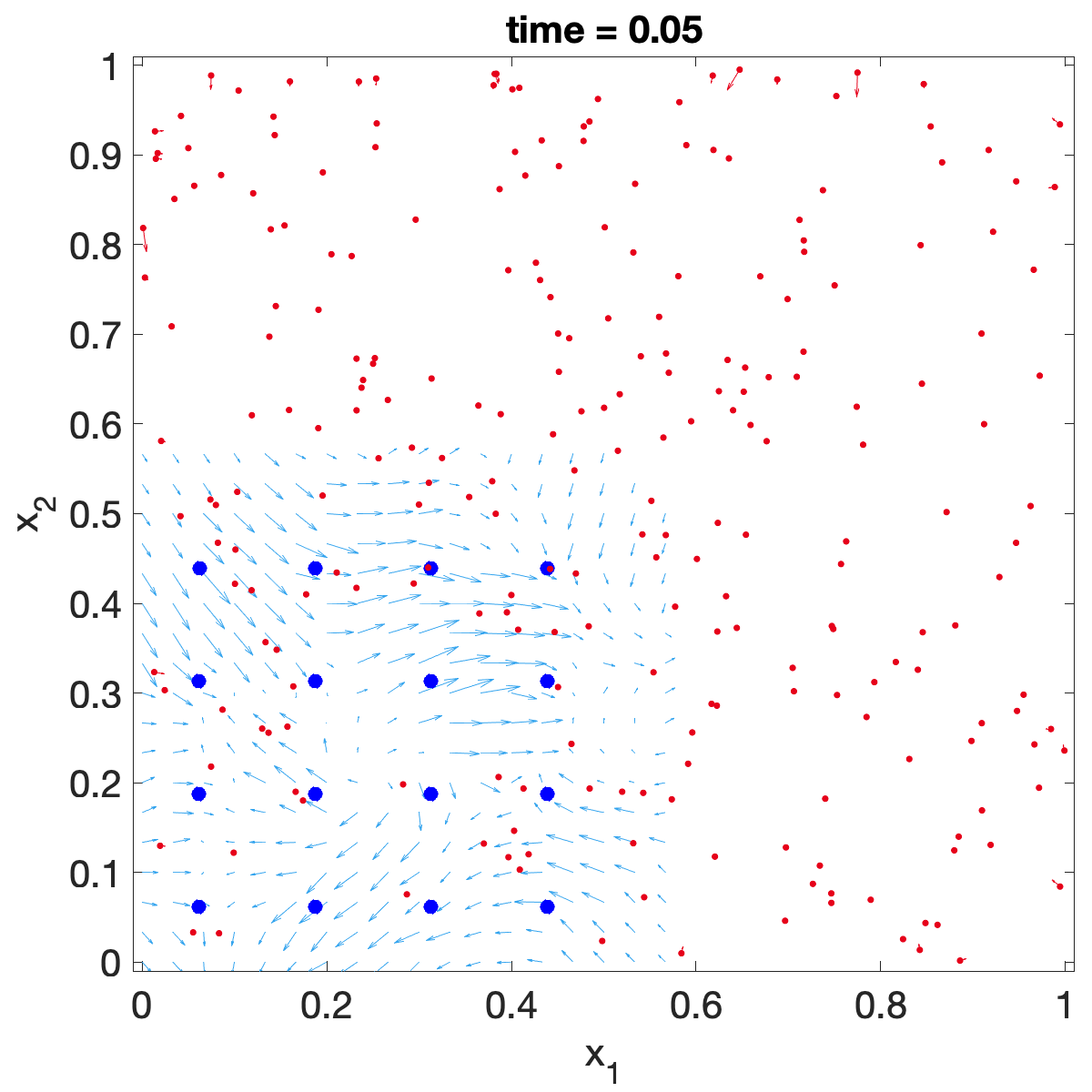}
    \end{subfigure}
    \begin{subfigure}[b]{0.19\textwidth}
        \centering
        \includegraphics[width=\textwidth]{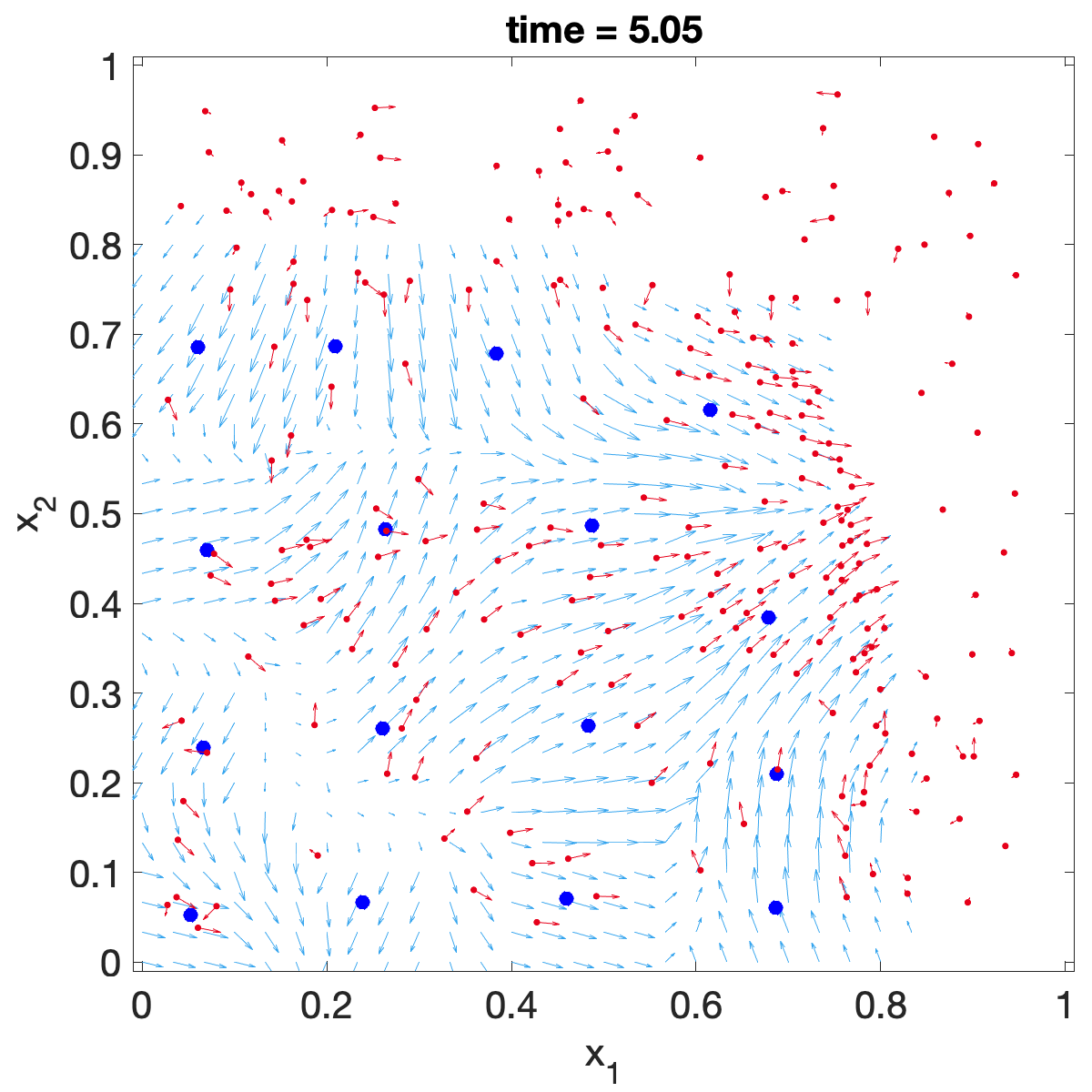}
    \end{subfigure}
    \begin{subfigure}[b]{0.19\textwidth}
        \centering
        \includegraphics[width=\textwidth]{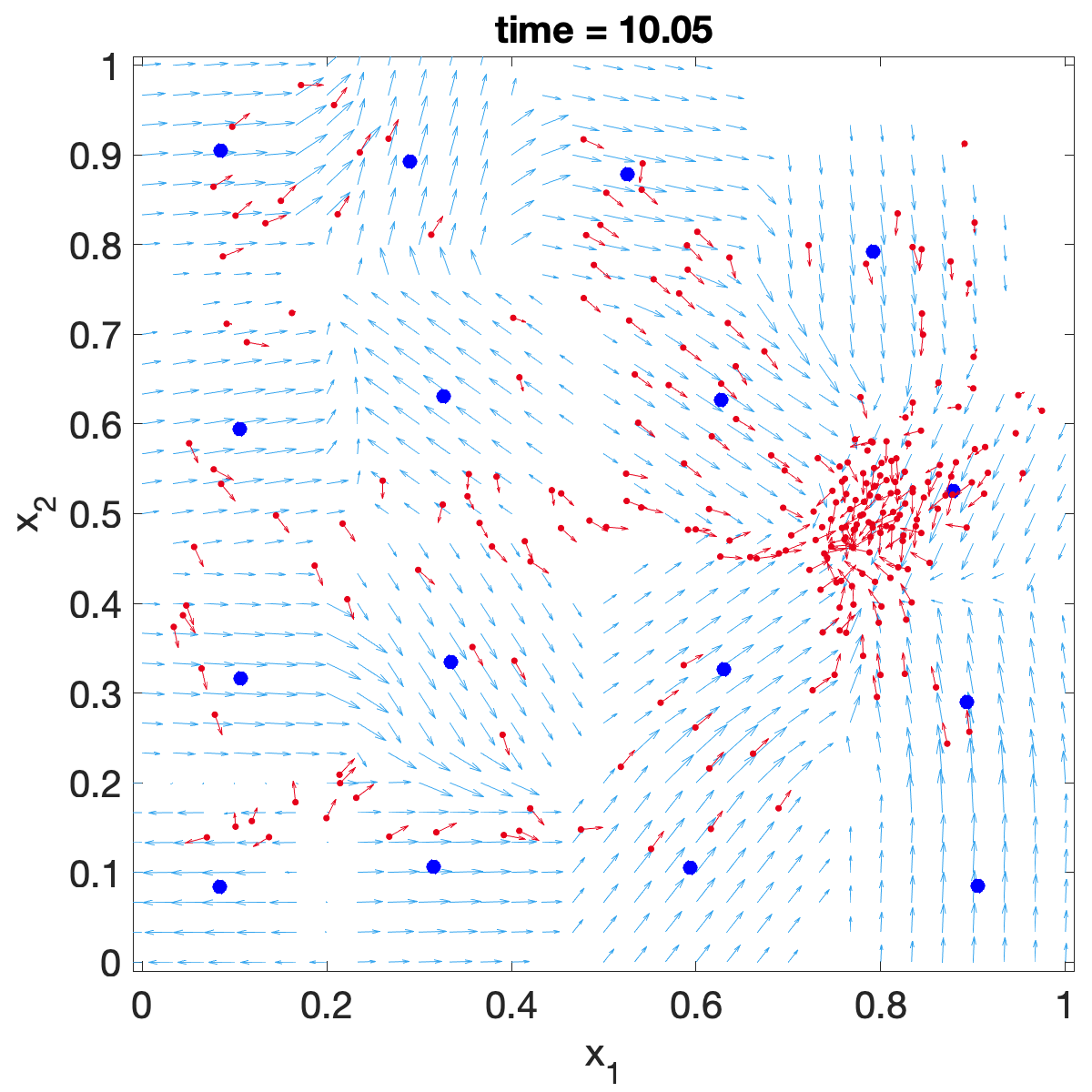}
    \end{subfigure}
    \begin{subfigure}[b]{0.19\textwidth}
        \centering
        \includegraphics[width=\textwidth]{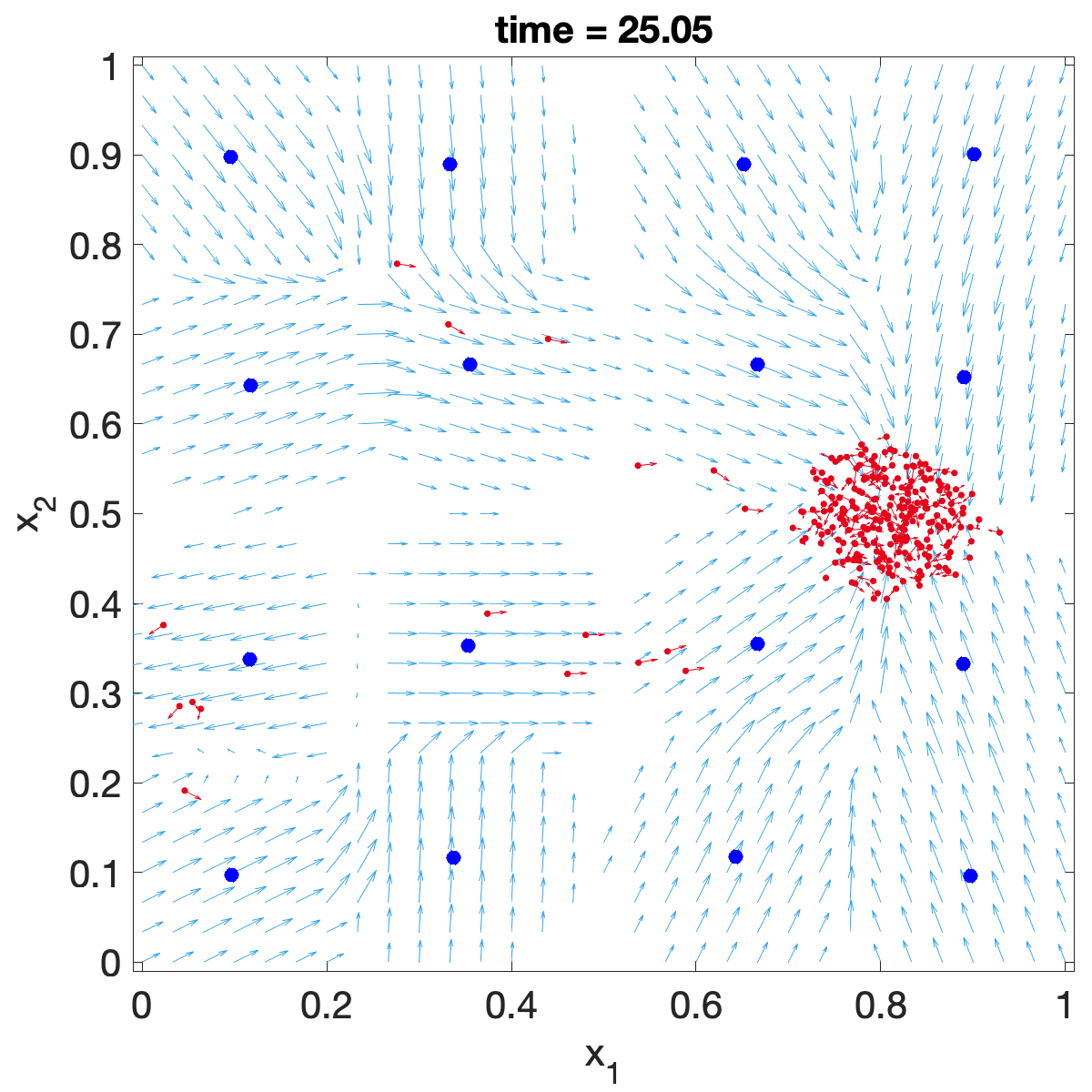}
    \end{subfigure}
    \begin{subfigure}[b]{0.19\textwidth}
        \centering
        \includegraphics[width=\textwidth]{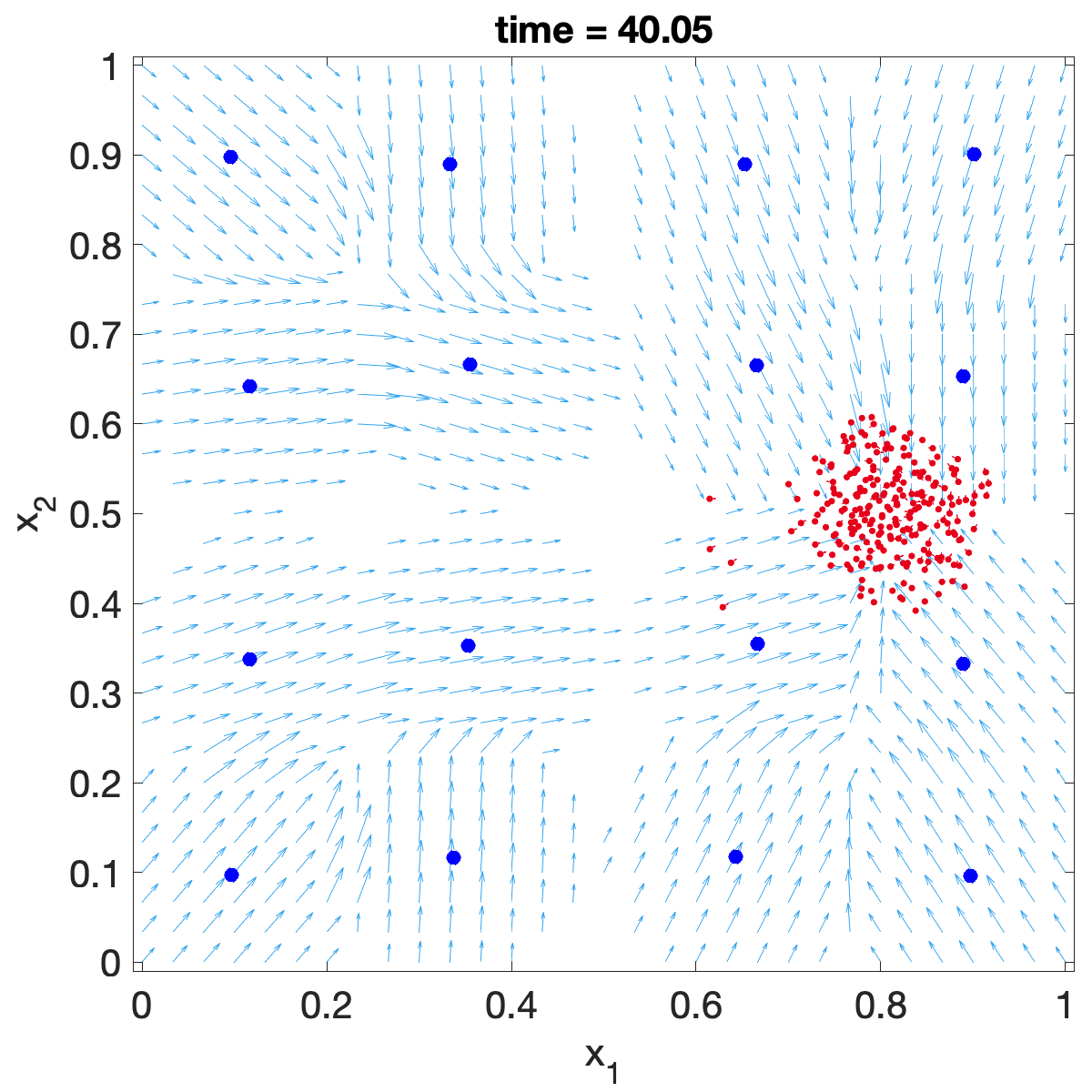}
    \end{subfigure}

    \begin{subfigure}[b]{0.19\textwidth}
        \centering
        \includegraphics[width=\textwidth]{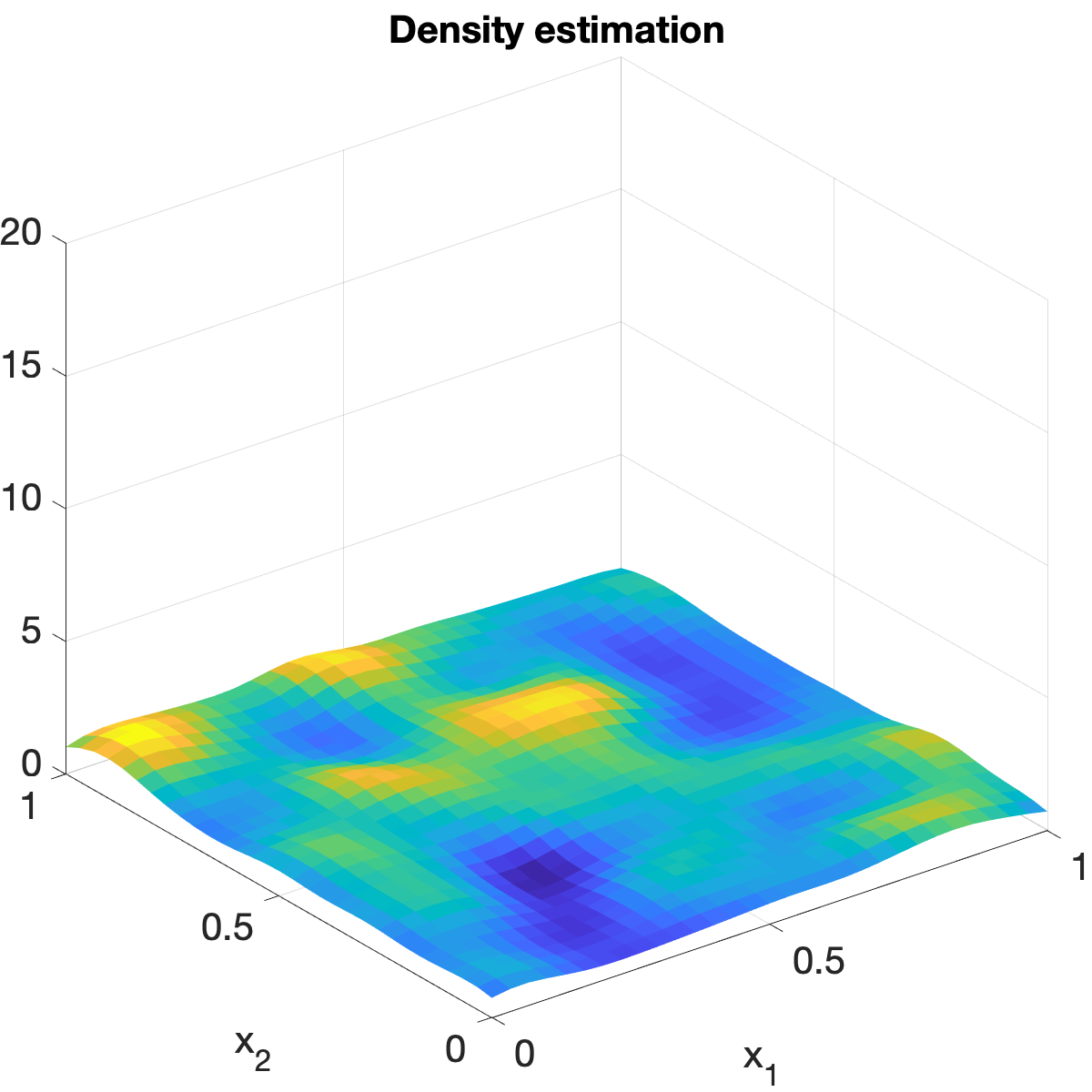}
    \end{subfigure}
    \begin{subfigure}[b]{0.19\textwidth}
        \centering
        \includegraphics[width=\textwidth]{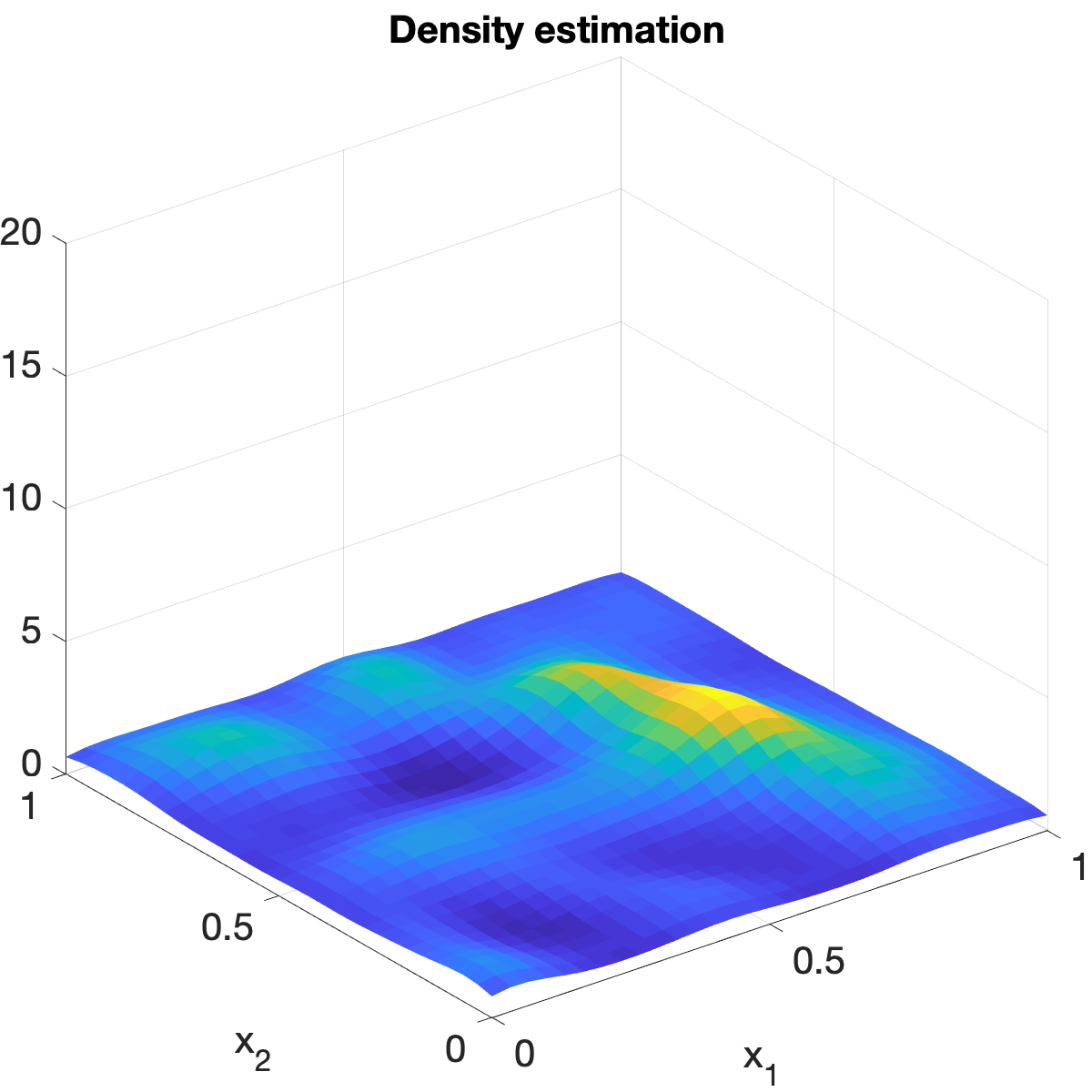}
    \end{subfigure}
    \begin{subfigure}[b]{0.19\textwidth}
        \centering
        \includegraphics[width=\textwidth]{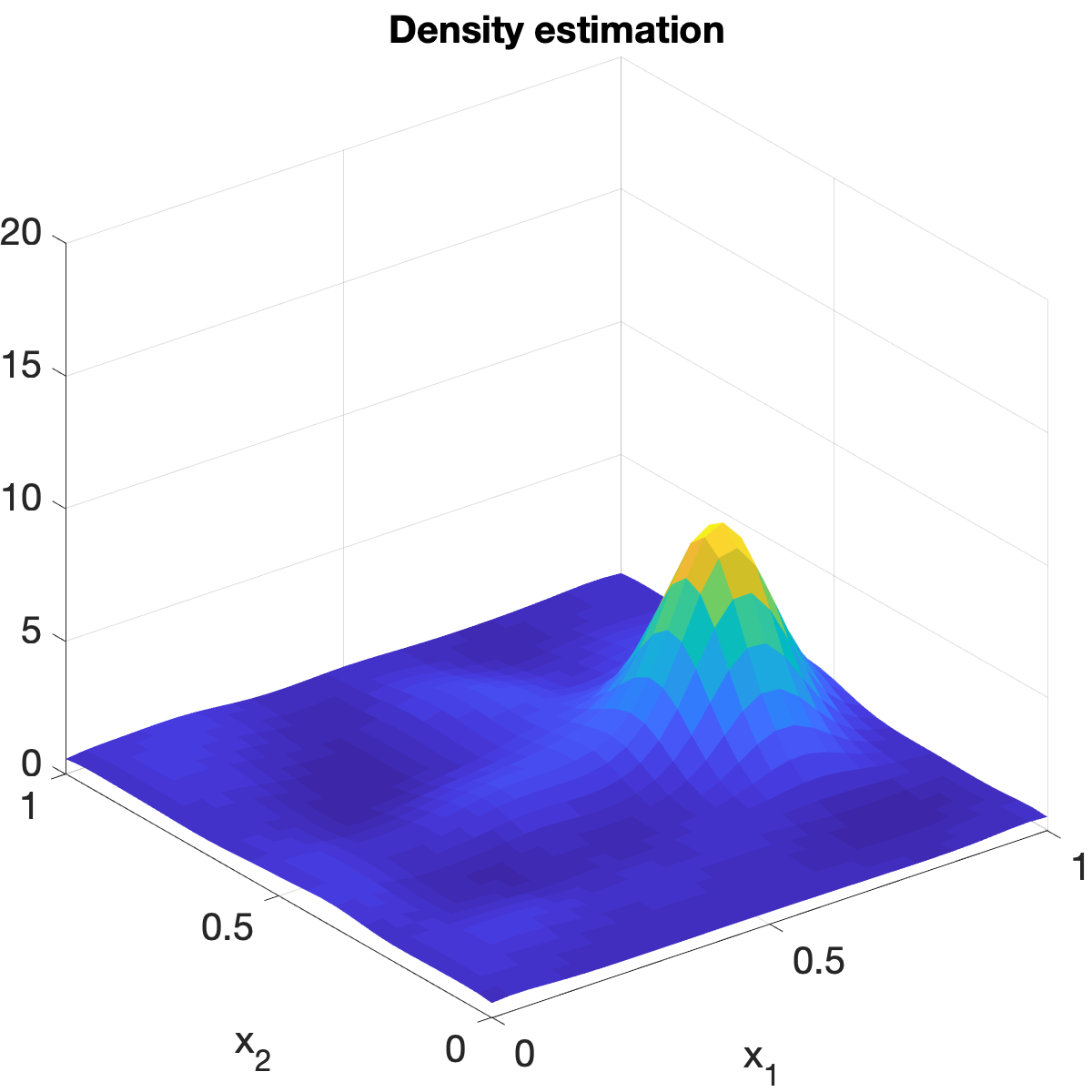}
    \end{subfigure}
    \begin{subfigure}[b]{0.19\textwidth}
        \centering
        \includegraphics[width=\textwidth]{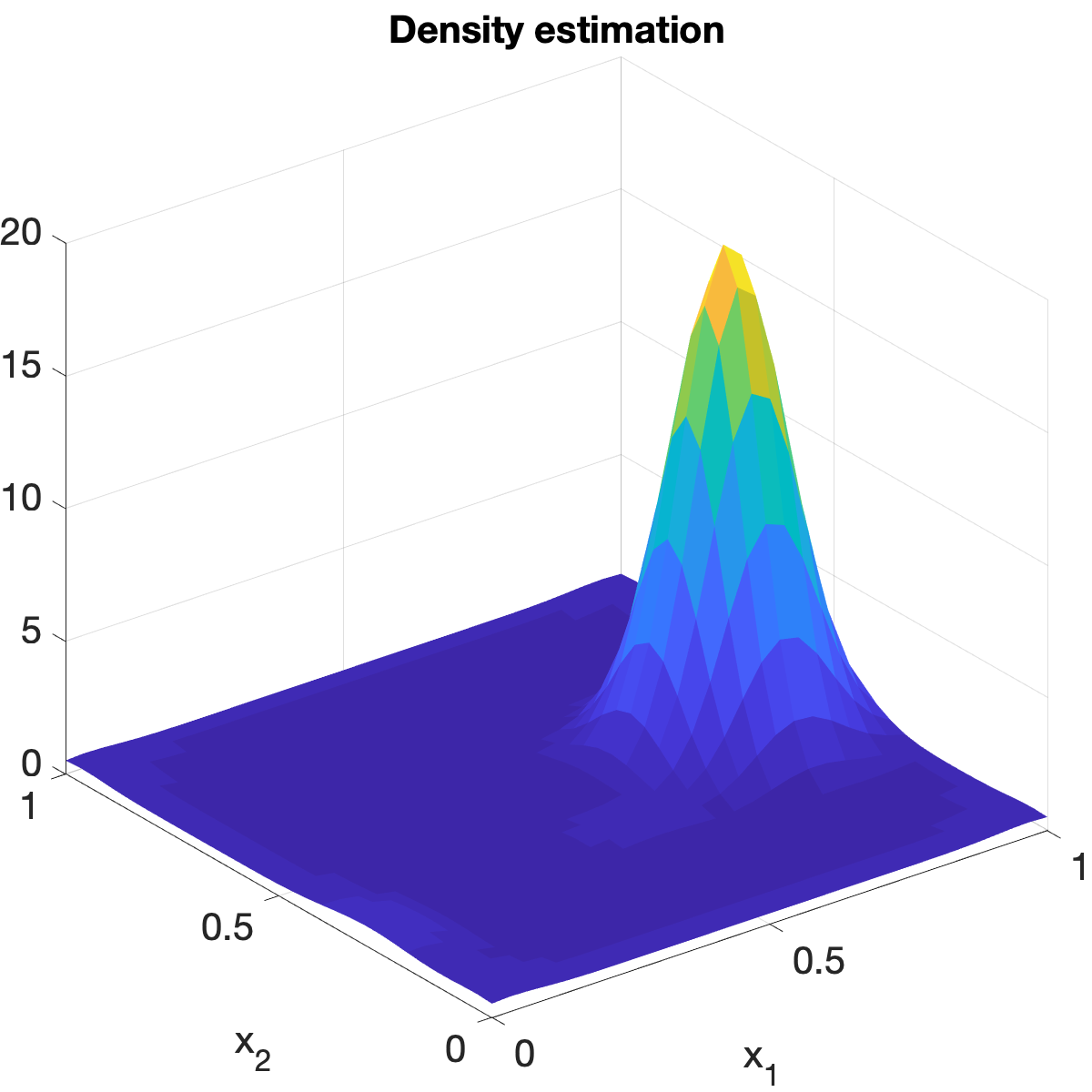}
    \end{subfigure}
    \begin{subfigure}[b]{0.19\textwidth}
        \centering
        \includegraphics[width=\textwidth]{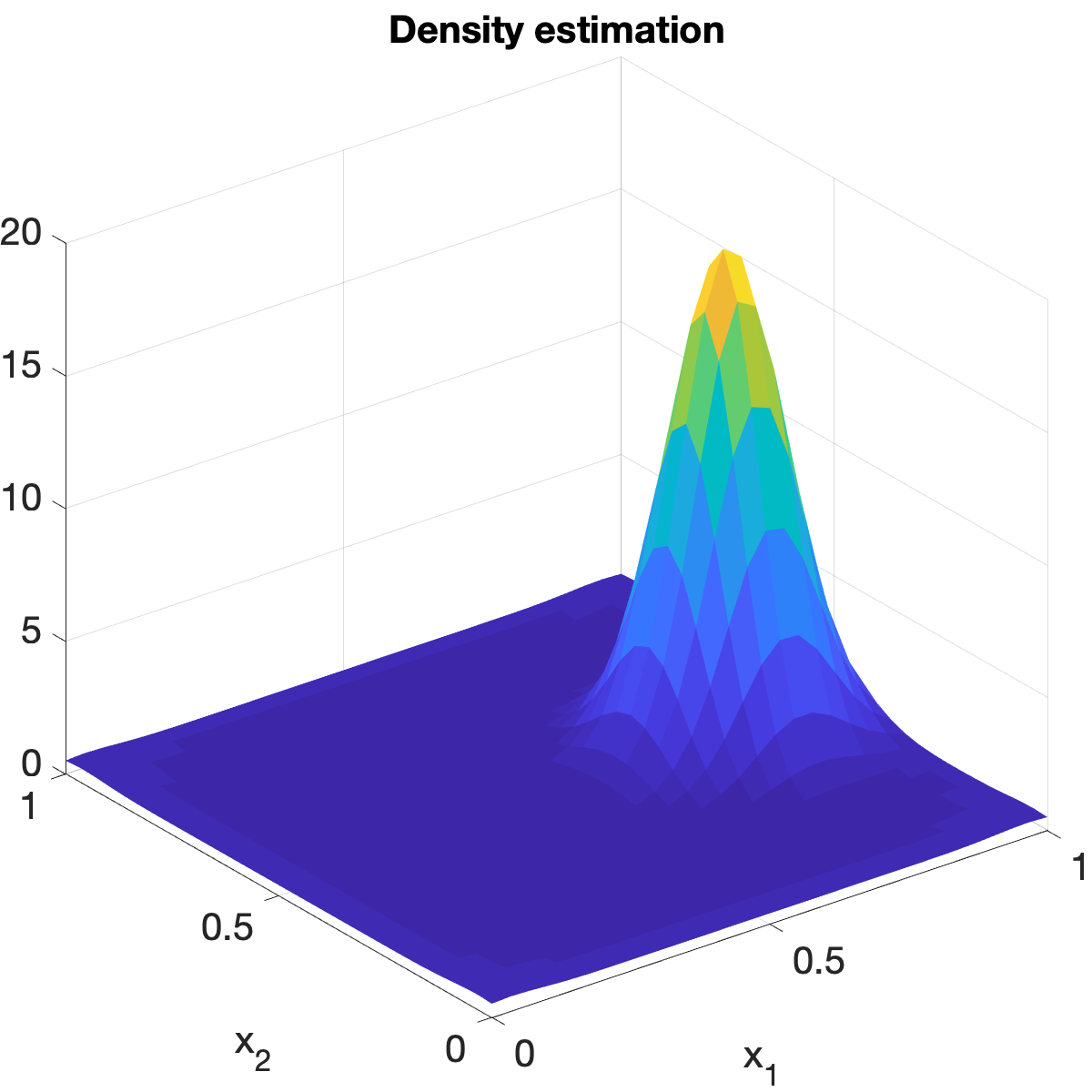}
    \end{subfigure}
    
    \begin{subfigure}[b]{0.19\textwidth}
        \centering
        \includegraphics[width=\textwidth]{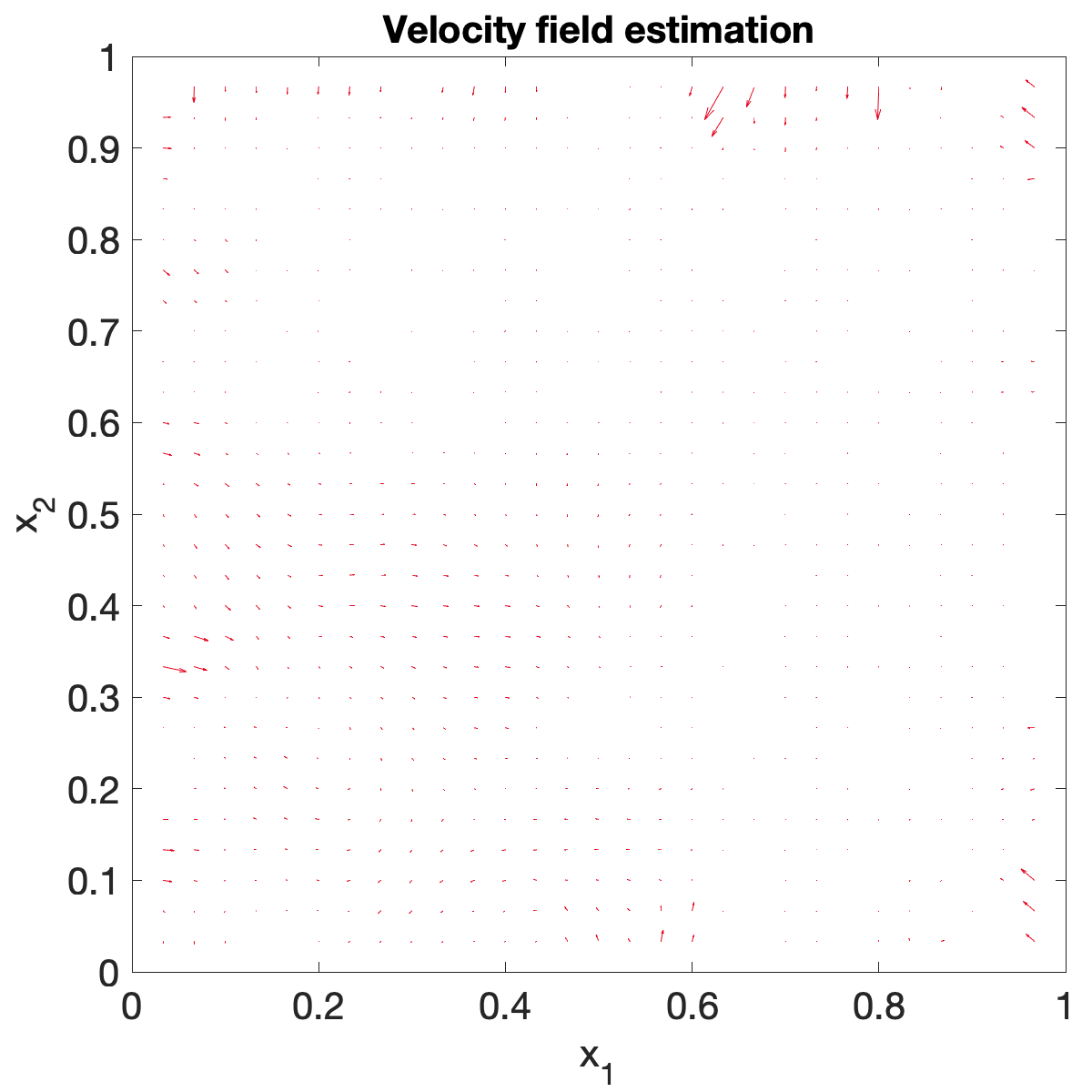}
    \end{subfigure}
    \begin{subfigure}[b]{0.19\textwidth}
        \centering
        \includegraphics[width=\textwidth]{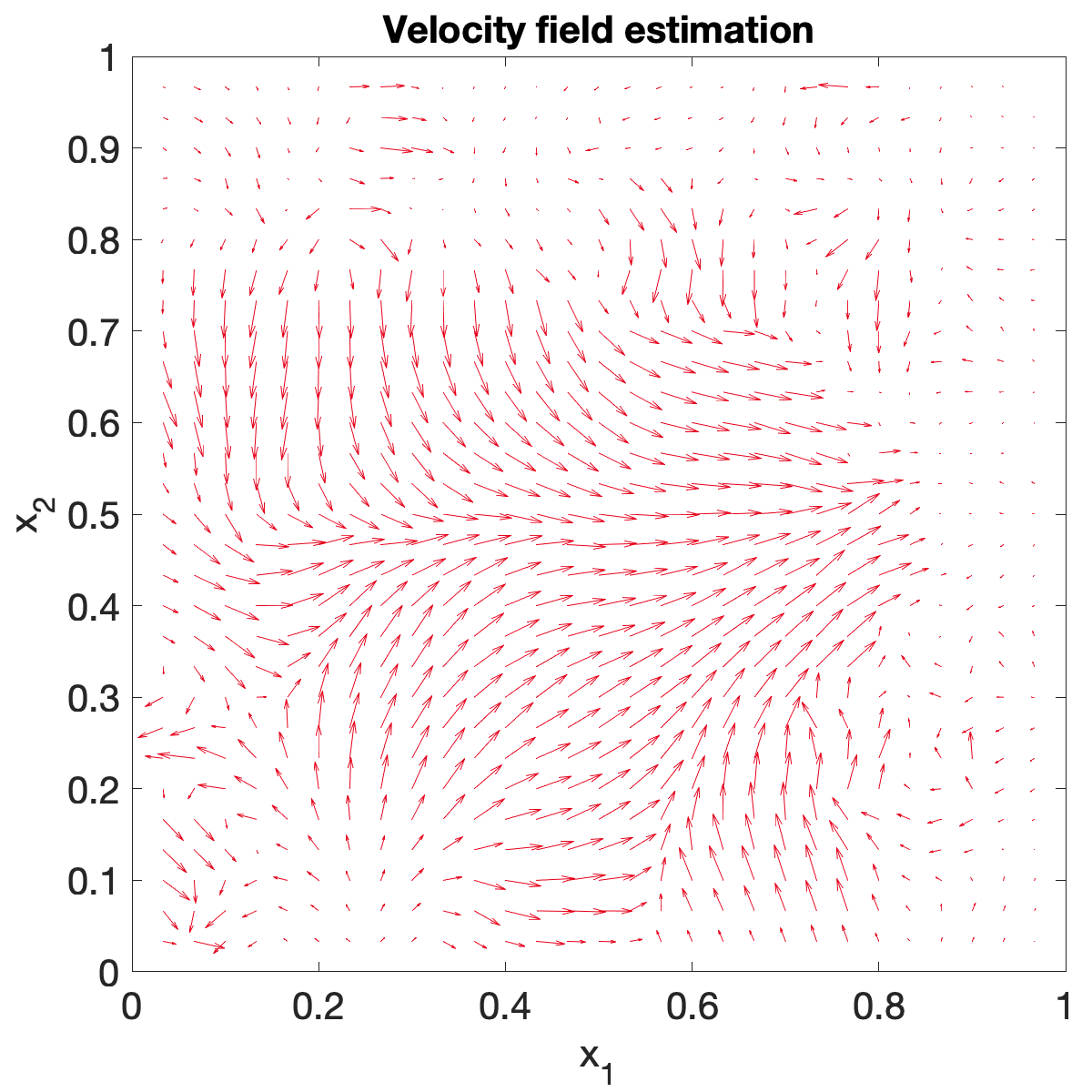}
    \end{subfigure}
    \begin{subfigure}[b]{0.19\textwidth}
        \centering
        \includegraphics[width=\textwidth]{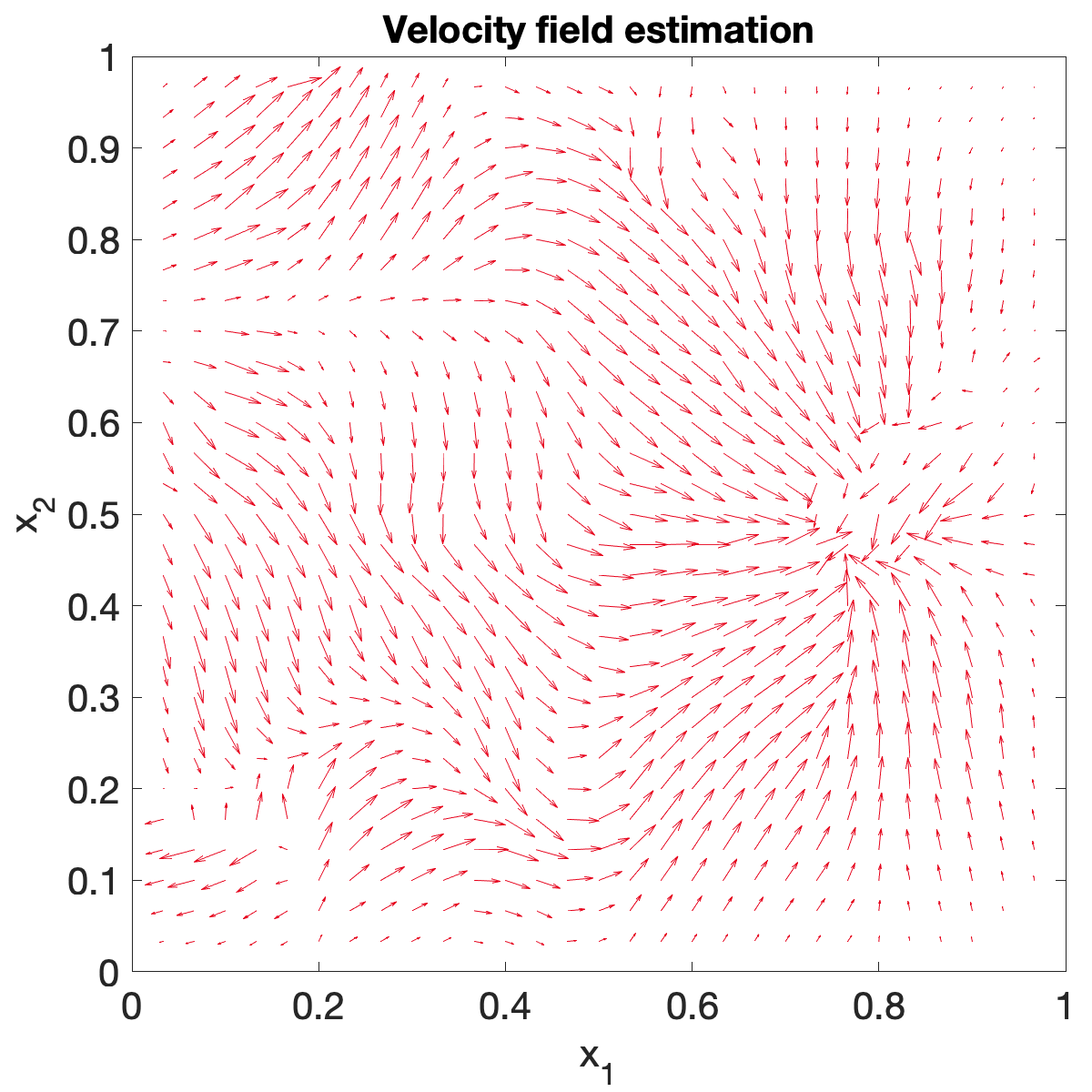}
    \end{subfigure}
    \begin{subfigure}[b]{0.19\textwidth}
        \centering
        \includegraphics[width=\textwidth]{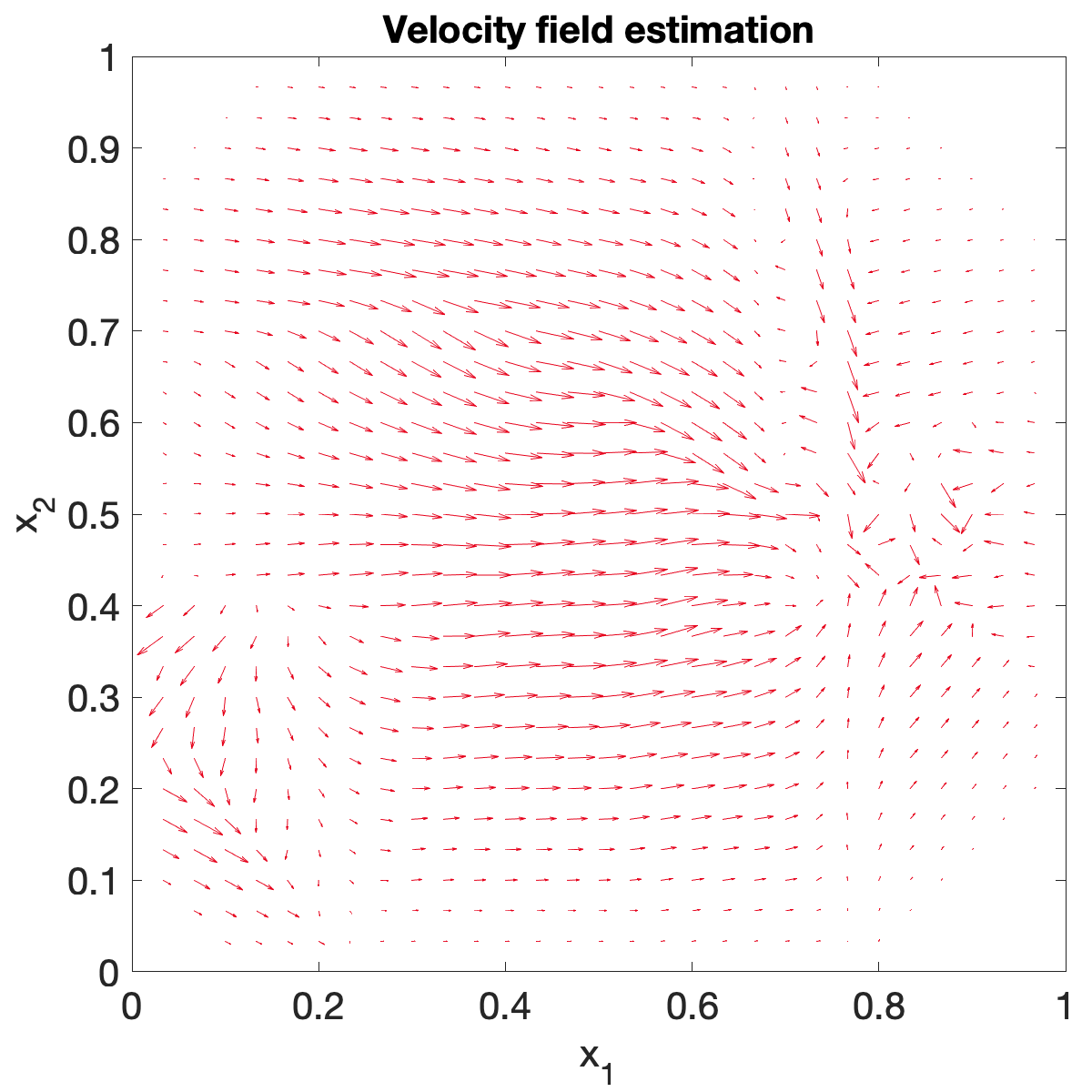}
    \end{subfigure}
    \begin{subfigure}[b]{0.19\textwidth}
        \centering
        \includegraphics[width=\textwidth]{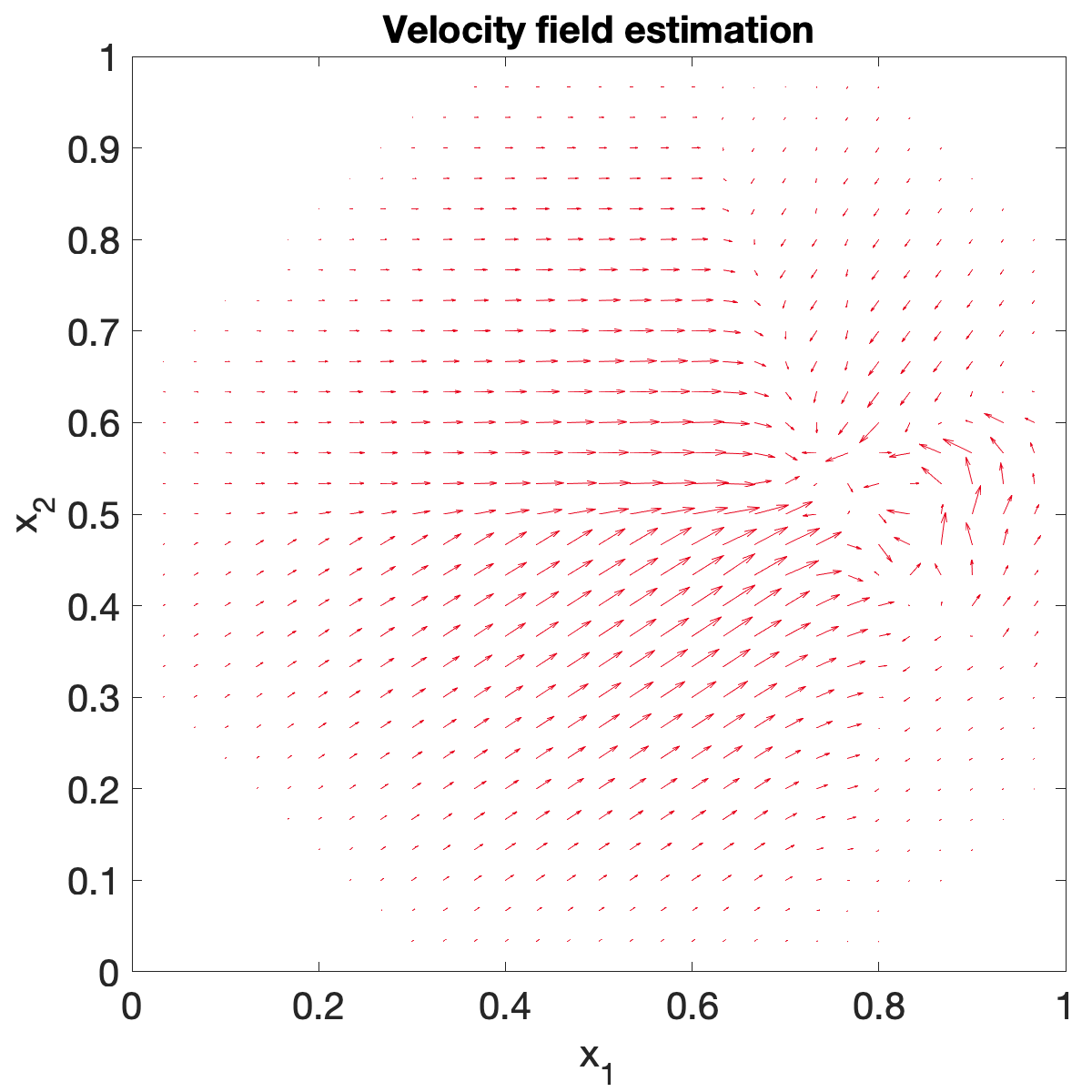}
    \end{subfigure}
    \caption{Illustration of the evacuation process when the environment was free of obstacles.
    Each column of subfigures represented the same instant.
    Each row of subfigures represented the time evaluation of certain variables. \\
    Row 1: Human positions $\{x_j(t)\}_{j=1}^N$ (red dots), human velocities $\{v_j(t)\}_{j=1}^N$ (red arrows attached to the red dots), robot positions $\{r_i(t)\}_{i=1}^n$ (black dots), and the navigation force fields $\{F_i(x,t)\}_{i=1}^n$ with round supports generated by the robots (black arrows surrounding the black dots).\\
    Row 2: Estimation of the real-time crowd density $\rho(x,t)$ using kernel density estimation.\\
    Row 3: Estimation of the real-time crowd velocity field $u(x,t)$ using linear interpolation.
    }
    \label{fig:evacuation}
\end{figure*}

\begin{figure*}[h]
    \centering
    \begin{subfigure}[b]{0.19\textwidth}
        \centering
        \includegraphics[width=\textwidth]{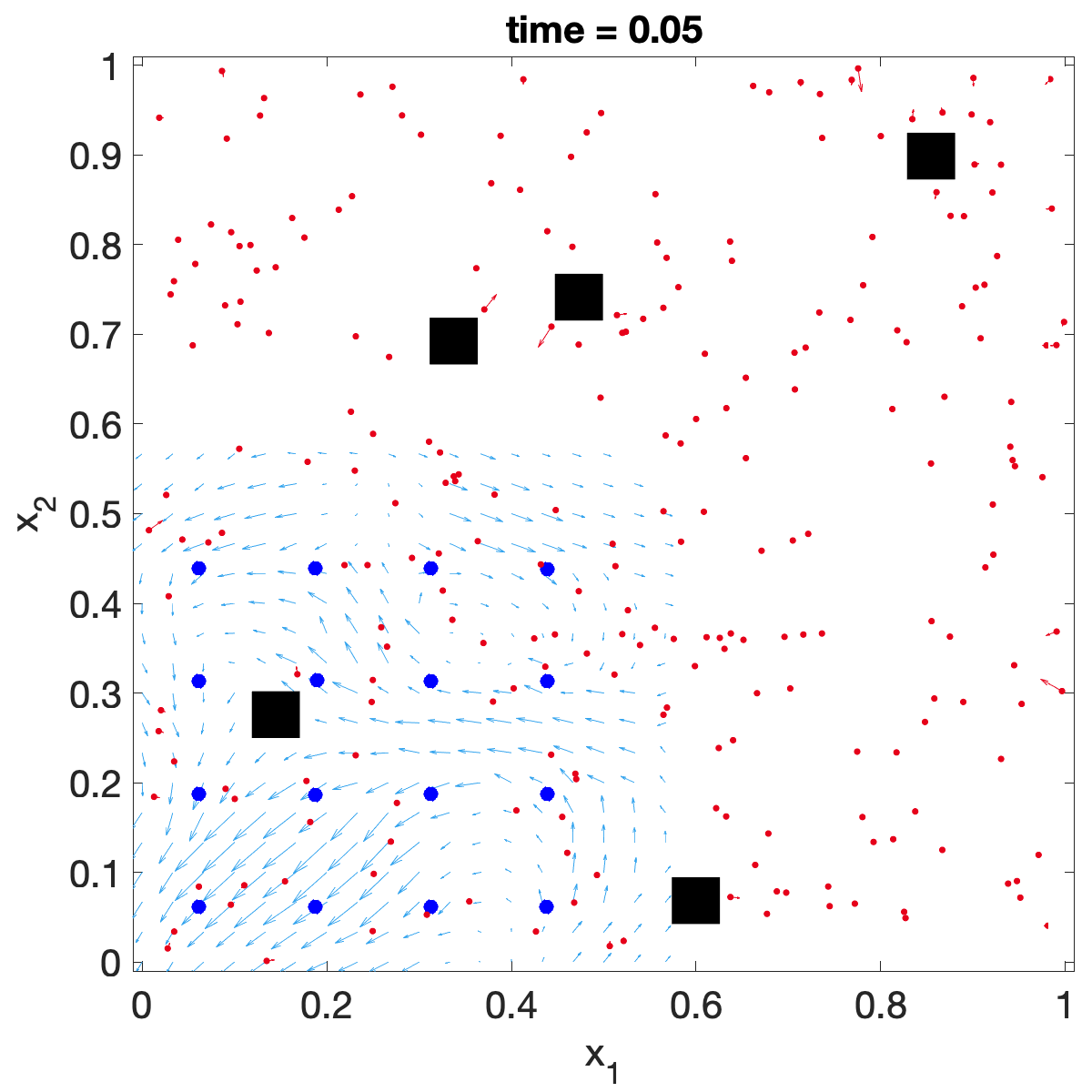}
    \end{subfigure}
    \begin{subfigure}[b]{0.19\textwidth}
        \centering
        \includegraphics[width=\textwidth]{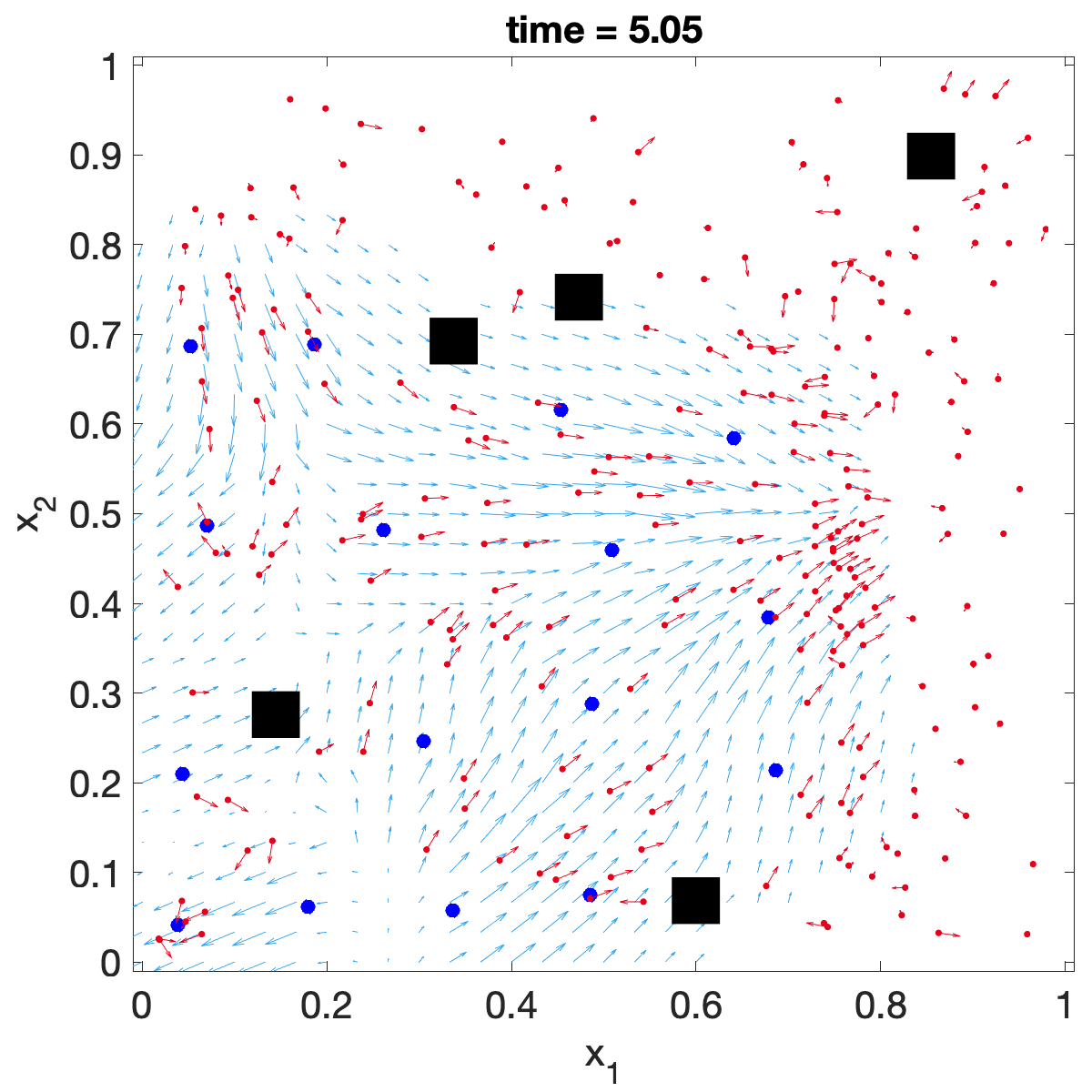}
    \end{subfigure}
    \begin{subfigure}[b]{0.19\textwidth}
        \centering
        \includegraphics[width=\textwidth]{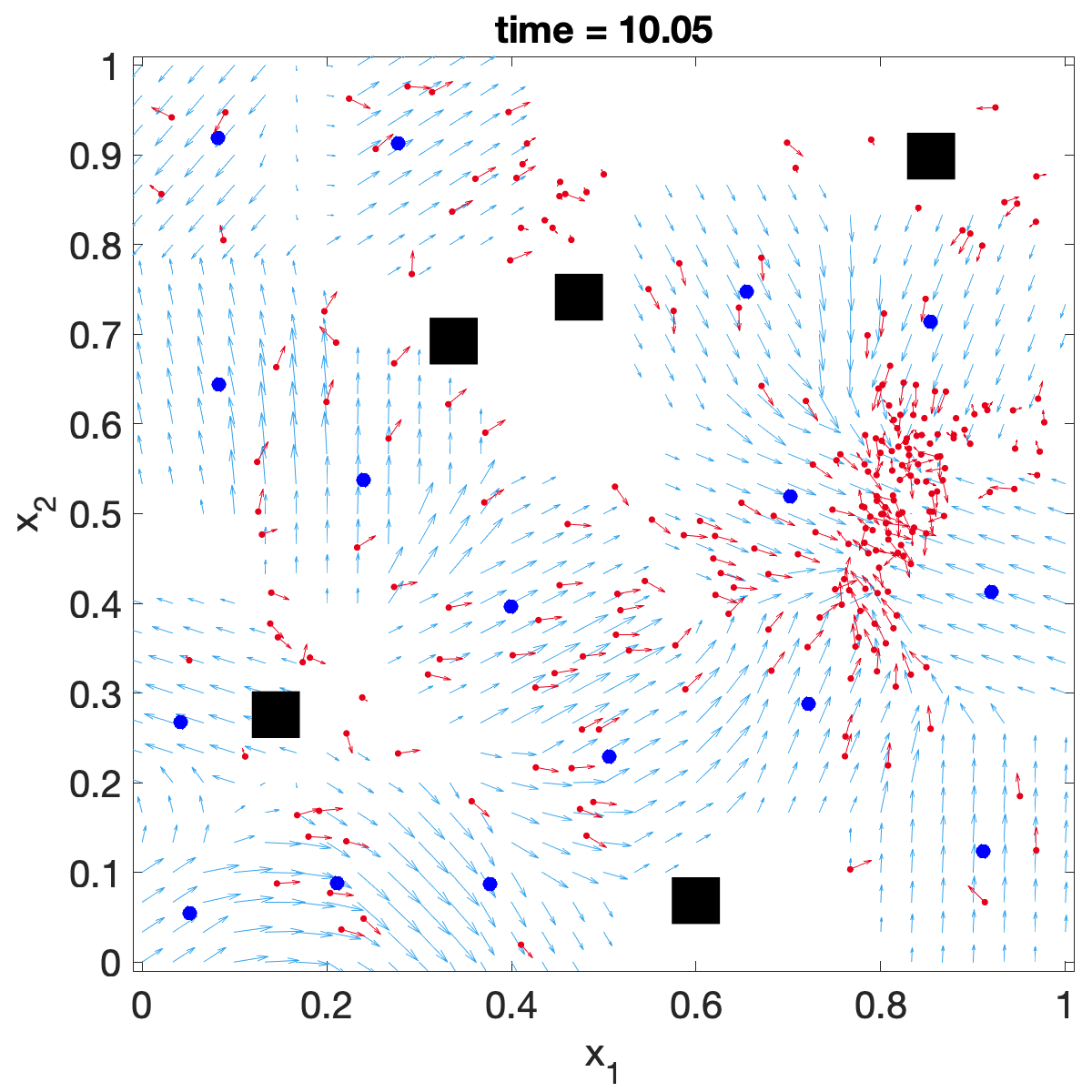}
    \end{subfigure}
    \begin{subfigure}[b]{0.19\textwidth}
        \centering
        \includegraphics[width=\textwidth]{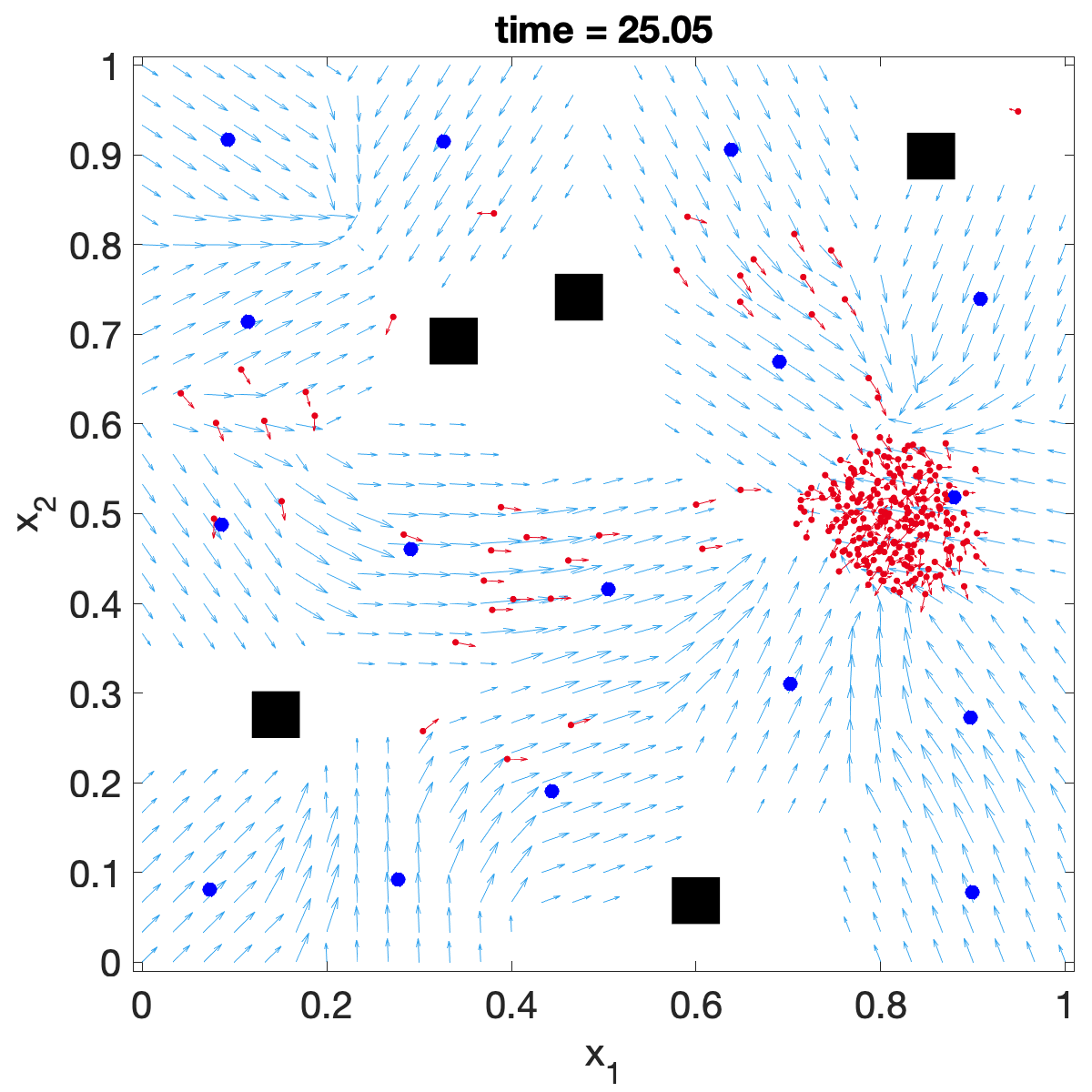}
    \end{subfigure}
    \begin{subfigure}[b]{0.19\textwidth}
        \centering
        \includegraphics[width=\textwidth]{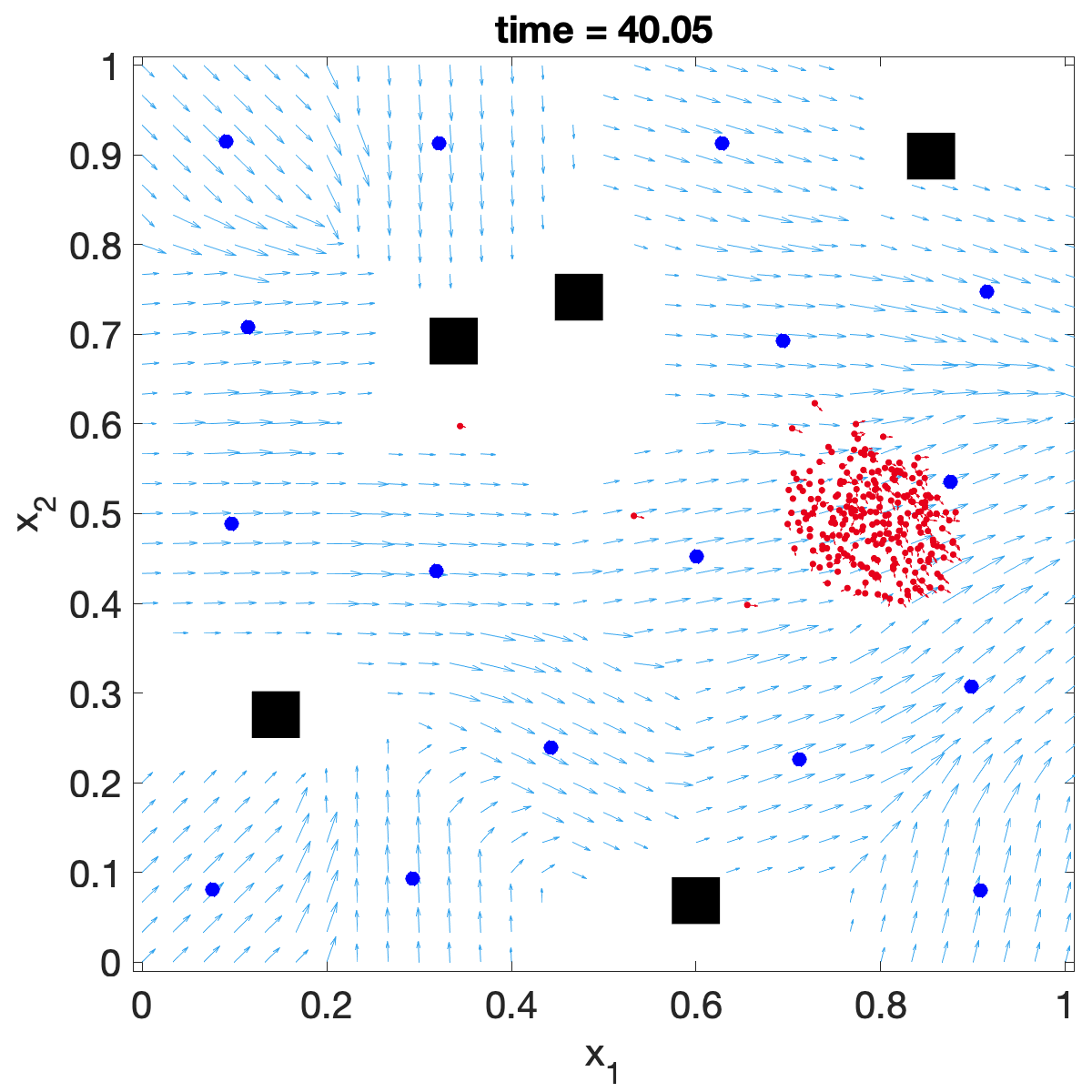}
    \end{subfigure}

    \begin{subfigure}[b]{0.19\textwidth}
        \centering
        \includegraphics[width=\textwidth]{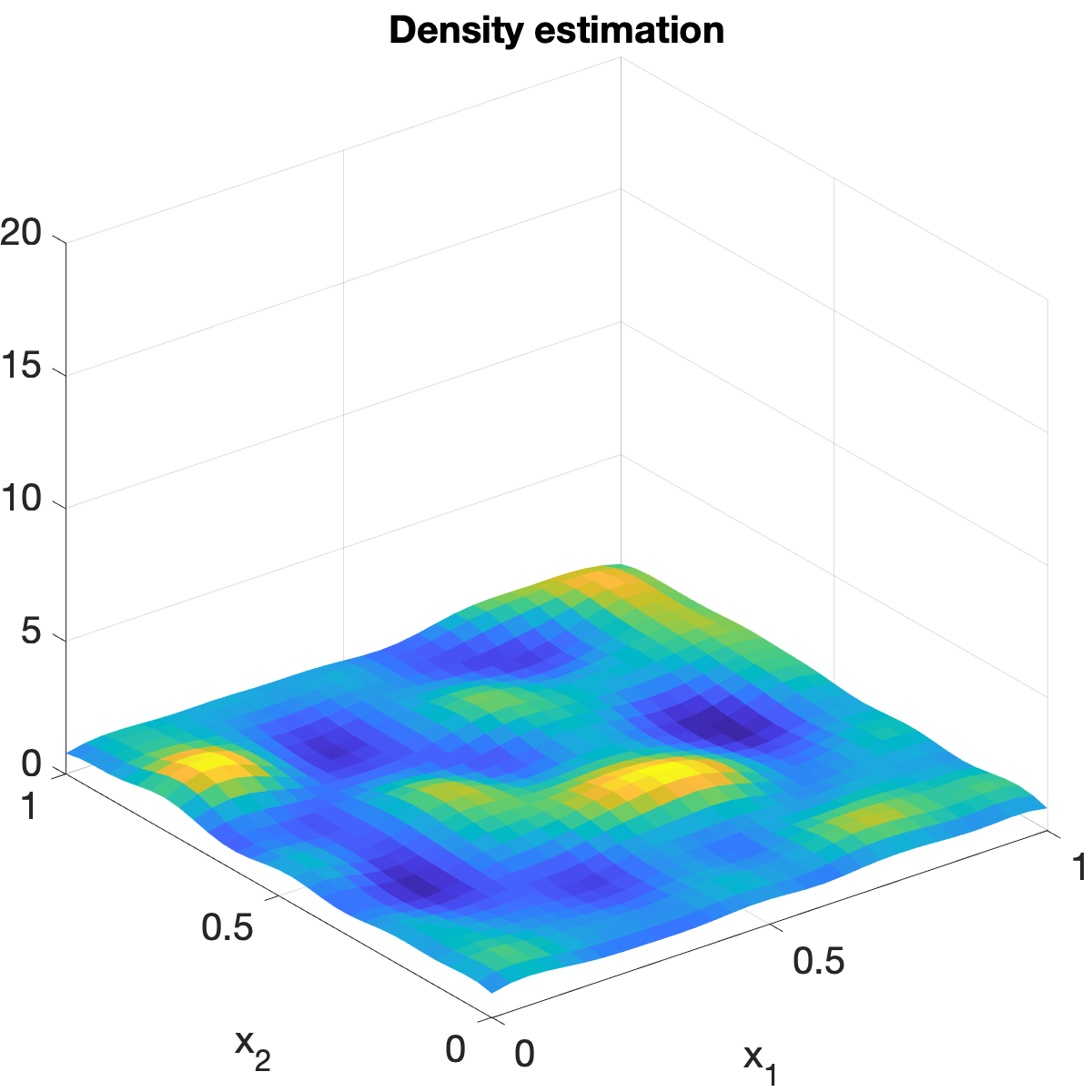}
    \end{subfigure}
    \begin{subfigure}[b]{0.19\textwidth}
        \centering
        \includegraphics[width=\textwidth]{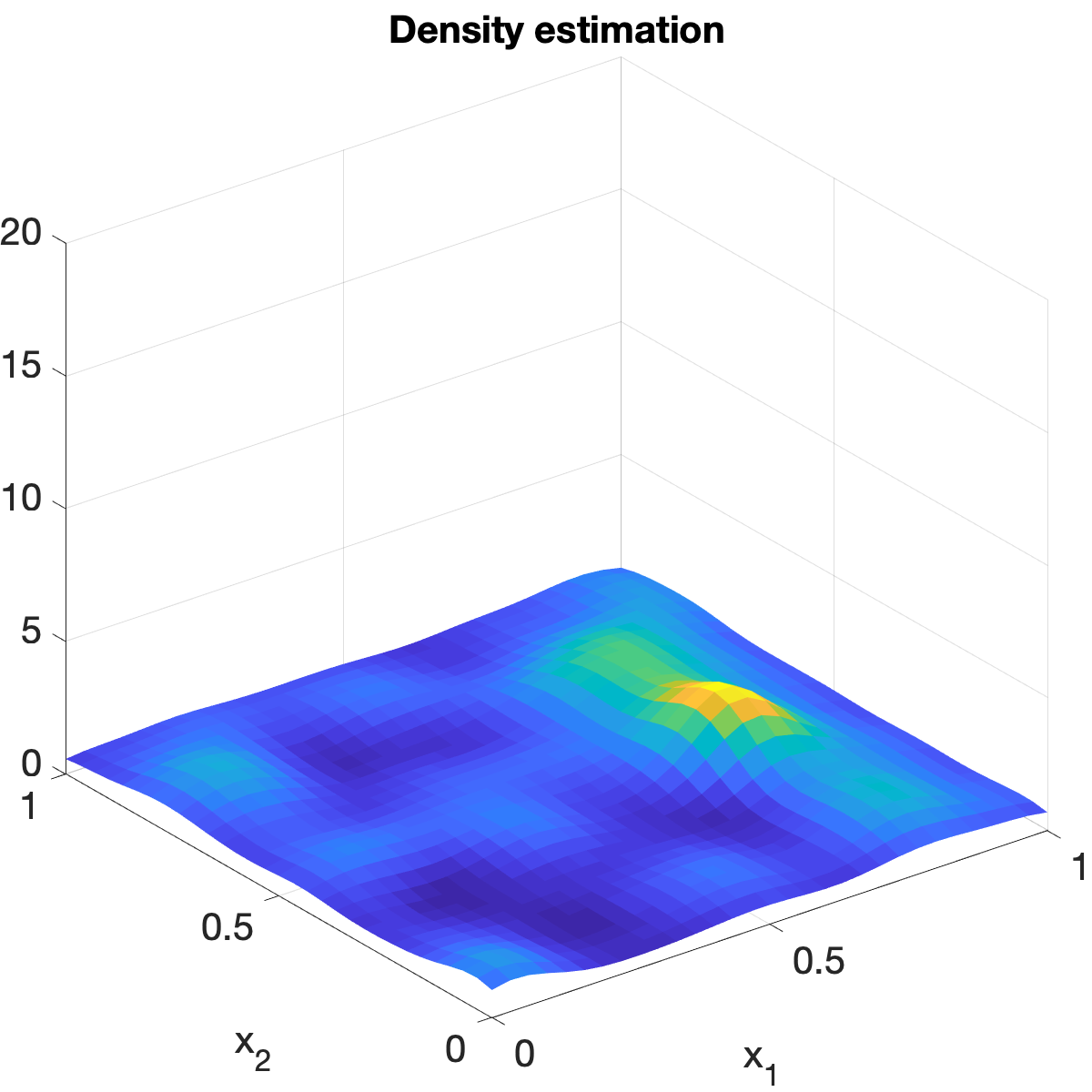}
    \end{subfigure}
    \begin{subfigure}[b]{0.19\textwidth}
        \centering
        \includegraphics[width=\textwidth]{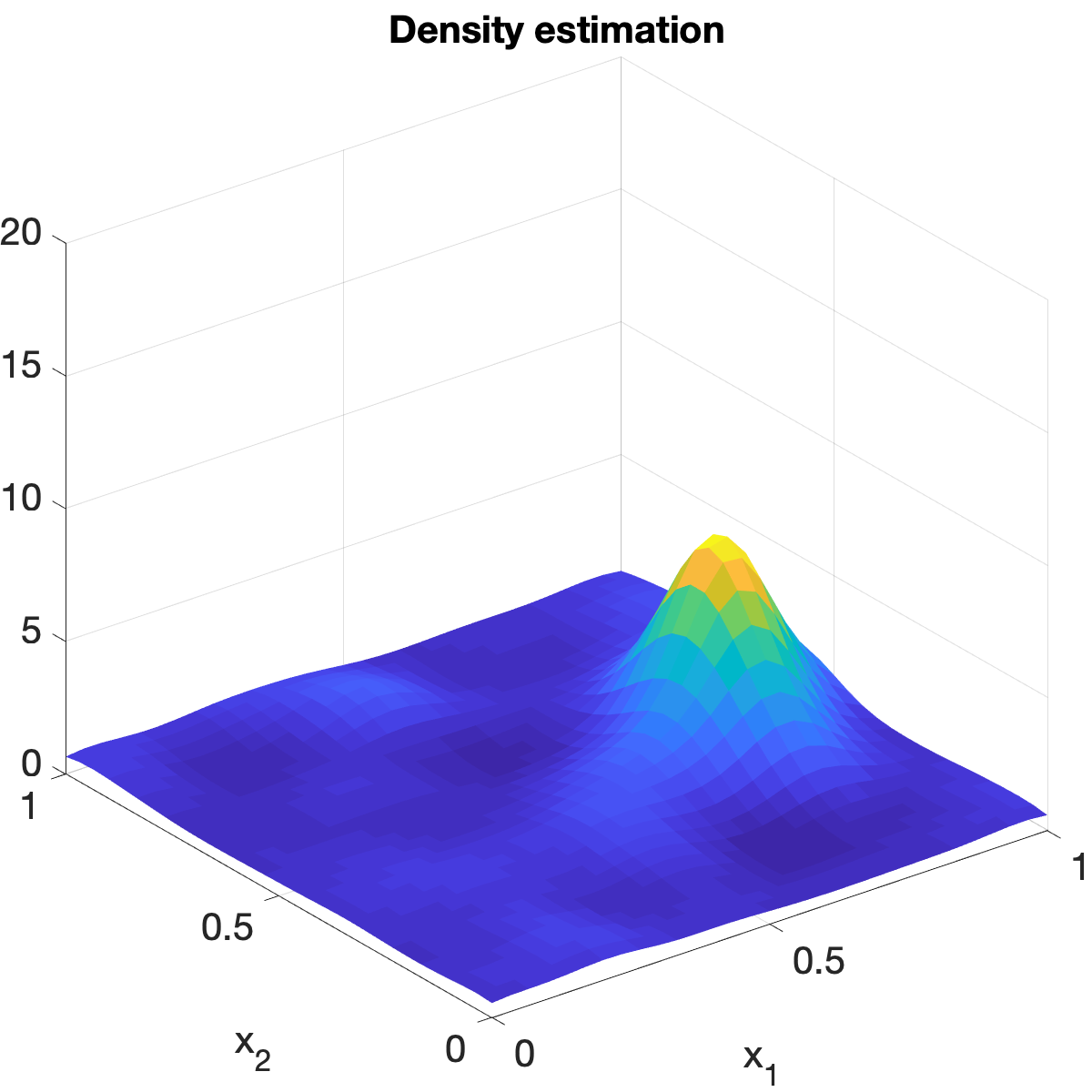}
    \end{subfigure}
    \begin{subfigure}[b]{0.19\textwidth}
        \centering
        \includegraphics[width=\textwidth]{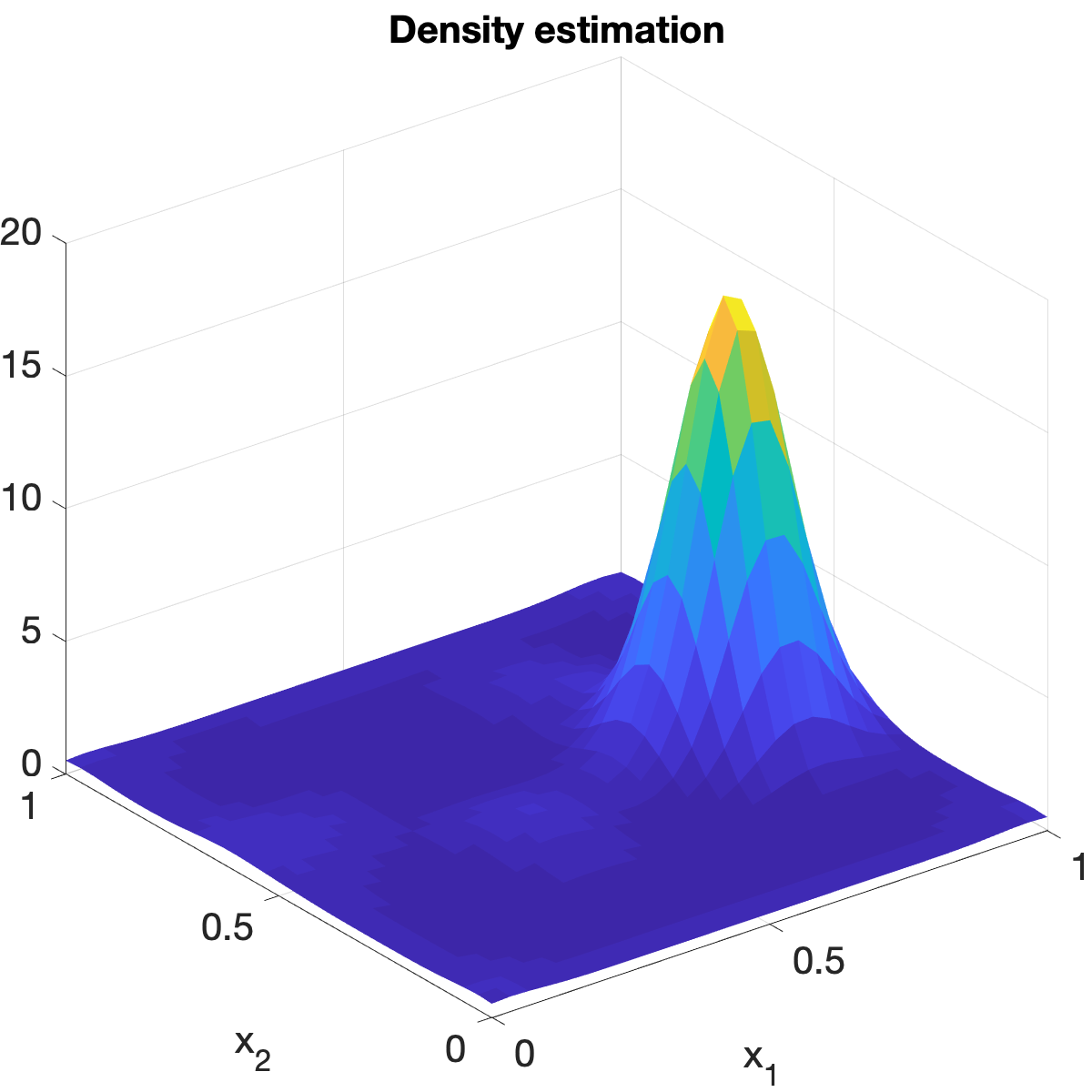}
    \end{subfigure}
    \begin{subfigure}[b]{0.19\textwidth}
        \centering
        \includegraphics[width=\textwidth]{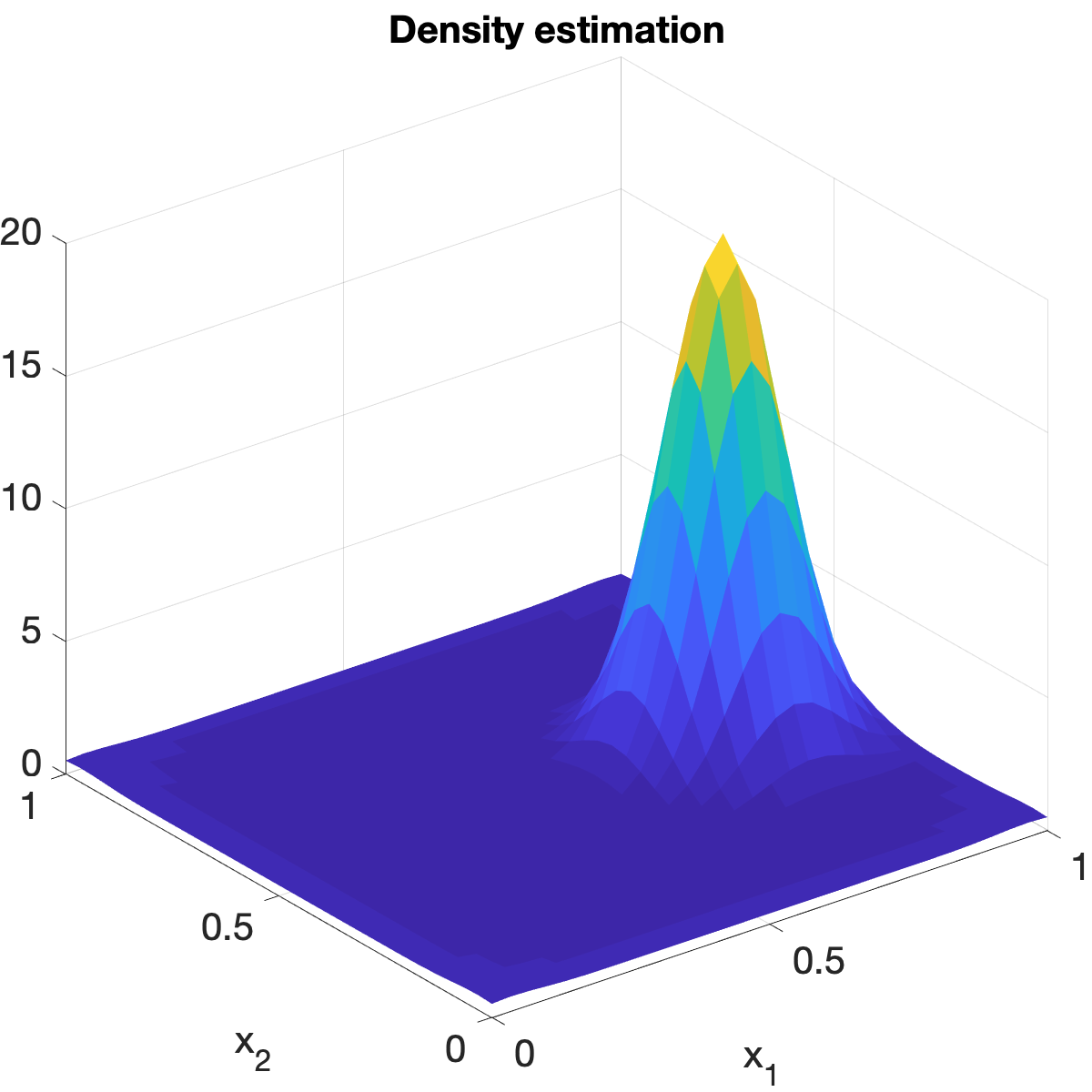}
    \end{subfigure}
    
    \begin{subfigure}[b]{0.19\textwidth}
        \centering
        \includegraphics[width=\textwidth]{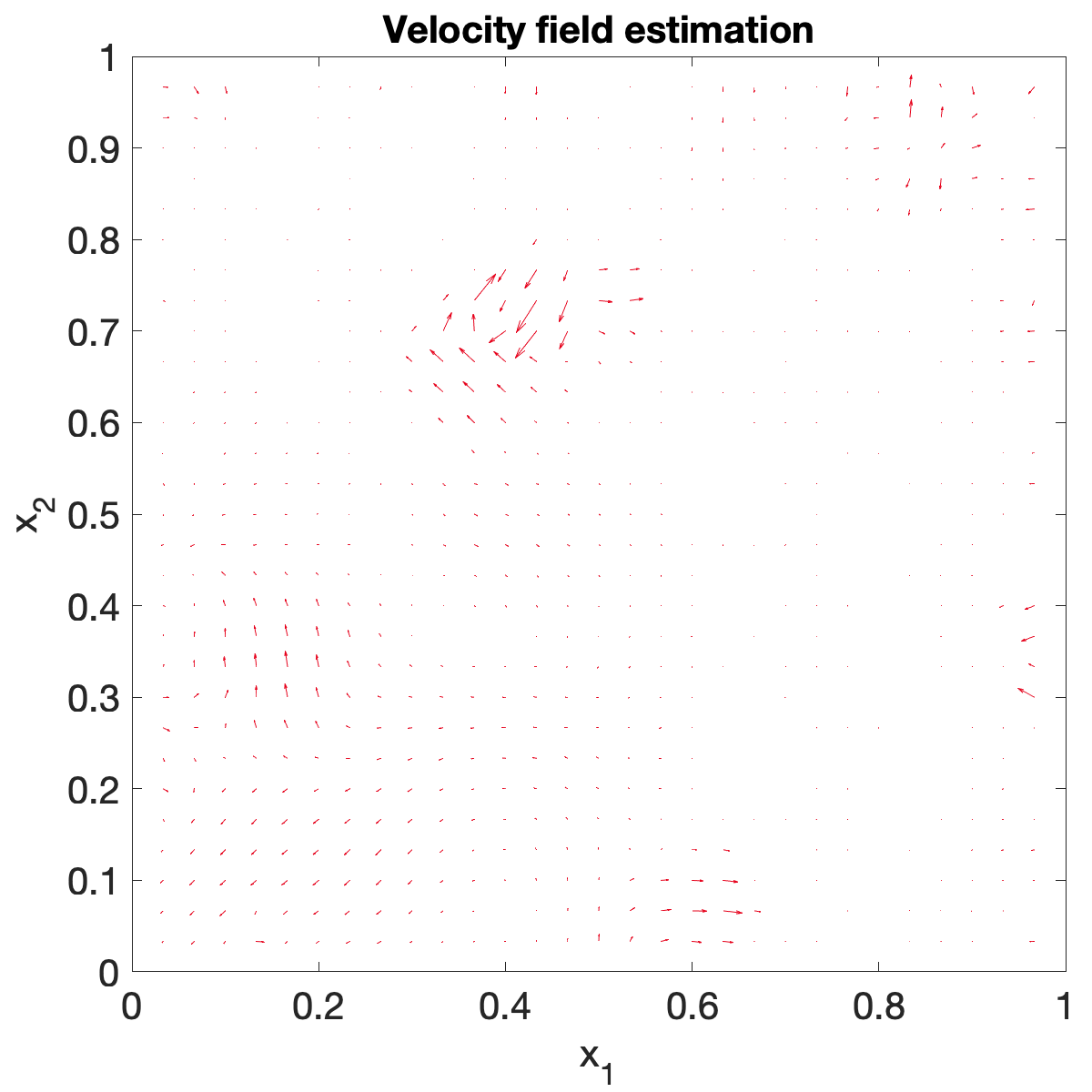}
    \end{subfigure}
    \begin{subfigure}[b]{0.19\textwidth}
        \centering
        \includegraphics[width=\textwidth]{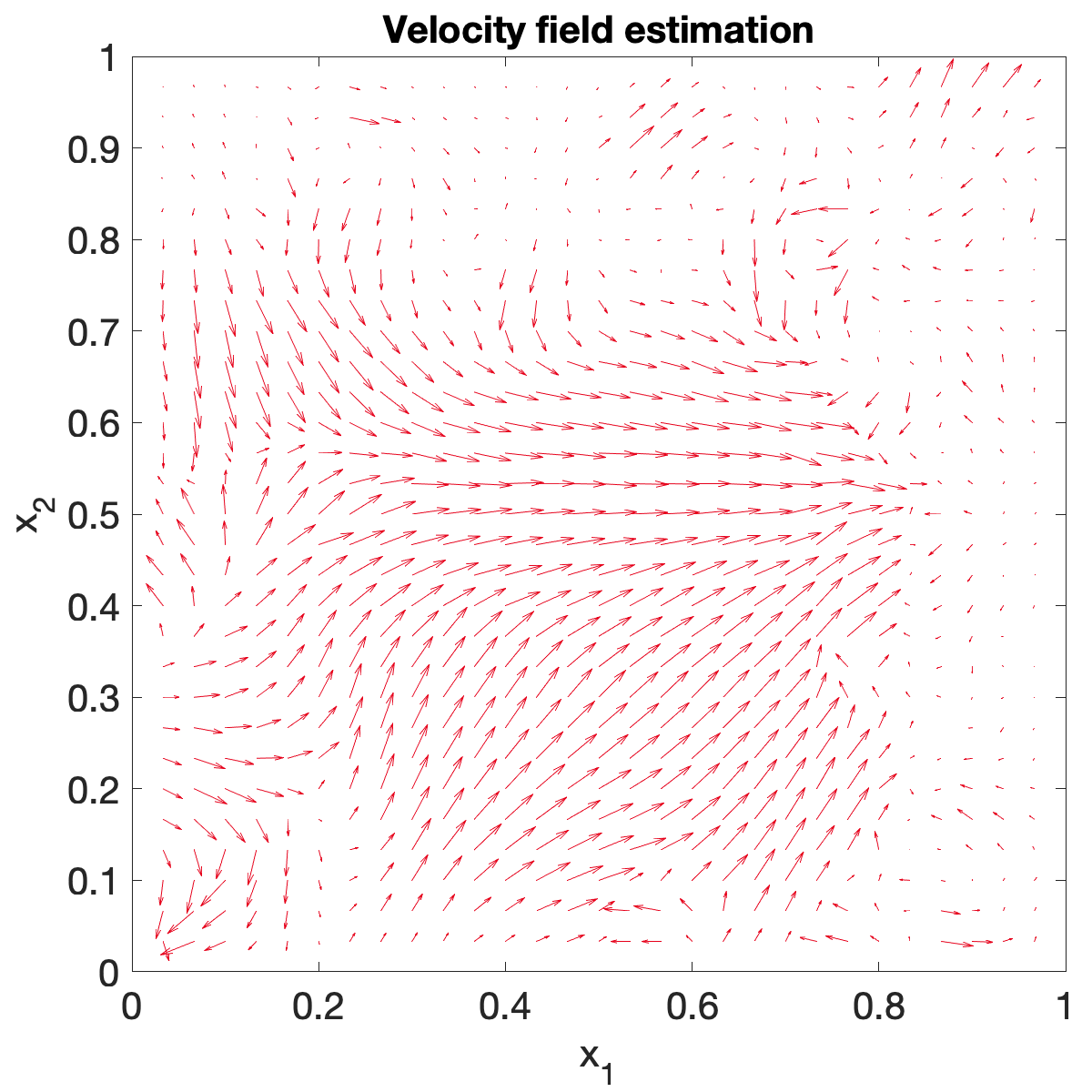}
    \end{subfigure}
    \begin{subfigure}[b]{0.19\textwidth}
        \centering
        \includegraphics[width=\textwidth]{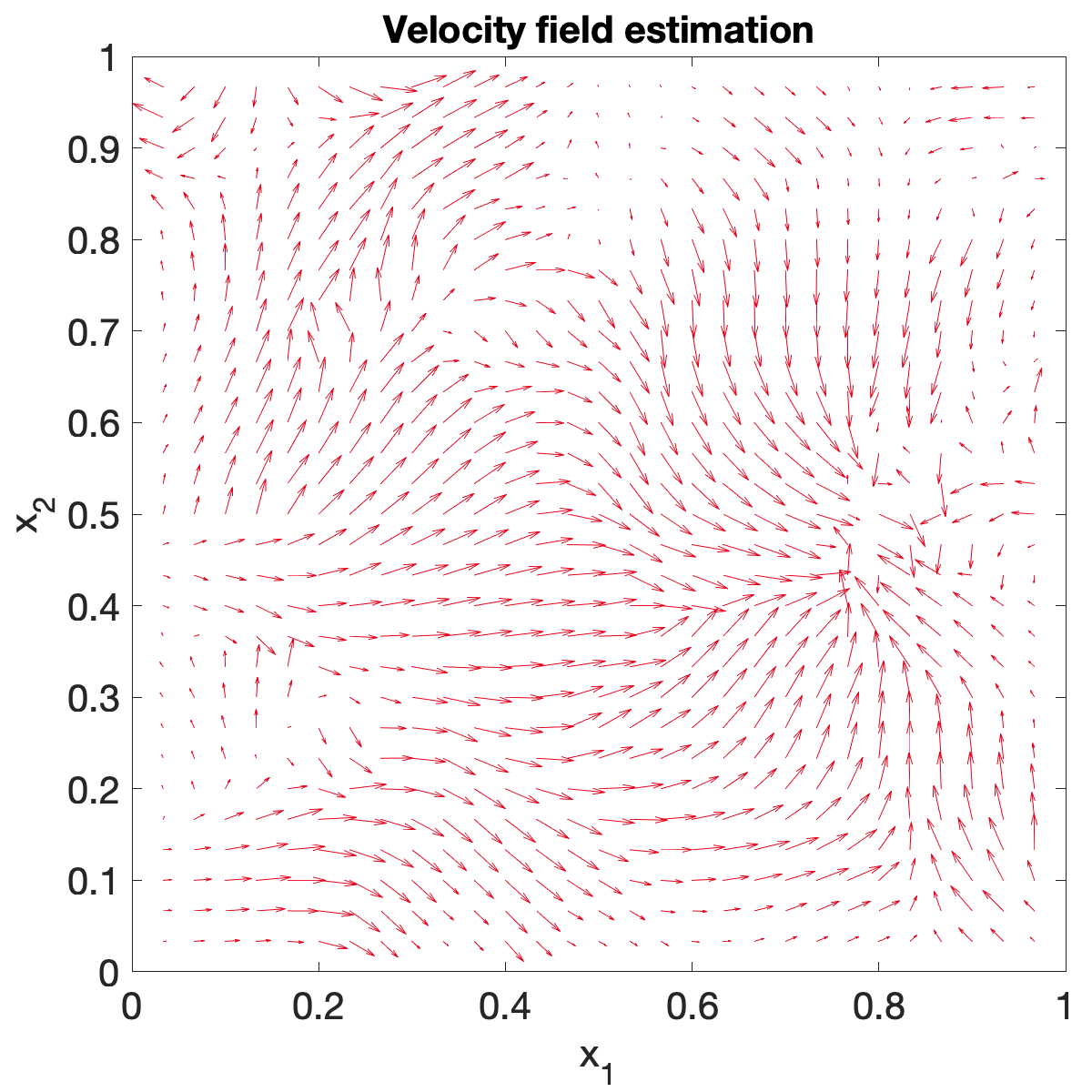}
    \end{subfigure}
    \begin{subfigure}[b]{0.19\textwidth}
        \centering
        \includegraphics[width=\textwidth]{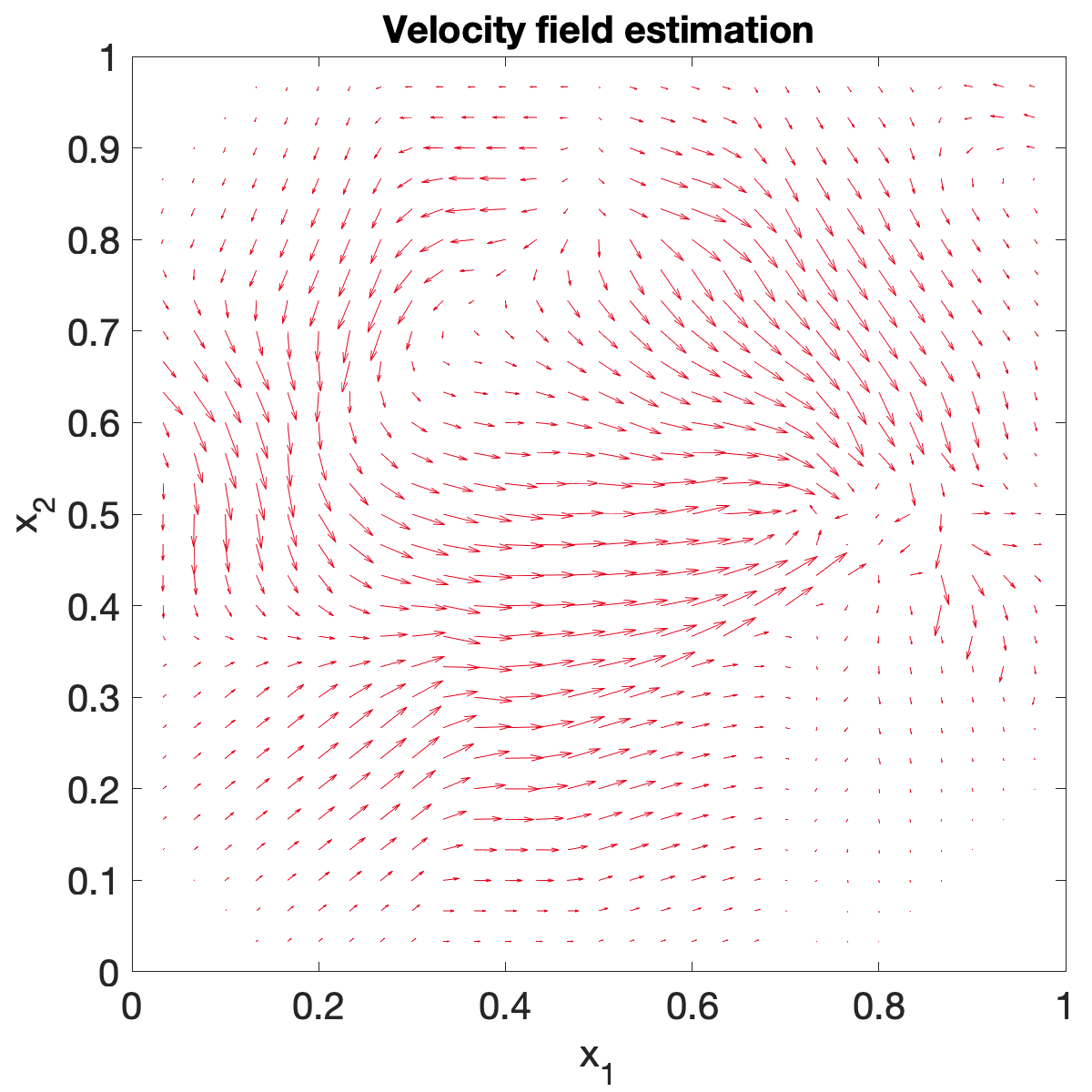}
    \end{subfigure}
    \begin{subfigure}[b]{0.19\textwidth}
        \centering
        \includegraphics[width=\textwidth]{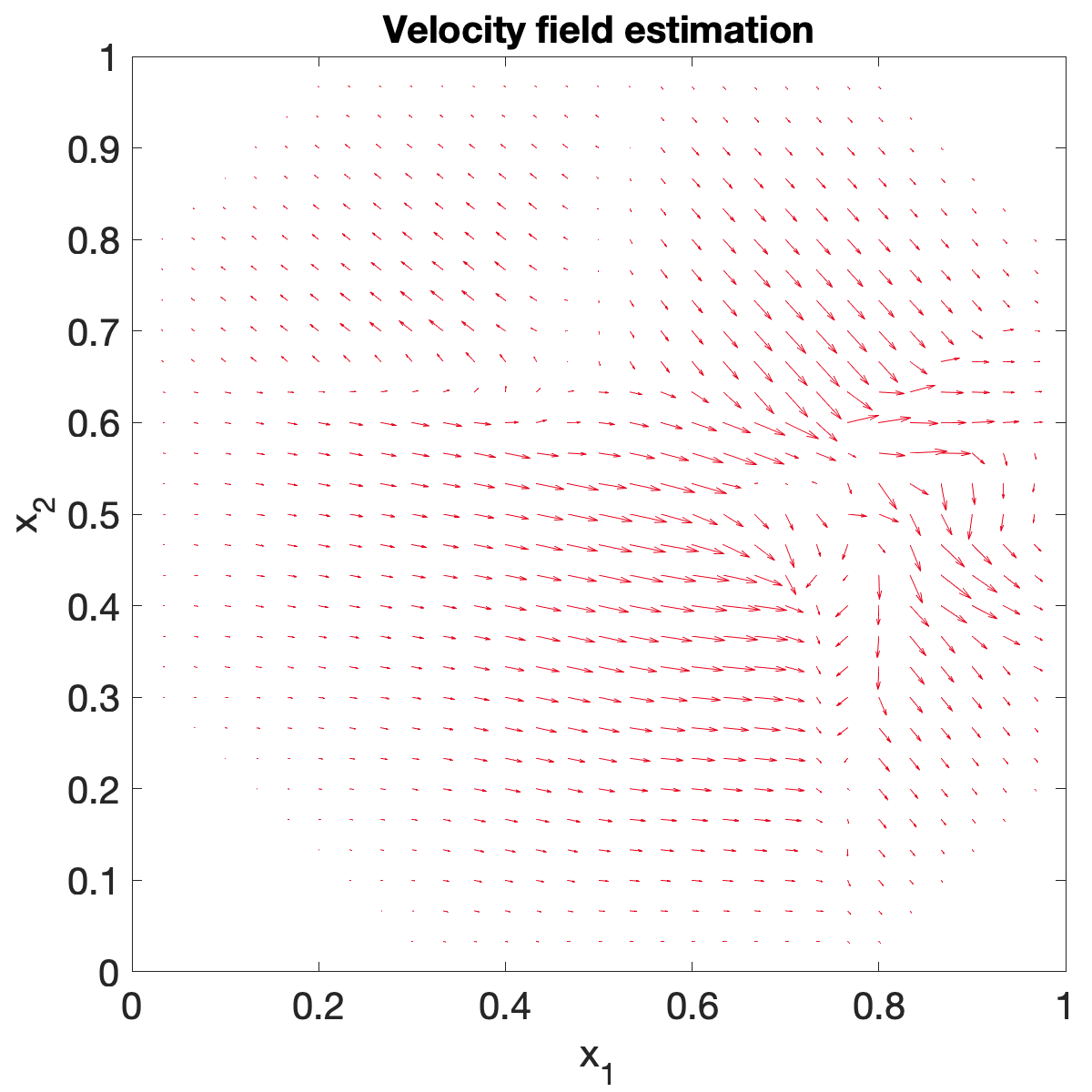}
    \end{subfigure}
    \caption{Illustration of the evacuation process when there were unknown static obstacles (black blocks) in the environment.
    }
    \label{fig:evacuation with static obstacles}
\end{figure*}

\begin{figure*}[h]
    \centering
    \begin{subfigure}[b]{0.19\textwidth}
        \centering
        \includegraphics[width=\textwidth]{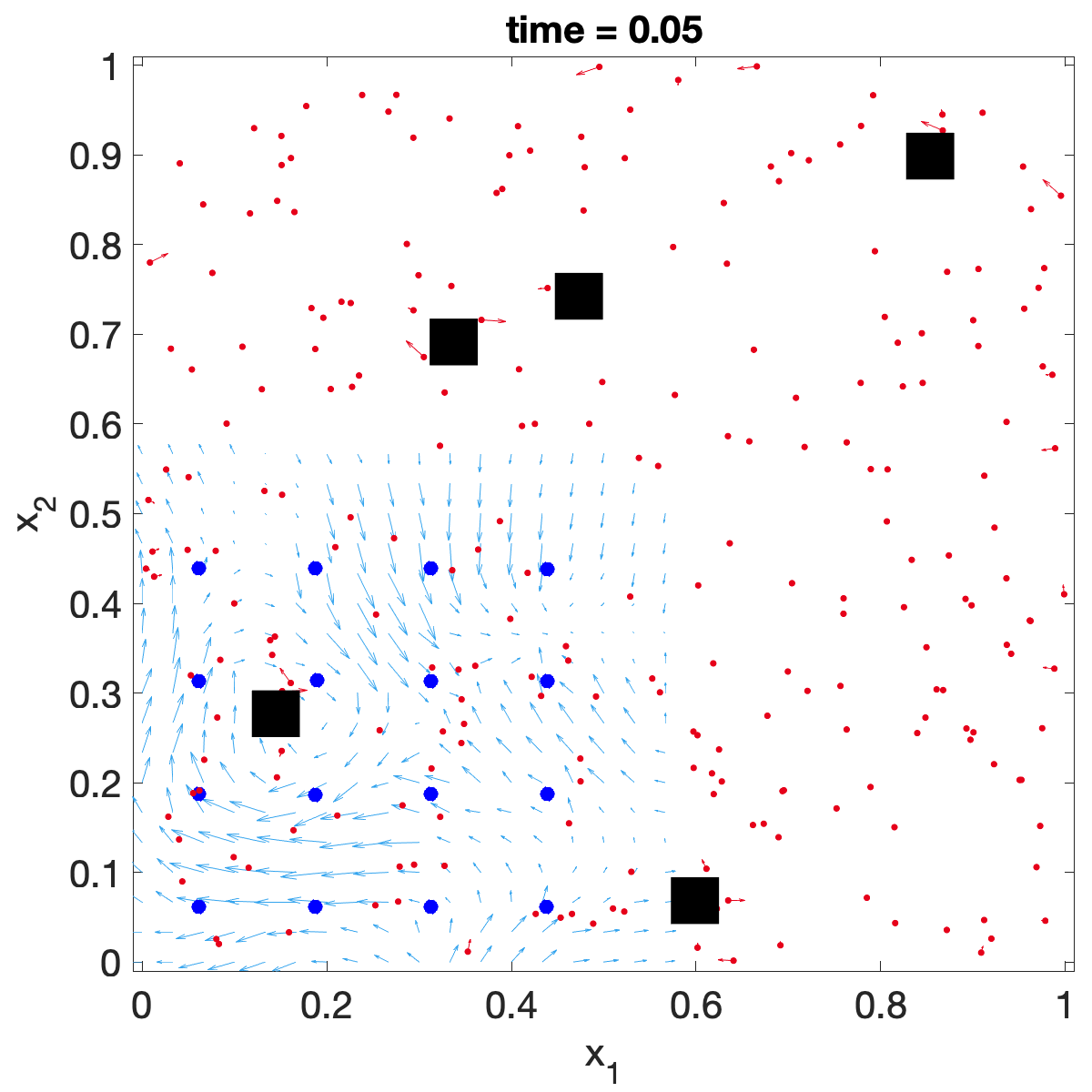}
    \end{subfigure}
    \begin{subfigure}[b]{0.19\textwidth}
        \centering
        \includegraphics[width=\textwidth]{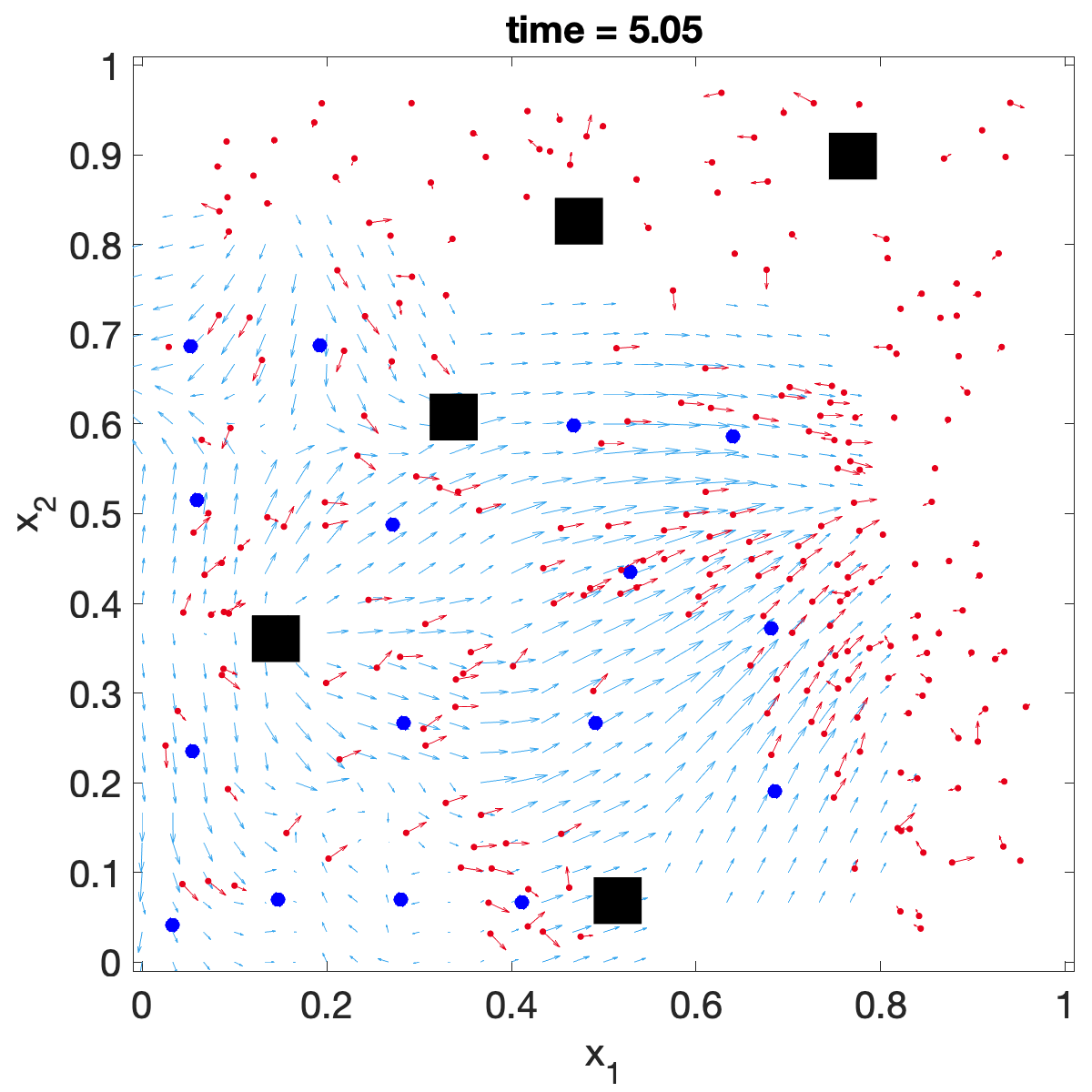}
    \end{subfigure}
    \begin{subfigure}[b]{0.19\textwidth}
        \centering
        \includegraphics[width=\textwidth]{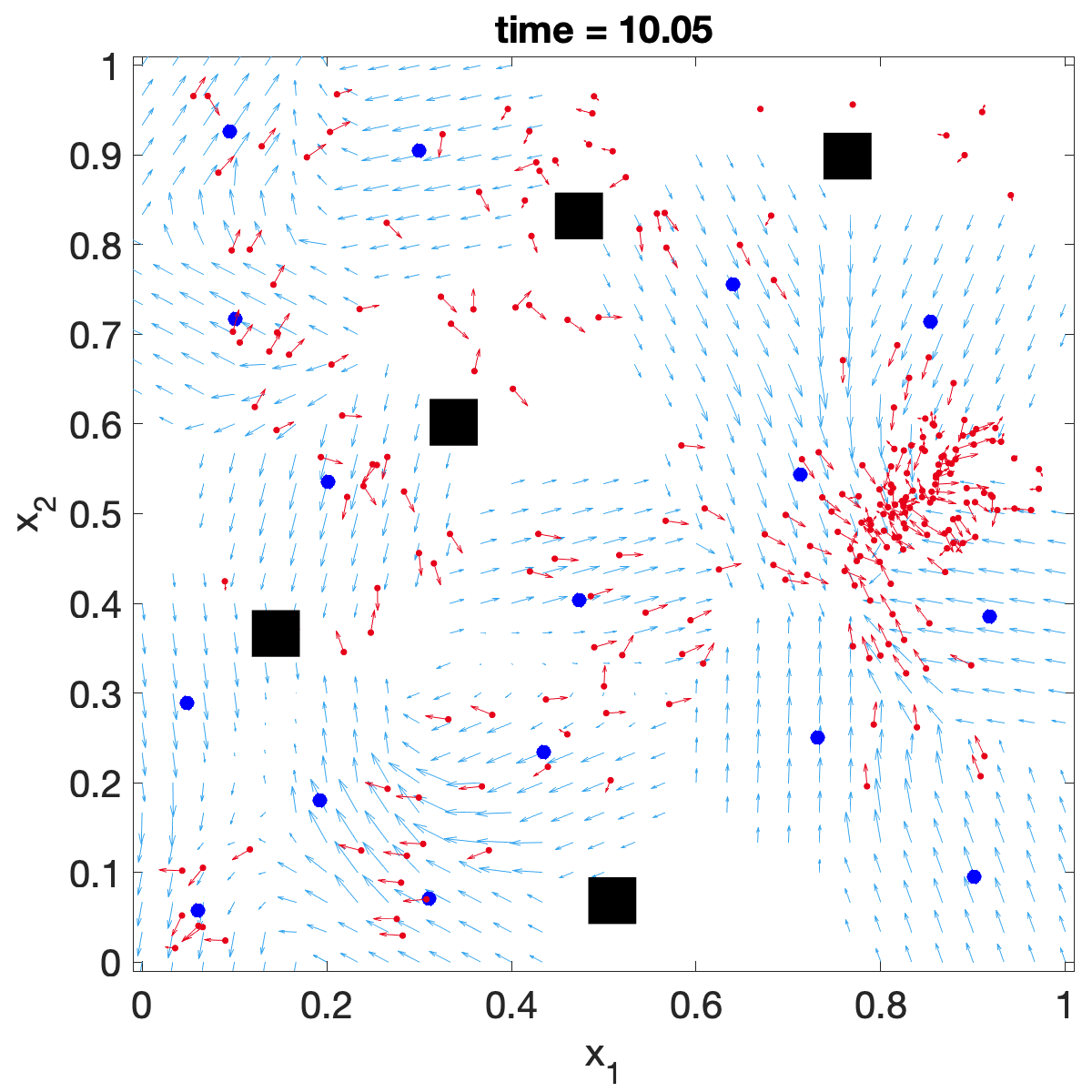}
    \end{subfigure}
    \begin{subfigure}[b]{0.19\textwidth}
        \centering
        \includegraphics[width=\textwidth]{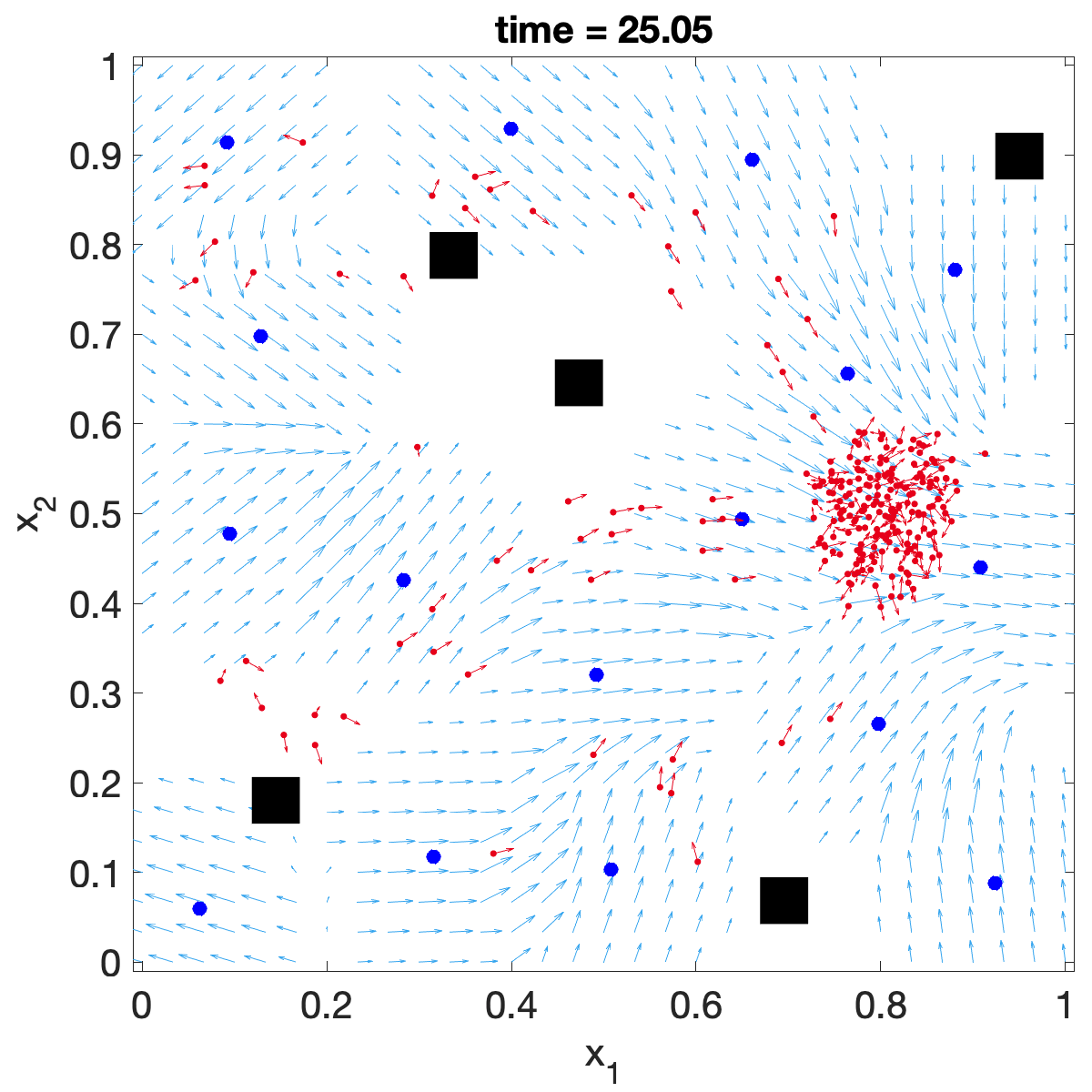}
    \end{subfigure}
    \begin{subfigure}[b]{0.19\textwidth}
        \centering
        \includegraphics[width=\textwidth]{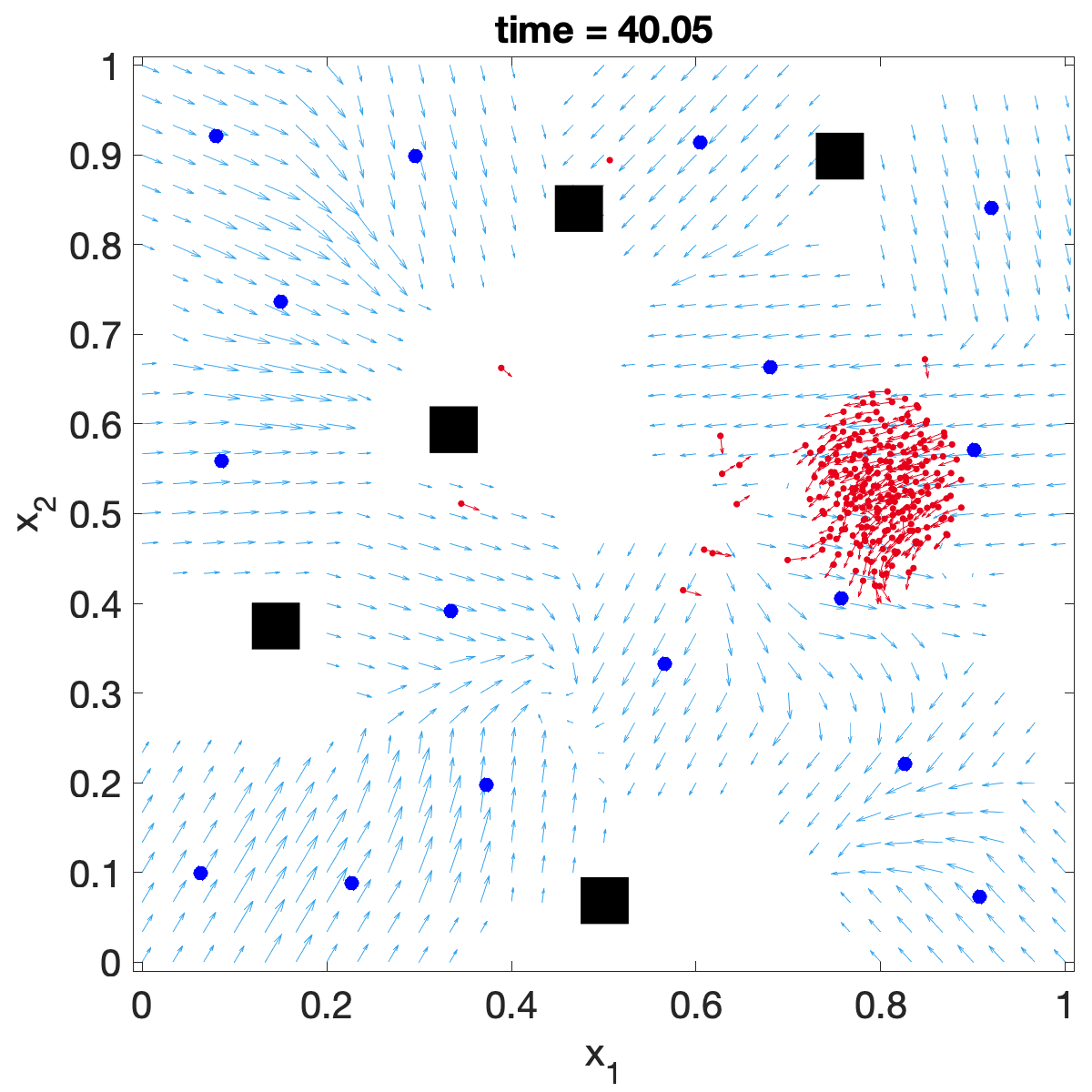}
    \end{subfigure}

    \begin{subfigure}[b]{0.19\textwidth}
        \centering
        \includegraphics[width=\textwidth]{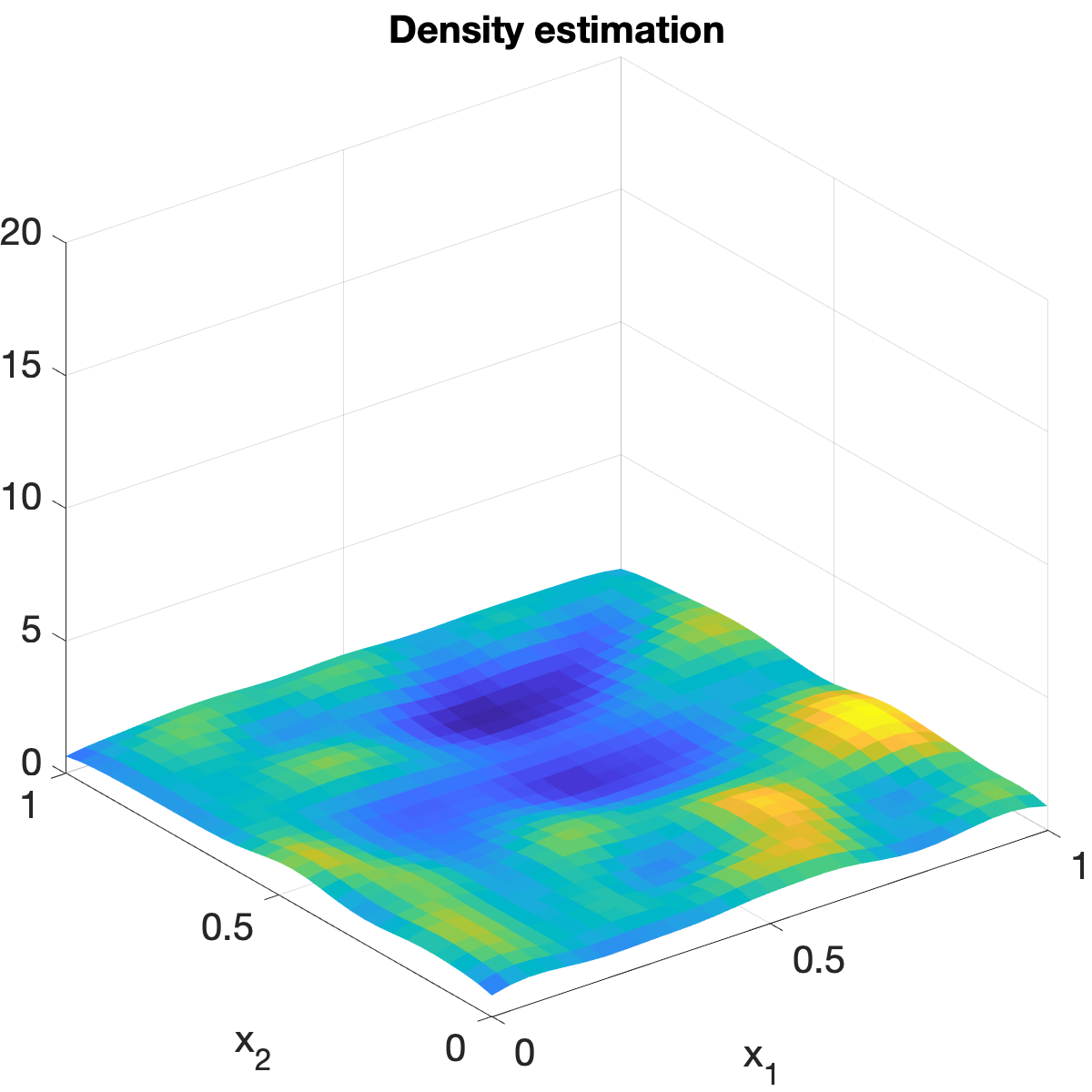}
    \end{subfigure}
    \begin{subfigure}[b]{0.19\textwidth}
        \centering
        \includegraphics[width=\textwidth]{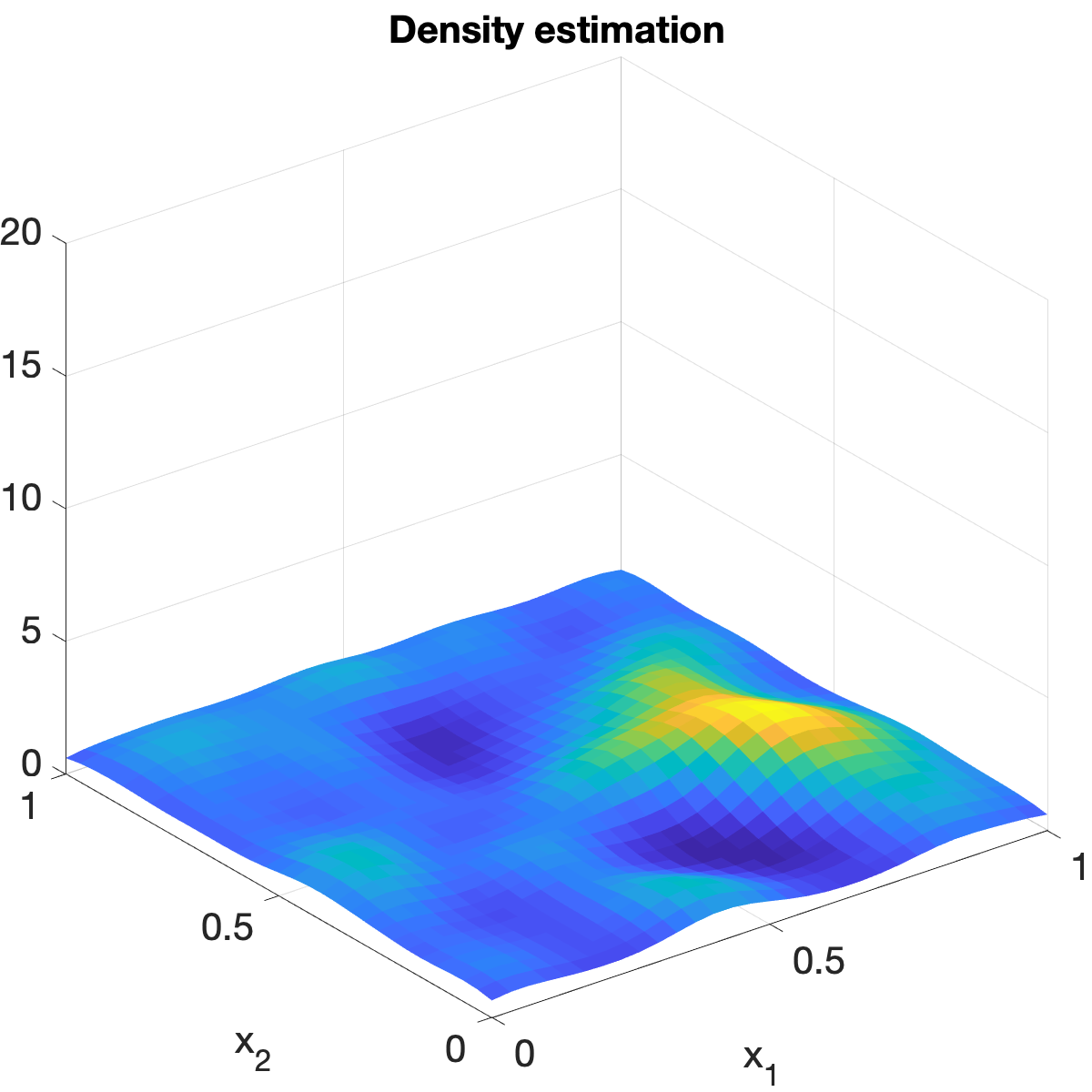}
    \end{subfigure}
    \begin{subfigure}[b]{0.19\textwidth}
        \centering
        \includegraphics[width=\textwidth]{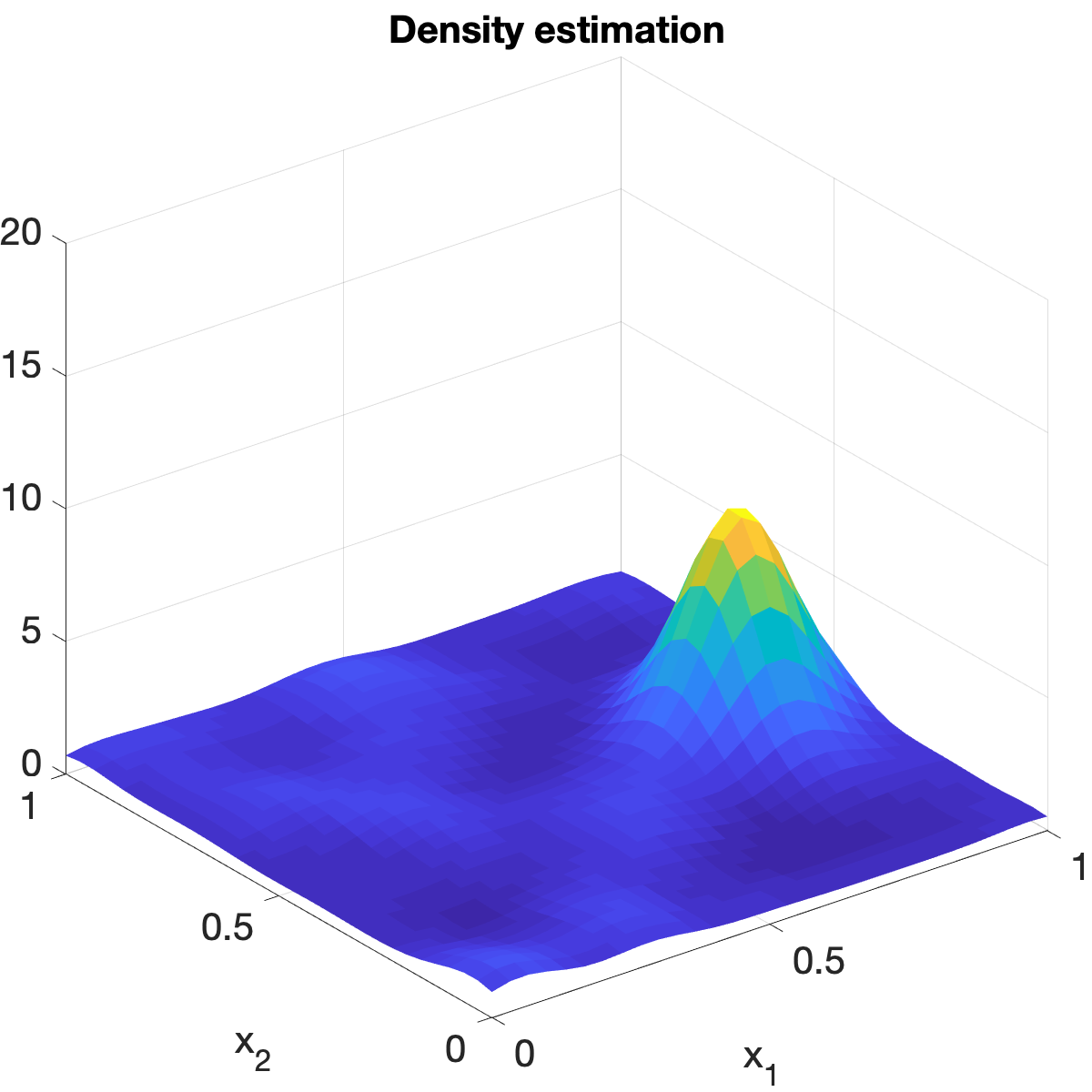}
    \end{subfigure}
    \begin{subfigure}[b]{0.19\textwidth}
        \centering
        \includegraphics[width=\textwidth]{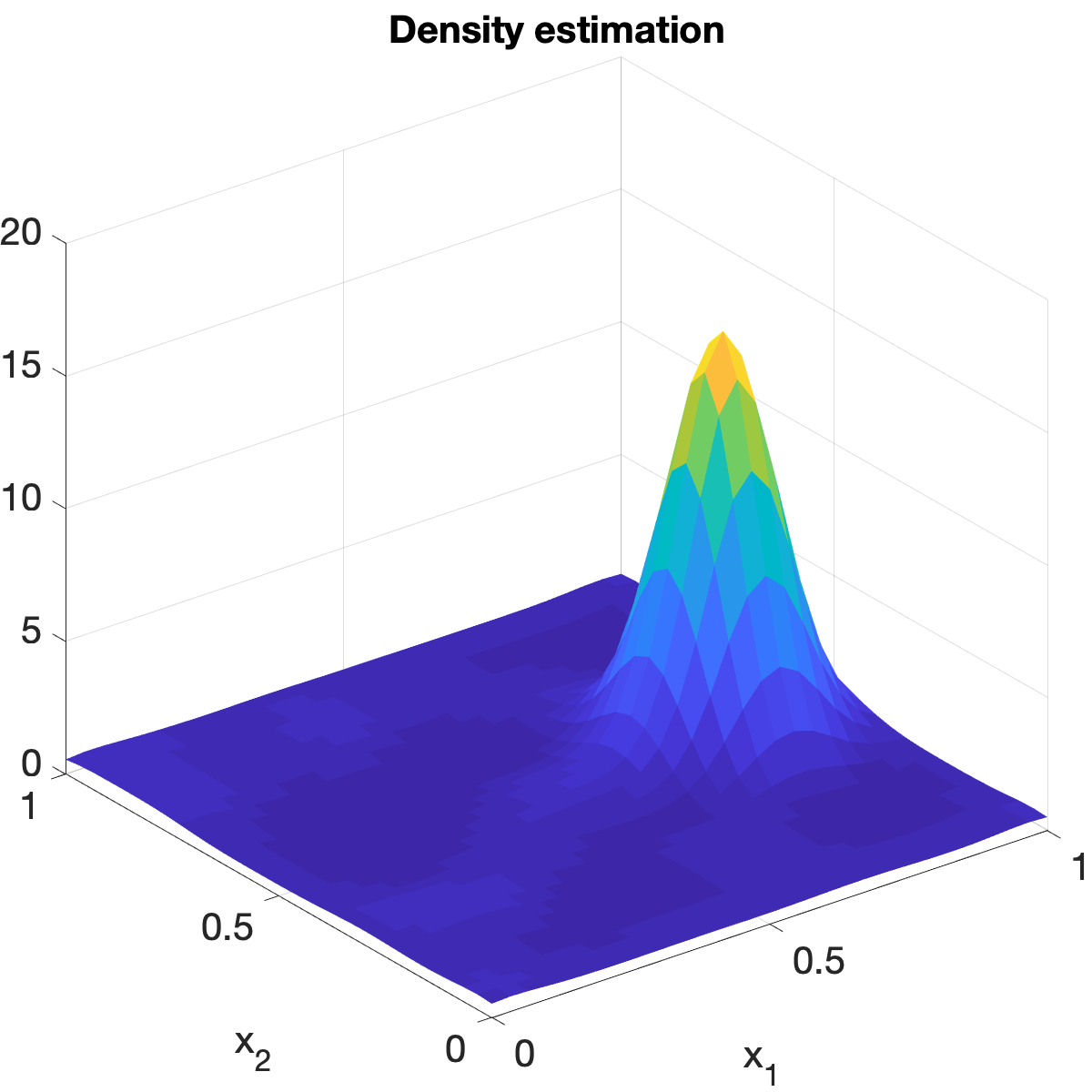}
    \end{subfigure}
    \begin{subfigure}[b]{0.19\textwidth}
        \centering
        \includegraphics[width=\textwidth]{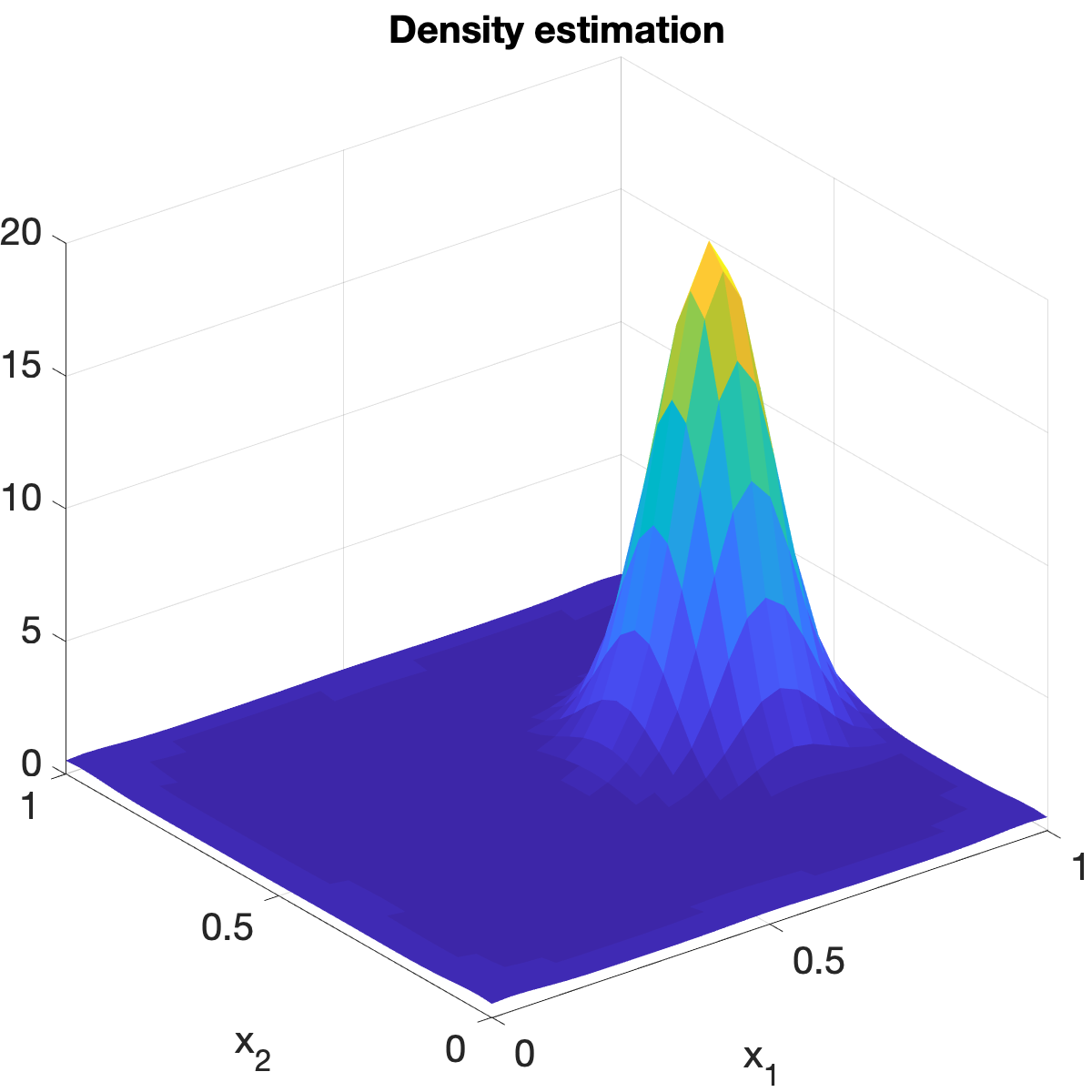}
    \end{subfigure}
    
    \begin{subfigure}[b]{0.19\textwidth}
        \centering
        \includegraphics[width=\textwidth]{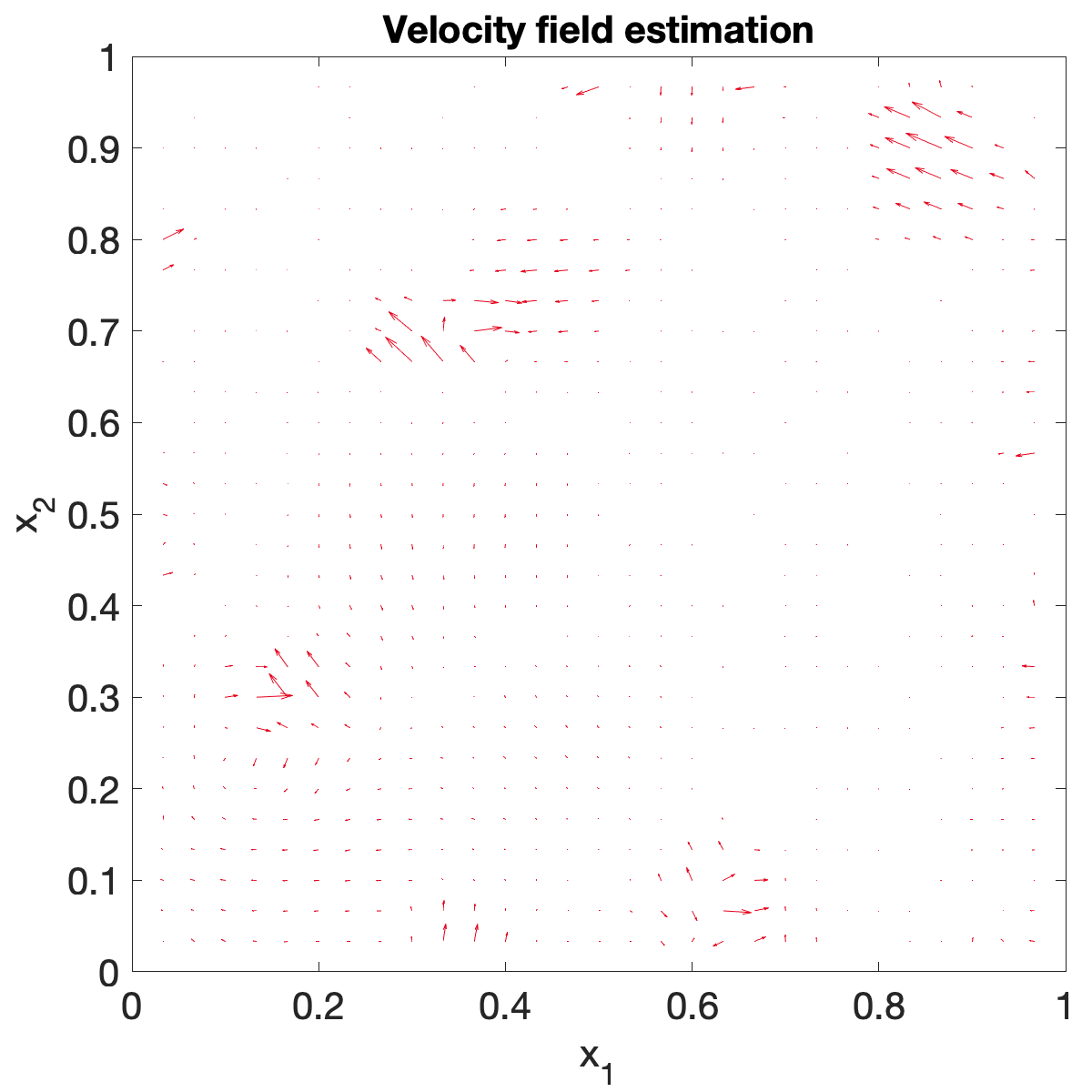}
    \end{subfigure}
    \begin{subfigure}[b]{0.19\textwidth}
        \centering
        \includegraphics[width=\textwidth]{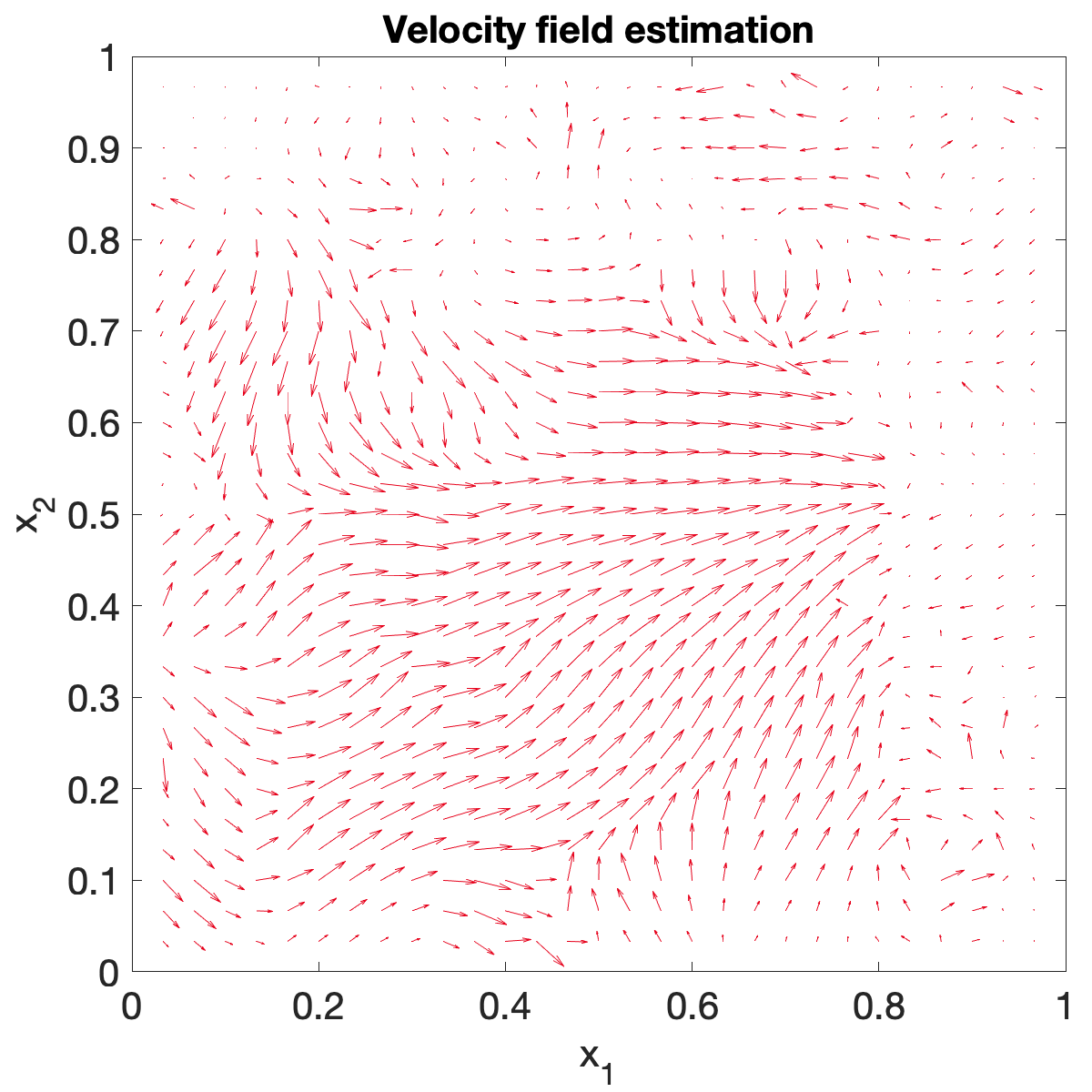}
    \end{subfigure}
    \begin{subfigure}[b]{0.19\textwidth}
        \centering
        \includegraphics[width=\textwidth]{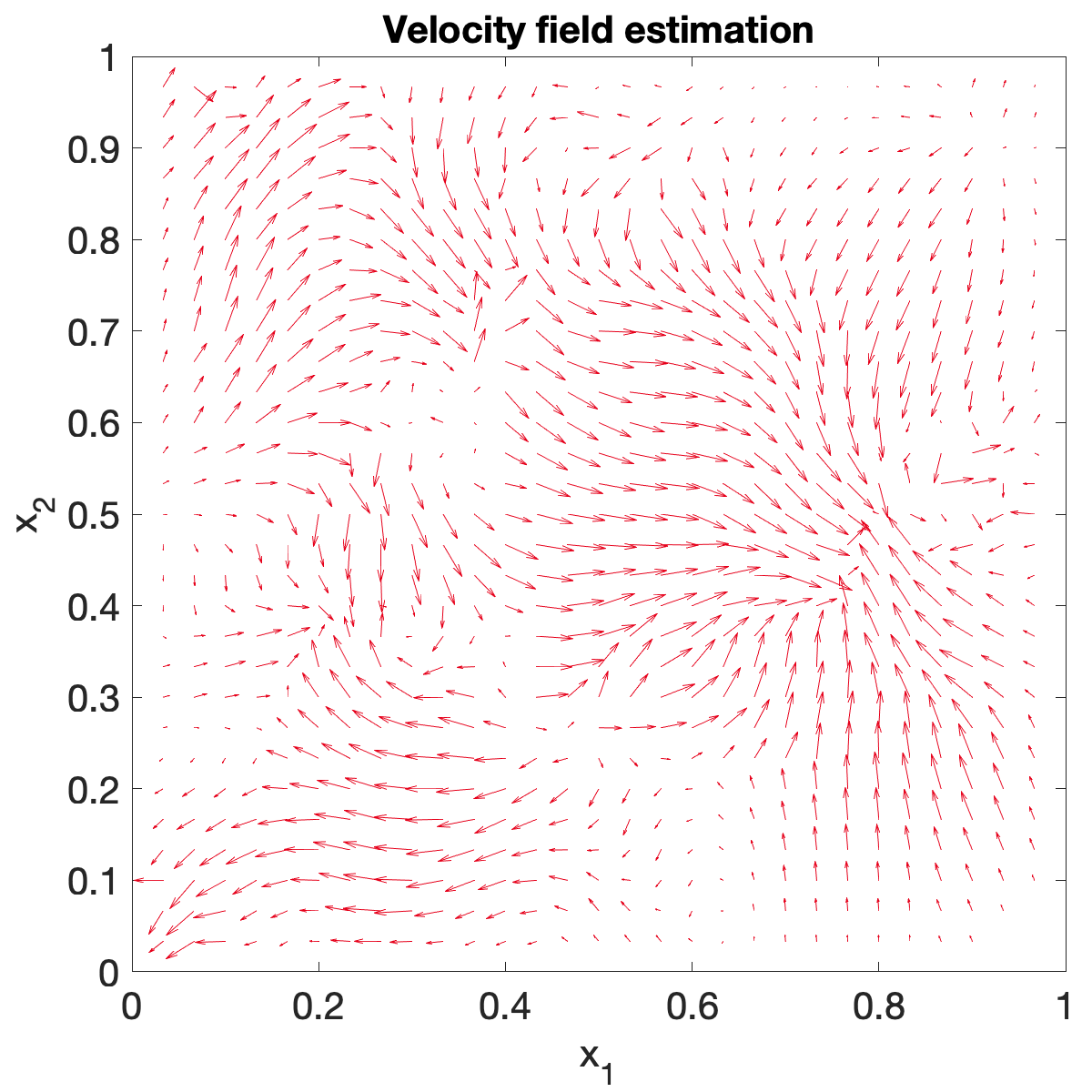}
    \end{subfigure}
    \begin{subfigure}[b]{0.19\textwidth}
        \centering
        \includegraphics[width=\textwidth]{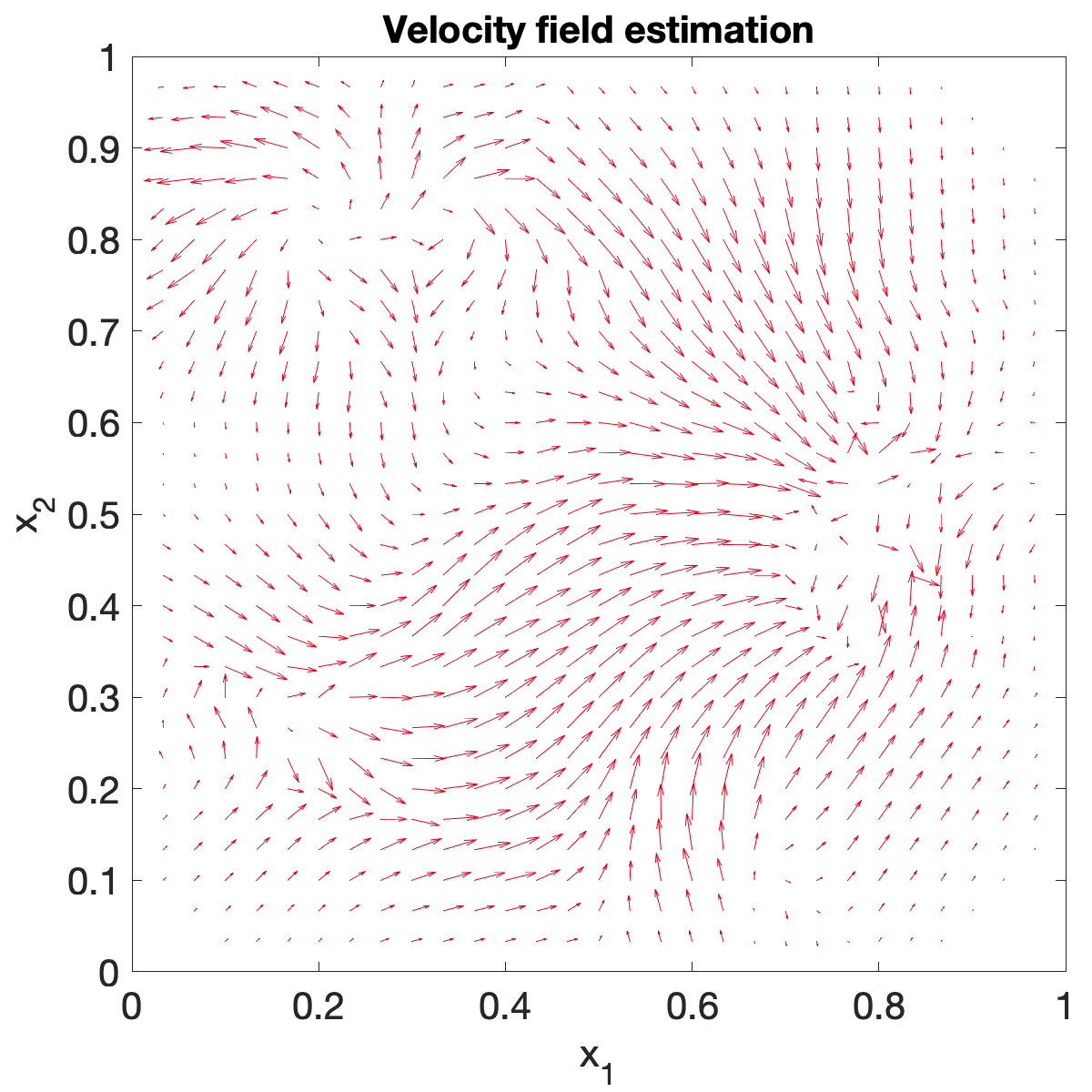}
    \end{subfigure}
    \begin{subfigure}[b]{0.19\textwidth}
        \centering
        \includegraphics[width=\textwidth]{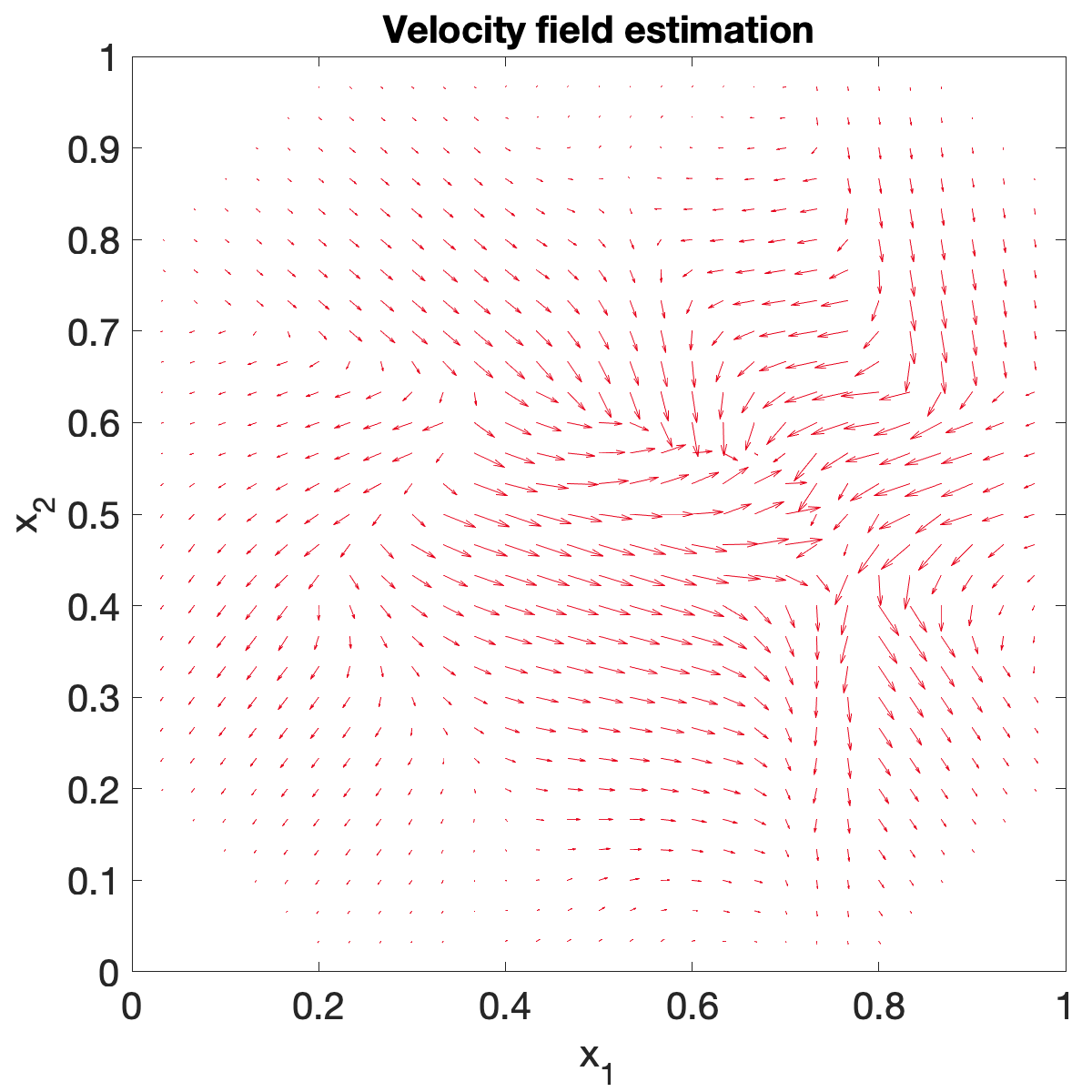}
    \end{subfigure}
    \caption{Illustration of the evacuation process when there were unknown dynamic obstacles (black blocks) in the environment.
    }
    \label{fig:evacuation with dynamic obstacles}
\end{figure*}

\begin{figure*}
    \centering
    \includegraphics[width=0.9\textwidth]{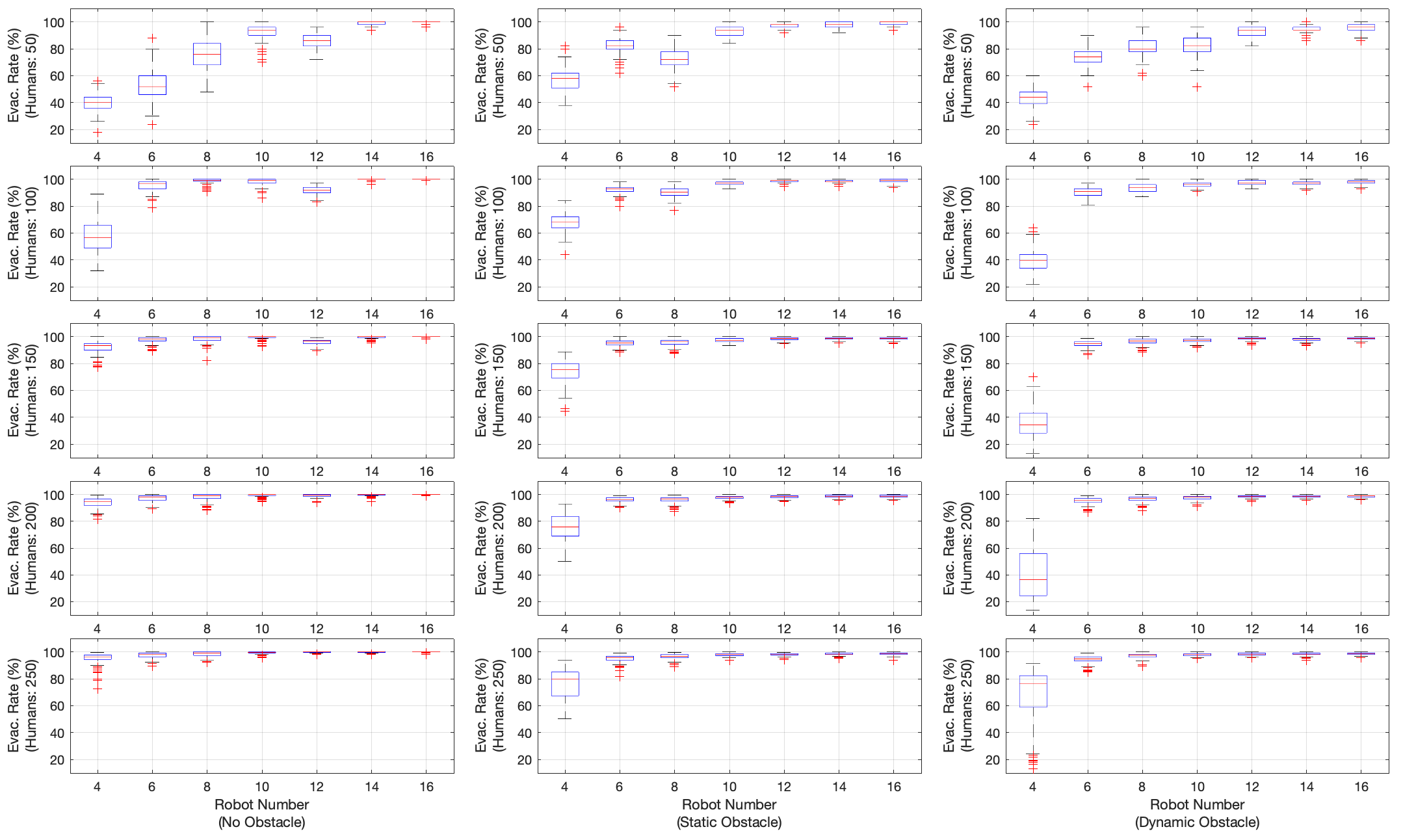}
    \caption{Boxplots of the evacuation rates under different combinations of human number, robot number, and obstacle setup. There were 128 simulation results for each combination.}
    \label{fig:Evacuation Rate Statistics}
\end{figure*}

\textit{Simulation setup.} 
The workspace $\Omega$ was set to be a square $\Omega=[0,1]^2$.
The humans' initial positions $\{x_j(0)\}_{j=1}^N$ were sampled according to the uniform distribution and their initial velocities $\{v_j(0)\}_{j=1}^N$ were set to 0.
The robots' initial positions $\{r_i(0)\}_{i=1}^n$ were chosen to form a regular array at the left lower corner, and their initial directions $\{\theta_i(0)\}_{i=1}^n$ were randomly selected.
Then these states were updated at each iteration according to their microscopic models \eqref{eq:human model} and \eqref{eq:robot model} for which we set the time difference to be $dt=0.1$.
For \eqref{eq:human model}, we set the human-human social force model to be
\begin{align*}
    U(\xi)=0.02e^{-\|\xi\|/0.05}-0.01e^{-\|\xi\|/0.1},
\end{align*}
and the human-robot interaction model \eqref{eq:human-robot force} to be
\begin{align*}
    \bar{K}(\xi)=
    \begin{cases}
    0.05\exp(-\|\xi\|^2/0.03), & \text{if } \|\xi\|<0.15 \\
    0, & \text{if } \|\xi\|\geq0.15.
    \end{cases}
\end{align*}
In other words, the impact of each robot was limited to humans who were within a radius of 0.15.
It is clear that once the robot inputs $\{\tau_i(t),\eta_i(t)\}_{i=1}^n$ are specified, the human and robot states can be consecutively updated at every iteration.
The purpose of evacuation was to drive humans to a small area located in a safe location $\mu_{\mathrm{safe}}=[\frac{13}{16},\frac{1}{2}]^T$.
Hence, the target density $\rho_*$ was chosen to be a narrow Gaussian distribution centered at $\mu_{\mathrm{safe}}$:
\begin{align*}
    \rho_*(x)=\frac{1}{2\pi\sigma}\exp\Big(-\frac{\|x-\mu_{\mathrm{safe}}\|^2}{2\sigma^2}\Big)
\end{align*}
where $\sigma=0.085$; see Fig. \ref{fig:target density} for illustration.

{\color{black}
To investigate the impact of the human number, the robot number, and the environment, we performed simulations for different combinations of these conditions.
In particular, the human number was set to be $N\in\{50,100,150,200,250\}$, and the robot number was set to be $n\in\{4,6,8,10,12,14,16\}$.
For the environment, we had three setups for the obstacles and the unknown social force $G$:
\begin{enumerate}
    \item No obstacle: The environment was free of obstacles, but $G$ still included other unknown social forces.
    We assumed $G=G_1$ and $G_1$ was given by
    \begin{align*}
       G_1(x,t)=\begin{bmatrix}
            -0.01\sin(\frac{\pi}{5}t) \\
            -0.01\sin(\frac{\pi}{5}t)
        \end{bmatrix}
    \end{align*}
    which represented a spatially uniform but time-varying force field.
    \item Static obstacle: Five unknown static obstacles of size $0.05\times0.05$ were placed randomly in the environment and $G=G_1+G_2$ where $G_2$ was the additional repulsive force from the obstacles according to the following:
    \begin{align*}
        G_2(x_j,t)=-\nabla_{x_j}\sum_{l\in\mathcal{O}_j}\frac{0.0005}{\|x_j-s_l\|}, \text{ if } \|x_j-s_l\|\leq0.03
    \end{align*}
    and $G_2(x_j,t)=0$ otherwise, where $\mathcal{O}_j$ is the set of all obstacles seen by the $j$-th human, and $s_l$ is the position of the $l$-th obstacle.
    Thus, $G_2$ may vanish if all humans no longer encounter obstacles after a certain time.
    \item Dynamic obstacle: The obstacles were dynamic and assume $G=G_1+G_2$.
    The positions of the obstacles were updated by adding $\pm0.1\sin(0.2t)$ to one of their coordinates, causing them to move back and forth either horizontally or vertically.
\end{enumerate}
In other words, there were $5\times7\times3=105$ combinations of conditions.
For each combination, we performed 128 simulations with random initialization of human positions.
Each simulation was run for a duration of $T=80$.
A human was considered to have been successfully evacuated if they ended up within a 0.15 radius of the safe location $\mu_{\mathrm{safe}}$.
The evacuation rate was computed according to
\begin{align*}
    \mathrm{Evac.~Rate} = \frac{\mathrm{Evacuated~Number}}{\mathrm{Total~Number}}\times100~\%.
\end{align*}}

\textit{Algorithm setup.}
The robot inputs $\{\tau_i(t),\eta_i(t)\}_{i=1}^n$ were computed at every iteration according to the protocols outlined in Algorithm \ref{algorithm:evacuation}.
The workspace was discretized as a $30\times30$ grid to perform numerical computations of the involved differentiation and integration.
At every iteration, the crowd density $\rho(x,t)$ was estimated from $\{x_j(t)\}_{i=1}^N$ using the kernel density estimator \eqref{eq:KDE} where we chose $H$ to be the Gaussian kernel and set $h=0.07$.
The crowd velocity field $u(x,t)$ was estimated from $\{x_j(t),v_j(t)\}_{i=1}^N$ using linear interpolation.
For the position controllers \eqref{eq:position controller}, we set $k_r=0.003$, $k_o=0.002$, and $\nu=1$.
To approximate the unknown force field $G$ according to \eqref{eq:ANN}, we chose a one-layer radial basis function (RBF) neural network with 25 neurons regularly deployed to cover $\Omega$.
For the direction controllers consisting of \eqref{eq:ud}, \eqref{eq:Fd}, \eqref{eq:update law}, and \eqref{eq:direction controller}, the feedback gains were chosen to be $k_\rho(x,t)=0.05/\|\nabla(\rho-\rho_*)\|$, $k_u=0.1$, $\Gamma=0.1$, $k_w=0.1$, $k_\eta=0.1$, and $\beta_i(t)=1/n'(t)$, where $n'(t)\leq n$ is the cardinality of $\{\beta_i(t),i=1,\dots,n|\beta_i(t)>0\}$.

\textit{Result.}
Figures~\ref{fig:evacuation}, \ref{fig:evacuation with static obstacles}, and \ref{fig:evacuation with dynamic obstacles} show the evacuation process of 250 people and 16 robots in three different environmental settings: no obstacle, static obstacle, and dynamic obstacle, respectively.
In the first rows of these figures, under the position controllers, the robots spread quickly and exhibited satisfactory coverage of the environment.
In the case of dynamic obstacles, the robots continued to move in the environment to ensure good coverage.
Under the direction controllers, the robots dynamically changed the directions of their navigation force fields to affect human motions and guided them toward the target location.
We noticed that since each robot's range of influence $\Omega_i(t)$ was set to be a disk of radius 0.15, their total range of influence $\Omega_r(t)$ was actually unable to completely cover the entire environment, as can be seen from the blank areas without black arrows.
Blank areas were even larger when there were obstacles that caused the environment to be irregular.
However, the robots were still able to guide the entire crowd toward the target area.
This verified our hypothesis following Theorem \ref{thm:asymptotic stability} that the stability conclusion still holds as long as the robots' total range of influence $\Omega_r(t)$ is able to cover the entire crowd $\Omega_\rho(t)$, rather than the entire environment $\Omega$.
In the second and third rows, real-time estimates of the macroscopic states (the crowd density $\rho(x,t)$ and velocity field $u(x,t)$) were displayed, which were consistent with the microscopic states given in the first rows.
We observed that the crowd density estimate eventually converged to the target density given in Fig. \ref{fig:target density}.
The $L^2$ norms of the density errors $\tilde{\rho}$ for the three experiments were plotted in Fig. \ref{fig:error}, which verified the convergence of the crowd density.

{\color{black}
The impact of the human number, the robot number, and the obstacle setup was summarized using boxplots in Fig.~\ref{fig:Evacuation Rate Statistics}. 
Each column of subfigures represented the same obstacle setup.
Each row of subfigures represented the same human number.
Each subfigure showed the boxplots of evacuation rates of 128 simulations with respect to the robot number.
We observed the following implications.
\begin{enumerate}
    \item The evacuation rate increased as the human number increased. 
    This was due to the fact that the macroscopic approximation was more precise when the human population was larger. 
    Furthermore, social force models were more likely to generate self-organizing patterns when there were more people.
    \item The evacuation rate increased as the robot number increased.
    This was due to the fact that more robots were able to cover a larger area of the environment.
    It was interesting to observe that when the human number was larger than 200, a number of 10 robots, which could not cover the entire environment, were able to evacuate more than 90\% of humans.
    Once again, this verified our hypothesis following Theorem \ref{thm:asymptotic stability} that the stability conclusion still holds as long as the robots' total range of influence $\Omega_r(t)$ is able to cover the entire crowd $\Omega_\rho(t)$, rather than the entire environment $\Omega$.
    \item The presence of obstacles, particularly dynamic ones, always reduced the evacuation rate. 
    This was because they created an irregular environment, which hindered the robots' ability to cover the environment. 
    However, when the number of humans and robots was sufficiently large, for example, when $N\geq200$ and $n\geq12$, the effect of the obstacles became negligible.
\end{enumerate}}

\section{Conclusion}
\label{section:conclusion}
In this research, we investigated a robot-guided evacuation of a crowd in which humans greatly outnumber robots. 
We employed a two-scale modeling approach based on hydrodynamic models to characterize the behavior of both humans and robots, as well as their interactions.
We developed controllers for robots to explore the environment, avoid obstacles, and adjust their local navigation force fields in real time based on the density and velocity field of the crowd to guide them to a secure area. 
We conducted extensive simulations to validate the effectiveness of the proposed algorithms.
Our next step is to extend the evacuation algorithm to unstructured settings of the environment and decentralize the algorithm.

\bibliographystyle{IEEEtran}
\bibliography{References}

\end{document}